\documentclass[12pt]{report}

\usepackage[T1]{fontenc}
\usepackage[utf8]{inputenc}
\usepackage{amsmath}
\usepackage{amsfonts}
\usepackage{amssymb}
\usepackage{algorithm}
\usepackage{algorithmic}
\usepackage{comment}
\usepackage{subcaption}
\usepackage{longtable}
\usepackage{xcolor}

\definecolor{myblue}{RGB}{31,73,125}
\definecolor{myred}{RGB}{192,0,0}
\definecolor{mygreen}{RGB}{0,110,0}

\usepackage[backend=biber,style=authoryear,dashed=false,natbib=true,giveninits,%
            uniquename=false,uniquelist=false,maxcitenames=2]{biblatex}
\usepackage[autostyle]{csquotes}
\usepackage{tikz}
\usepackage{pgfplots}
\usepackage{pifont}
\usetikzlibrary{arrows.meta,positioning,shapes,calc,
                decorations.pathreplacing,fit,
                backgrounds,patterns}
\pgfplotsset{compat=1.18}
\usepackage[colorlinks=true,linkcolor=myblue,citecolor=myblue,urlcolor=myblue]{hyperref}
\usepackage{amsthm}

\theoremstyle{plain}
\newtheorem{theorem}{Theorem}[chapter]
\newtheorem{lemma}[theorem]{Lemma}
\newtheorem{corollary}[theorem]{Corollary}
\newtheorem{proposition}{Proposition}[chapter]
\theoremstyle{definition}
\newtheorem{definition}{Definition}[chapter]

\newcommand{\iid}{\overset{\mathrm{i.i.d.}}{\sim}}
\newcommand{\distas}{\overset{d}{=}}
\DeclareMathOperator{\Unif}{Unif}

\title{Statistically Valid Post-Training Hyperparameter Selection\\[0.5ex]%
       \large From Tuning to Guarantees}
\author{Amirmohammad Farzaneh \and Osvaldo Simeone\\[1ex]
        \normalsize Institute for Intelligent Networked Systems (INSI)\\
        \normalsize Northeastern University London\\
        \normalsize \texttt{\{a.farzaneh, o.simeone\}@nulondon.ac.uk}}
\date{}

\begin{document}

\maketitle
\tableofcontents
\clearpage

\begin{abstract}
Post-training hyperparameter selection is a critical step in the deployment of modern artificial intelligence systems, given the need to tune degrees of freedom of pre-trained models such as inference-time parameters, implementation-level settings, and thresholds driving decision rules. Despite its practical importance, hyperparameter selection is typically performed using best-effort empirical methods such as grid search or Bayesian optimization, which provide no formal statistical guarantees on reliability or safety.

This monograph, intended for an audience of signal processing and machine learning researchers, presents a unified statistical framework for reliable post-training hyperparameter selection, centered on the learn-then-test (LTT) paradigm. LTT formulates the hyperparameter selection problem as multiple hypothesis testing over a candidate set of hyperparameters. The framework enables the choice of hyperparameters that provably satisfy application-specific reliability requirements---such as bounds on average risk, quantile risk, or information-theoretic constraints---with explicit, finite-sample control of error probabilities. The supporting statistical machinery, namely p-values, e-values, and concentration inequalities, is developed from first principles.

Starting from the core LTT methodology, the monograph progressively broadens the framework along four directions. First, it extends LTT beyond average risk to general reliability functionals, including quantile risk control and information-bottleneck constraints. Second, it addresses multi-objective settings where several reliability constraints must be enforced simultaneously, developing Pareto testing and reliability graph-based Pareto testing. Third, it relaxes the batch-calibration assumption via adaptive and sequential methods that exploit e-processes and anytime-valid inference to reduce the labelling budget. Fourth, it considers settings with limited ground-truth labels, combining testing-by-betting with prediction-powered inference (PPI) to enable reliable hyperparameter selection from autoevaluated data.

Throughout, the framework is illustrated on practical applications including image classification, wireless packet scheduling, and large language model (LLM) prompt engineering, with empirical comparisons that demonstrate the gap between heuristic tuning and statistically guaranteed selection.
\end{abstract}

\chapter{Reliable Post-Training Hyperparameter Selection}
\label{chapter:intro}

\section{Post-Training Hyperparameter Selection in Modern AI Systems}

\subsection{Scope}

Hyperparameters determine how AI models are structured, trained, and deployed. Unlike model parameters, which are estimated automatically from training data, hyperparameters must be specified using additional sources of information, including prior knowledge and calibration or validation data.

Broadly speaking, hyperparameters govern the operation of a machine learning system at different levels:

\begin{itemize}

    \item At the \textbf{model level}, hyperparameters include architectural choices such as type of inter-token processing, e.g., self-attention versus state-space models, as well as network depth and width.

    \item At the \textbf{learning level}, hyperparameters include learning rates, regularization coefficients, batch sizes, and training-time temperature in mixture-of-expert architectures.

    \item At the \textbf{inference level}, hyperparameters determine how model outputs are translated into actions or decisions, for example through confidence thresholds, rejection rules, or decoding parameters such as the temperature of a language head.

    \item At the \textbf{deployment level}, hyperparameters encode aspects such as precision of the activations and of the weights of a model.

\end{itemize}

In conventional machine learning pipelines, optimizing model-based or learning-based hyperparameters typically requires running multiple training rounds using methods such as cross-validation (see, e.g., \citep{simeone2022machine}). However, for modern large AI models, this is not feasible, since even a single full training run constitutes a major computational investment \citep{kaplan2020scaling,hoffmann2022training}.

Practical solutions for model-level or training-level hyperparameters include partial retraining across multiple hyperparameter configurations via optimized resource allocation strategies such as successive halving \citep{jamieson2016non} and Hyperband \citep{li2018hyperband}. For extremely large models, one instead typically relies on scaling laws to translate optimal choices at lower scales to effective selection at higher scales \citep{kaplan2020scaling, hoffmann2022training}.

Unlike model-level and training-level hyperparameters, inference and deployment-based hyperparameters can be optimized via post-training strategies on pre-trained models. Such post-training methods support data-driven evaluations of the performance of different hyperparameters even for large AI models. In this monograph, we target inference and deployment-based hyperparameters, thus focusing on post-training hyperparameter selection methods.

This focus distinguishes the present monograph from algorithmic treatments of hyperparameter optimization (HPO). These include \citep{franceschi2025hpo_fnt, snoek2012practical, bergstra2012random, li2018hyperband}, which study how to efficiently search large hyperparameter spaces under computational budgets using methods such as Bayesian optimization, multi-fidelity methods, and gradient-based bilevel schemes. In contrast to these methods, this monograph treats the candidate set $\Lambda$ as given, using any HPO method, and concentrates instead on what can be said statistically about the final selected configuration under finite calibration data. The main goal of this monograph is to expose signal processing and machine learning researchers to methodologies that endow post-training hyperparameter selection with explicit, finite-sample guarantees of reliability.

\subsection{Examples}

Examples of applications and use cases include the following:

\begin{itemize}

\item In large language models, decoding hyperparameters such as temperature, nucleus sampling thresholds, and safety filters directly influence response diversity, factuality, and refusal behavior \citep{bommasani2021opportunities, holtzman2019curious} (see Fig. \ref{fig:llm_temperature_example} for an example).

\begin{figure}[t]
    \centering
    \includegraphics[width=\linewidth]{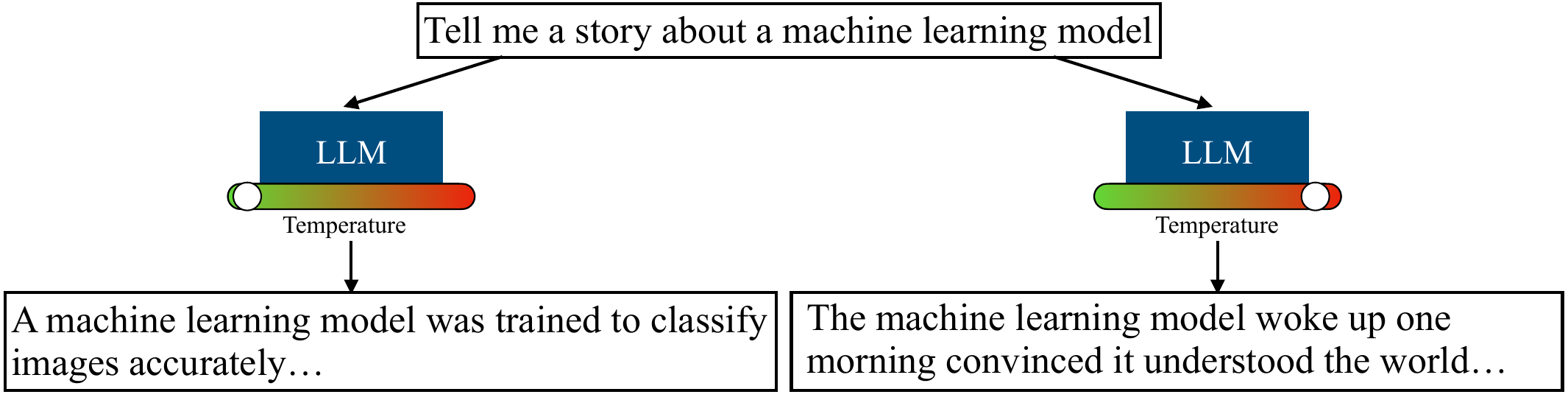}
    \caption{Illustration of the effect of the decoding temperature on the outputs of a pre-trained large language model. For the same prompt and identical decoding settings except for temperature, a low temperature yields conservative, literal, and highly repeatable responses, while a high temperature produces more diverse, creative, and narrative-driven outputs.}
    \label{fig:llm_temperature_example}
\end{figure}

\item In recommendation and ranking systems, regularization strengths and exploration parameters affect stability, long-term user engagement, and feedback loops \citep{covington2016deep}.

\item In computer vision pipelines, confidence thresholds determine whether predictions are accepted or deferred, shaping precision--recall trade-offs in detection and segmentation tasks \citep{he2017mask}.

\item In learning-enabled communication and control systems, hyperparameters regulate performance trade-offs among throughput, latency, reliability, and energy consumption \citep{shlezinger2021model, simeone2025conformal}.

\item In ground-truth in-the-loop systems \citep{geifman2017selective,raghu2019direct}, confidence thresholds and safety margins govern when models are permitted to act autonomously or escalate to fallback policies \citep{amodei2016concrete}.

\end{itemize}

\section{Limitations of Best-Effort Tuning Methods}

Hyperparameter selection has traditionally been approached as an empirical performance optimization problem. Given a calibration, or validation, dataset and a predefined evaluation protocol, the goal is to identify hyperparameter configurations that achieve strong average empirical performance on the dataset under computational constraints.

This optimization-centric perspective has given rise to a rich and mature literature on HPO. In the remainder of this section, we first briefly review HPO methods, and then discuss the gaps and limitations that motivate the reliability-oriented framework developed in this monograph.

\subsection{Optimization-Centric Hyperparameter Selection}
\label{sec:background}

The dominant paradigm in HPO frames the problem as one of black-box optimization. Hyperparameters are treated as decision variables, the training and evaluation pipeline defines a stochastic objective function, and the task is to efficiently search the hyperparameter space in order to maximize expected empirical performance.

Early approaches relied on exhaustive or randomized exploration strategies, such as grid search and random search, which remain widely used due to their simplicity and robustness \citep{bergstra2012random}. Subsequent work introduced model-based strategies, most notably Bayesian optimization, which constructs a surrogate model of the validation performance and uses acquisition functions to balance exploration and exploitation \citep{snoek2012practical,shahriari2016taking}.

To address the high computational cost of evaluating hyperparameter configurations, especially in modern deep learning, multi-fidelity and bandit-style approaches have been proposed. Methods such as successive halving and Hyperband allocate computational resources adaptively, terminating poorly performing configurations early while dedicating more resources to promising candidates \citep{jamieson2016non,li2018hyperband}. Related ideas appear in population-based training and evolutionary strategies, which maintain and evolve a population of configurations during training \citep{jaderberg2017population}.

Gradient-based HPO methods have also been developed, leveraging implicit differentiation and hypergradients to optimize continuous hyperparameters directly through the training process \citep{franceschi2018bilevel}. These methods are particularly effective when the training dynamics are differentiable and computationally tractable.

\subsection{Lack of Reliability and Post-Selection Guarantees}
\label{sec:Ch1_no_reliability}

From the perspective of this monograph, the methods reviewed above are best understood as \emph{best-effort tuning procedures}. While they are highly effective at improving average performance and reducing computational cost, they are not designed to provide explicit guarantees on the reliability of the selected hyperparameter once deployed. In fact, despite their algorithmic diversity, these methods share a common objective, namely, to \emph{optimize empirical performance}. Accordingly, standard optimization-based HPO approaches do not formally account for the uncertainty associated with empirical risk estimates.

As emphasized in~\citep{franceschi2025hpo_fnt}, the primary goals of HPO are sample efficiency, computational scalability, and asymptotic performance. Questions of post-selection validity, selective error control, or finite-sample guarantees on the selected configuration lie largely outside the scope of the optimization-centric framework.

Fig.~\ref{fig:conventional_fail} illustrates the mismatch between conventional HPO, which is based on the minimization of the empirical risk, and true risk optimization for a scalar hyperparameter $\lambda$. Panel (a) shows the true risk $R(\lambda)$ (blue dashed curve), together with a realization of the empirical risk $\widehat{R}(\lambda)$ (red curve).
The shaded region represents the support of the empirical risk across different realizations of the calibration data.

Given the realization of the empirical risk, and thus of the calibration data, depicted in panel (a) of Fig. \ref{fig:conventional_fail}, standard HPO selects
a hyperparameter configuration $\hat{\lambda}_{\mathrm{HPO}}$, as the minimizer of the empirical risk, which is significantly different from the true risk minimizer $\lambda^\star$.
This exemplifies the general fact that empirical risk minimization is inherently best-effort: it may favor hyperparameters that appear advantageous due to sampling variability rather than because they minimize the true risk $R(\lambda)$. This effect generally depends on the amount of available data and on the variability of the risk in the space of hyperparameters, and is particularly pronounced when data availability is limited.

Panel (b) of Fig.~\ref{fig:conventional_fail} provides a conceptual view of the key principle underlying statistically valid hyperparameter selection.
Instead of attempting to directly minimize the empirical risk, these methods fix a tolerated risk level and aim to identify hyperparameters whose true risk is not likely to exceed this target (dashed black line).
Formally, the goal is to control the probability that the true risk exceeds the tolerated threshold.
This naturally leads to a hypothesis testing formulation in which, for each hyperparameter $\lambda$, we test whether the true risk $R(\lambda)$ is above the specified level using statistical tools such as p-values or e-values \citep{lehmann2005testing, angelopoulos2023conformal}.

As highlighted by the simplified illustration in Fig. \ref{fig:conventional_fail}(b), reliable hyperparameter selection must account for the inherent uncertainty on the true risk that is associated with the empirical estimate (shaded red area). Intuitively, only the hyperparameters that, even under the worst-case true risk within the range of plausible values (upper boundary of the red area), yield a true risk below the target value (dashed black line) can be considered as reliable given the available information. This way, the set of selected hyperparameters (denoted as $\hat{\Lambda}$ in the figure) forms a certified subset of hyperparameters for which the true risk lies below the target level. Thus, while the procedure may be conservative and not recover all safe hyperparameters, it guarantees that every selected configuration satisfies the prescribed risk requirement with high probability.

\begin{figure}[t]
    \centering

    \begin{subfigure}[t]{0.48\linewidth}
        \centering
        \includegraphics[width=\linewidth]{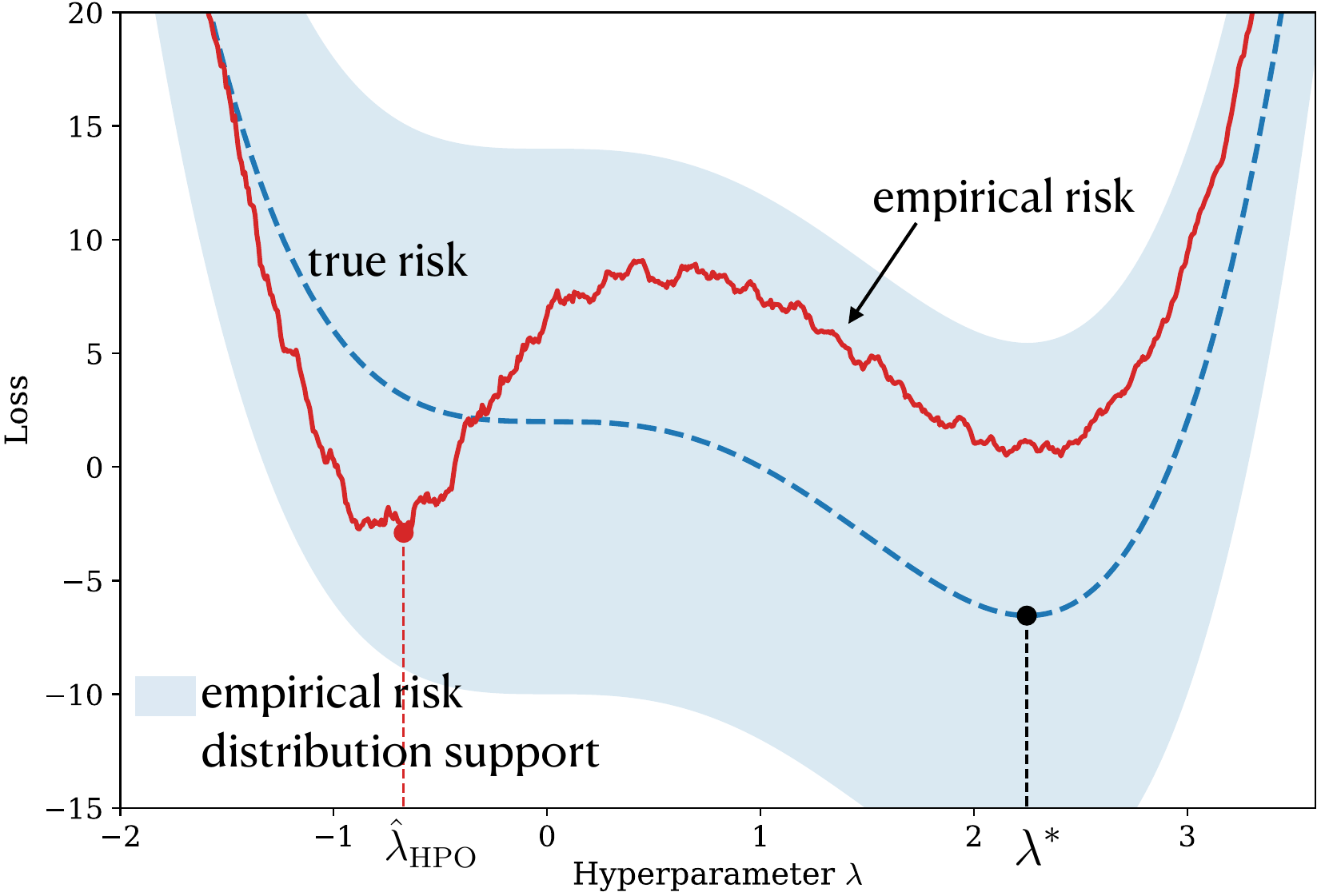}
        \caption{}
        \label{fig:conventional_fail_a}
    \end{subfigure}
    \hfill
    \begin{subfigure}[t]{0.48\linewidth}
        \centering
        \includegraphics[width=\linewidth]{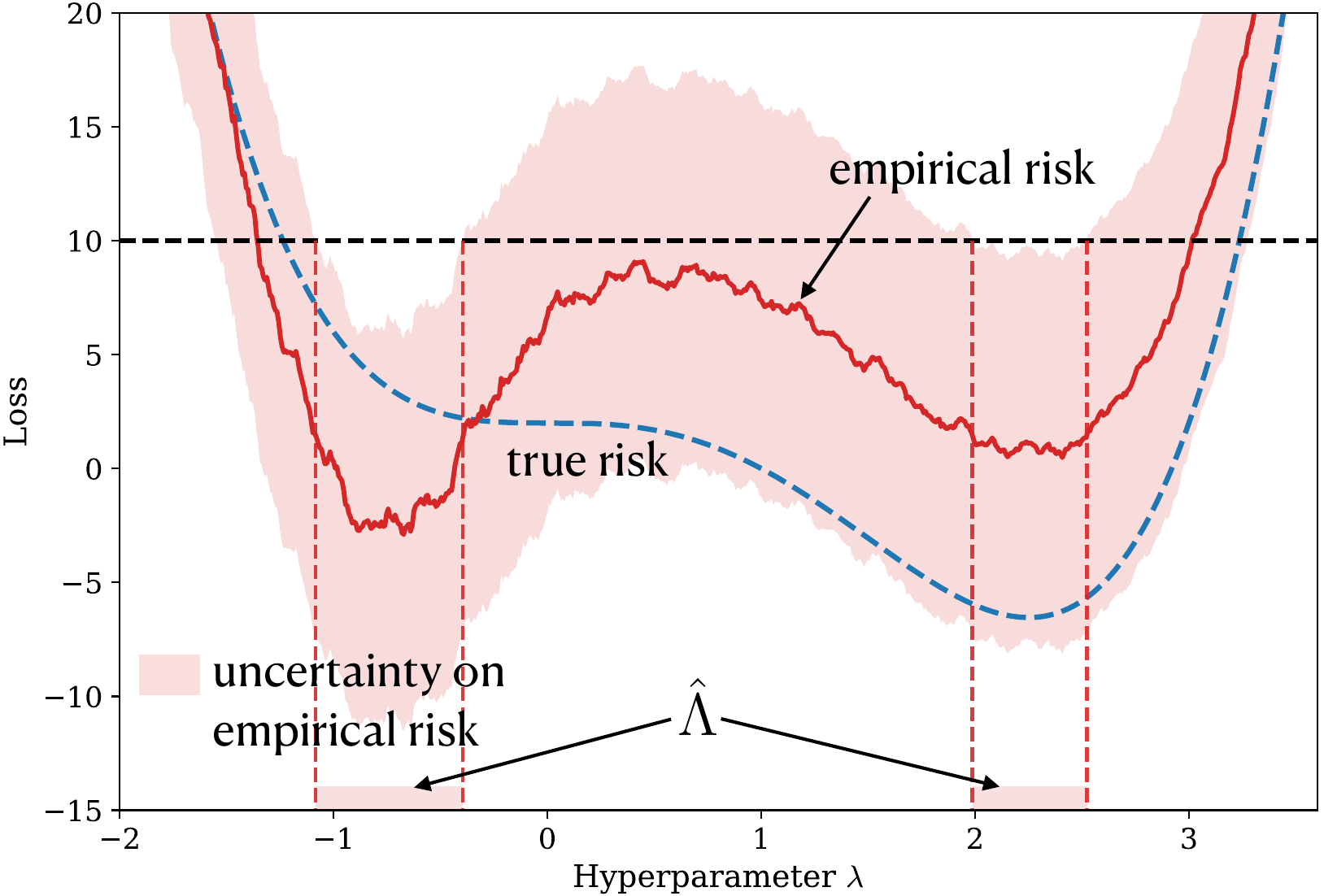}
        \caption{}
        \label{fig:conventional_fail_b}
    \end{subfigure}
    \caption{
    (a) Empirical versus true risk as a function of the hyperparameter $\lambda$.
    The blue dashed curve shows the true risk $R(\lambda)$, while the red curve shows a realization of the empirical risk $\widehat{R}(\lambda)$.
    The shaded region depicts the support of empirical risks across data realizations.
    Minimizing empirical risk selects $\hat{\lambda}_{\mathrm{HPO}}$, which may differ from the true minimizer $\lambda^\star$.
    (b) Risk-controlling hyperparameter selection.
    A tolerated risk level (horizontal dashed line) defines acceptable configurations.
    Statistical uncertainty bands are used to test, for each $\lambda$, whether the true risk exceeds the target.
    The highlighted set $\hat{\Lambda}$ denotes the hyperparameters selected by the statistical procedure, which form a certified subset of the truly safe configurations (those with $R(\lambda)$ below the threshold).
    }
    \label{fig:conventional_fail}
\end{figure}

It is important to emphasize that the illustration in Fig. \ref{fig:conventional_fail}(b) hides several important challenges and should be taken as suggestive rather than prescriptive. These challenges, which will be addressed throughout this monograph, are as follows:

\begin{itemize}

\item \textbf{User-defined failure probability:} A reliable hyperparameter selection procedure must provide a mechanism to control the probability of selecting an unreliable hyperparameter.

\item \textbf{Set-wide error control:} The probabilistic guarantee must apply at the level of the entire set ($\hat{\Lambda}$) of selected hyperparameters. For instance, one can impose the constraint that no hyperparameter in set $\hat{\Lambda}$ is unreliable with the given target probability, or that the expected fraction of unreliable hyperparameters in set $\hat{\Lambda}$ is smaller than a target value.

\end{itemize}

The next section formalizes these requirements.

\section{Problem Definition}
\label{sec:ch1_problem}

Informally, a reliable hyperparameter choice is one that performs well enough, according to a user-defined threshold, with a low probability of violating critical requirements. Reliability is therefore a statement about the future behavior of the deployed system, under uncertainty arising from data variability and environment stochasticity. This perspective, detailed in this section, aligns with recent calls to move from performance-centric to reliability- and safety-centric evaluation of AI systems \citep{amodei2016concrete}.

\subsection{Hyperparameter Selection Domain}

Throughout this monograph, we use the symbol $\lambda$ to denote a hyperparameter or a vector of hyperparameters
controlling the behavior of a machine learning system. In practice, hyperparameters are often selected from a finite or discretized set $\Lambda = \{\lambda_1, \ldots, \lambda_K\}$ of $K = |\Lambda|$ candidates. This pre-selection step can be implemented using any of the HPO methods reviewed in Sec.~\ref{sec:background} from a possibly continuous hyperparameter space.

\subsection{Reliability Requirements}

To determine when a hyperparameter is
acceptable or reliable, we start by specifying an application-specific loss measure, such as prediction error, constraint violation rate, or system-level cost. A hyperparameter is then declared to be reliable if a chosen statistic of the loss, such as its expectation or a tail probability, lies below a user-defined threshold $\alpha$.

To elaborate, for a generic prediction problem, consider a random test point $(X,Y)$ drawn from the data distribution, where $X$ denotes the input features and $Y$ is the ground-truth label. For a given hyperparameter $\lambda$, let $L_\lambda(X,Y)$ denote the incurred loss on a single test point. Examples include the $0$--$1$ loss
\begin{equation}
\label{eq:example_loss}
    L_\lambda(X,Y) \;=\; \mathbb{I}\!\big\{ \hat{Y}_\lambda(X) \neq Y \big\},
\end{equation}
where $\hat{Y}_\lambda(X)$ is the predictor deployed under hyperparameter $\lambda$, and $\mathbb{I}\{\cdot\}$ denotes the indicator function, which equals $1$ if its argument is true and $0$ otherwise.

A typical target statistic is the average of the loss (\ref{eq:example_loss}), i.e.,
\begin{equation}
    R(\lambda) = \mathbb{E}\left[L_\lambda(X,Y)\right].
\end{equation}
Other options encompass quantiles, such as the median.

Having defined a risk function $R(\lambda)$, a hyperparameter $\lambda$ is said to be reliable if it is able to control the risk $R(\lambda)$ below a user-specified target level $\alpha$, i.e.,
\begin{equation}
\label{eq:risk_requirement}
   \lambda\;\text{reliable}\;\Leftrightarrow\; R(\lambda) = \mathbb{E}\!\left[ L_\lambda(X,Y) \right] \;\le\; \alpha.
\end{equation}
Importantly, reliability is a property of the deployed system and of the underlying data distribution, not of a finite dataset.

A hyperparameter selection procedure observes a finite calibration dataset and outputs a candidate $\hat{\lambda} \in \Lambda$. Because the data are random, the selected hyperparameter is also random. For a selected hyperparameter $\hat{\lambda}$, the basic guarantee we seek is that the hyperparameter $\hat{\lambda}$ is reliable with probability no smaller than a target $1-\delta$, i.e.,
\begin{equation}
\label{eq:ch1_guarantee}
    \mathbb{P}\!\big[ \hat{\lambda} \text{ is reliable} \big] = \mathbb{P}\!\big[ R(\hat{\lambda})\leq \alpha \big] \;\ge\; 1 - \delta,
\end{equation}
where the probability is over the randomness in the data used for selection.

We emphasize that the guarantee \eqref{eq:ch1_guarantee} does \emph{not} claim that the selected hyperparameter is optimal; or that failures never occur; or that performance is uniform across all possible conditions. However, it does claim that the act of selecting an unreliable hyperparameter is itself a rare event; and that the frequency of this event is explicitly quantified by the failure probability $\delta$.

Fig.~\ref{fig:ch1_reliability_guarantee} provides a visual illustration of the reliability guarantee \eqref{eq:ch1_guarantee} using a simple real-data classification experiment.
We consider a fixed dataset and a finite grid of regularization hyperparameters for an $\ell_2$-regularized logistic regression model \citep{friedman2010regularization}.
For each repetition, the available data are randomly split into training and calibration sets.
A hyperparameter $\hat{\lambda}$ is selected based solely on the calibration data and the resulting model is evaluated on a fixed held-out test set to estimate its deployment risk $R(\hat{\lambda})$.
Repeating this procedure over independent calibration realizations yields a distribution of test risks.

The blue curve corresponds to hyperparameters selected using a statistically valid selection rule that accounts for uncertainty due to finite calibration data (see Fig. \ref{fig:conventional_fail}).
While individual realizations may exceed the target risk level $\alpha = 0.2$, the probability of such violations is controlled.
In particular, the shaded region corresponding to events where the risk exceeds the maximum tolerance $\alpha$, i.e., $R(\hat{\lambda}) > \alpha$, occupies at most a $\delta = 0.1$ fraction of the distribution, illustrating that guarantee~(\ref{eq:ch1_guarantee}) holds.

For comparison, the orange curve corresponds to hyperparameters selected by grid search, which minimizes the empirical risk on the calibration dataset without correcting for selection uncertainty.
Because the search is conducted over many candidate configurations, as illustrated in Fig. \ref{fig:conventional_fail}, this approach can select hyperparameters whose apparent calibration performance is overly optimistic.
As a result, the violation probability can substantially exceed the prescribed tolerance $\delta = 0.1$.
This contrast highlights the importance of statistical correction in reliable hyperparameter selection.

\begin{figure}[h!]
    \centering
    \includegraphics[width=0.8\linewidth]{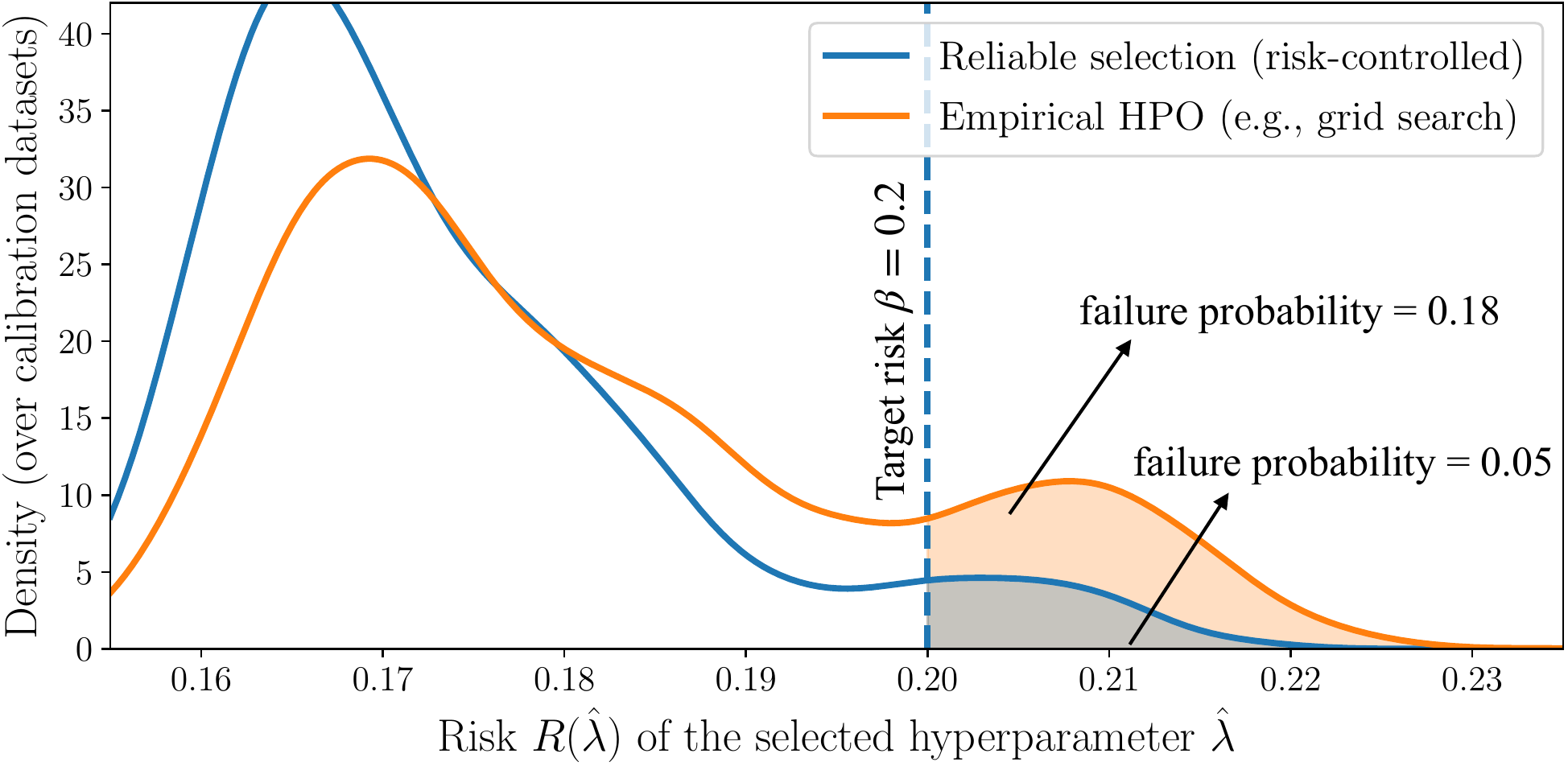}
    \caption{Distribution of the risk $R(\hat{\lambda})$ for the selected hyperparameter $\hat{\lambda}$ over repeated training-calibration data splits.
    The dashed line marks the target risk $\alpha = 0.2$.
    The blue curve corresponds to a statistically valid selection rule, whose violation probability
    $\mathbb{P}(R(\hat{\lambda}) > \alpha)$ is controlled below threshold $\delta = 0.1$ as per (\ref{eq:ch1_guarantee}).
    The orange curve corresponds to conventional HPO based on empirical risk minimization,
    which exhibits a larger violation probability.}
    \label{fig:ch1_reliability_guarantee}
\end{figure}

The remainder of this monograph is concerned with developing methods that achieve such guarantees in practice.

\section{Summary of Notation and Acronyms}
\label{sec:ch1_notation}

Table \ref{tab:notation_full} summarizes the notation used in the monograph, while Table \ref{tab:acronyms} reports the main acronyms found in the text.

\begin{longtable}{p{0.20\linewidth}p{0.72\linewidth}}
\caption{Summary of mathematical notation used throughout the monograph.}
\label{tab:notation_full} \\
\textbf{Symbol} & \textbf{Definition} \\
\hline
\endfirsthead
\textbf{Symbol} & \textbf{Definition} \\
\hline
\endhead
\hline
\multicolumn{2}{r}{\textit{Continued on next page}} \\
\endfoot
\hline
\endlastfoot
$\lambda$ & Hyperparameters (vector-valued) \\
$\Lambda$ & Finite candidate set of hyperparameters, $\Lambda = \{\lambda_1, \dots, \lambda_{|\Lambda|}\}$ \\
$K$ & Number of candidate hyperparameters, $K = |\Lambda|$ \\
$\mathcal{X}$ & Input space \\
$\mathcal{Y}$ & Output/label space \\
$(X, Y)$ & Random test point (features, outcome) drawn from distribution $\mathbb{P}$ \\
$Z$ & Random test point drawn from distribution $\mathbb{P}$ (e.g., $Z = (X,Y)$)\\
$\mathbb{P}$, $\mathbb{P}(\cdot)$, $\mathbb{P}(X,Y)$, $\mathbb{P}(Z)$  & Unknown data-generating distribution; also used as $\mathbb{P}(\cdot)$ to denote the probability of an event under this distribution \\
$f_\lambda$ & Pre-trained model operating under hyperparameter $\lambda$, $f_\lambda: \mathcal{X} \to \mathcal{Y}$ \\
$L_\lambda(X,Y)$ & Per-sample loss under hyperparameter $\lambda$, $L_\lambda: \mathcal{X} \times \mathcal{Y} \to [0,1]$ \\
$R(\lambda)$ & True risk, $R(\lambda) = \mathbb{E}[L_\lambda(X,Y)]$ \\
$\widehat{R}_n(\lambda)$ & Empirical risk on $n$ calibration samples \\
$\mathcal{D}^{\mathrm{cal}}$ & Calibration dataset $\mathcal{D}^{\mathrm{cal}} = \{(X_i, Y_i)\}_{i=1}^n$ \\
$n$ & Number of calibration samples \\
$\alpha$ & Reliability threshold; $\lambda$ is reliable $\Leftrightarrow$ $R(\lambda) \le \alpha$ (single-objective) \\
$\delta$ & User-specified failure probability (outage rate) \\
$\hat{\Lambda}$ & Certified subset of selected hyperparameters, $\hat{\Lambda} \subseteq \Lambda$ \\
$\hat{\lambda}$ & Single deployed hyperparameter, $\hat{\lambda} \in \hat{\Lambda}$ \\
$\mathcal{H}_k$ & Null hypothesis for candidate $\lambda_k$: $R(\lambda_k) > \alpha$ \\
$\mathcal{K}_0$ & Set of indices of true null hypotheses \\
$\hat{\mathcal{K}}$ & Set of rejected hypothesis indices (discoveries) \\
$p_k$ & P-value for null hypothesis $\mathcal{H}_k$ \\
$e_k$ & E-value for null hypothesis $\mathcal{H}_k$ \\
$\mathbb{I}\{\cdot\}$ & Indicator function (equals 1 if argument is true, 0 otherwise) \\
$(\cdot)_+$ & Positive part, $(x)_+ = \max\{x, 0\}$ \\
$\mathcal{R}(\lambda)$ & General risk functional, $\mathcal{R}(\lambda) = \Phi(\mathbb{P}_\lambda)$ \\
$R_q(\lambda)$ & Quantile risk: $(1-q)$-quantile of the distribution of loss $L_\lambda$ \\
$I(A;B)$ & Mutual information between random variables $A$ and $B$ \\
$R_l(\lambda)$ & Risk function for objective $l$ in multi-objective setting \\
$\alpha$ & Vector of risk thresholds $(\alpha_1, \dots, \alpha_{L_C})$ in multi-objective setting (Ch.~\ref{chapter:multi_objective}) \\
$L_C$ & Number of reliability (constrained) objectives \\
$\Lambda^{\mathrm{Par}}$ & Pareto-optimal subset of hyperparameters \\
\end{longtable}

\begin{longtable}{p{0.20\linewidth}p{0.72\linewidth}}
\caption{List of acronyms used throughout the monograph.}
\label{tab:acronyms} \\
\textbf{Acronym} & \textbf{Definition} \\
\hline
\endfirsthead
\textbf{Acronym} & \textbf{Definition} \\
\hline
\endhead
\hline
\multicolumn{2}{r}{\textit{Continued on next page}} \\
\endfoot
\hline
\endlastfoot
AI   & Artificial Intelligence \\
BH   & Benjamini--Hochberg (FDR-controlling procedure) \\
BY   & Benjamini--Yekutieli (FDR-controlling procedure under arbitrary dependence) \\
CGF  & Cumulant Generating Function \\
DAG  & Directed Acyclic Graph \\
E-HPO & Empirical Hyperparameter Optimization \\
FDR  & False Discovery Rate \\
FDP  & False Discovery Proportion \\
FST  & Fixed Sequence Testing \\
FWER & Family-Wise Error Rate \\
HPO  & Hyperparameter Optimization \\
IB   & Information Bottleneck \\
IB-LTT & Information Bottleneck Learn-Then-Test \\
KPI  & Key Performance Indicator \\
LLM  & Large Language Model \\
LTT  & Learn-Then-Test \\
MGF  & Moment Generating Function \\
MHT  & Multiple Hypothesis Testing \\
O-RAN & Open Radio Access Network \\
PCA  & Principal Component Analysis \\
PT   & Pareto Testing \\
QLTT & Quantile Learn-Then-Test \\
QoS  & Quality of Service \\
RG   & Reliability Graph \\
RG-PT & Reliability Graph-Based Pareto Testing \\
UE   & User Equipment \\
6G   & Sixth Generation (wireless networks) \\
\end{longtable}

\section{Scope and Organization of the Monograph}

As mentioned, the remainder of this monograph develops the statistically valid hyperparameter selection framework outlined in Sec.~\ref{sec:ch1_problem} and progressively broadens its scope along four directions: more general notions of risk, multiple simultaneous objectives, adaptive and sequential evaluation budgets, and certification with limited labeled data.

\textbf{Chapter~\ref{chapter:mht}} introduces the core methodology, namely the learn-then-test (LTT) framework of \citep{angelopoulos2025learn}, which formulates reliable hyperparameter selection as a multiple hypothesis testing (MHT) problem. Each candidate hyperparameter is associated with the null hypothesis that its true risk exceeds the target level, and a finite-sample certified set is produced by any MHT procedure that controls a global error criterion such as the family-wise error rate (FWER) or the false discovery rate (FDR). The chapter develops the supporting statistical machinery, namely p-values and e-values, and shows how concentration inequalities and testing-by-betting arguments translate empirical risk estimates into valid test statistics.

\textbf{Chapter~\ref{chapter:applications}} illustrates the LTT framework on two representative applications: image classification on Fashion-MNIST \citep{xiao2017fashion} and packet scheduling for a learning-based wireless scheduler \citep{de2020radio, Nokia}. In both cases, conventional empirical hyperparameter optimization is shown to violate the prescribed reliability constraint at rates substantially exceeding the user-specified outage probability, while LTT controls the violation rate as guaranteed by theory.

\textbf{Chapter~\ref{chapter:beyond_average}} extends LTT beyond the average risk to more general reliability functionals. The chapter presents a unified inversion-based recipe for constructing valid p-values from one-sided confidence bounds, and instantiates it for two important non-average criteria: quantile risk control, which yields the quantile LTT (QLTT) procedure of \citep{farzaneh2024quantile} and provides protection against tail events; and information-theoretic relevance constraints, which yields the IB-LTT procedure of \citep{farzaneh2025ibmht} for the information bottleneck problem.

\textbf{Chapter~\ref{chapter:multi_objective}} addresses settings in which several reliability constraints must be enforced simultaneously while best-effort secondary objectives are optimized. The chapter develops Pareto testing (PT) \citep{laufer-goldshtein2023efficiently}, which restricts hypothesis testing to the empirical Pareto frontier and applies fixed sequence testing along the resulting ordering, and reliability graph-based Pareto testing (RG-PT) \citep{farzaneh2025multiobjective}, which encodes prior structural knowledge through a directed acyclic graph and certifies hyperparameters via the DAGGER algorithm of \citep{ramdas2017dagger}.

\textbf{Chapter~\ref{chapter:adaptive}} relaxes the batch-calibration assumption of the previous chapters, allowing calibration data to arrive sequentially and acquisition decisions to depend on the evidence accumulated so far. Building on e-processes and Ville's inequality, the chapter develops adaptive LTT (aLTT) \citep{zecchin2024adaptive}, which combines testing-by-betting with $\varepsilon$-greedy acquisition to identify reliable hyperparameters with substantially fewer evaluations than batch LTT, while preserving FWER and FDR guarantees at every, possibly data-dependent, stopping time.

\textbf{Chapter~\ref{chapter:autoeval}} considers settings where ground-truth labels are expensive, but cheap autoevaluated labels are available, as in LLM-as-a-judge pipelines \citep{gu2024survey}. Combining the testing-by-betting machinery of Chapter~\ref{chapter:adaptive} with prediction-powered inference \citep{angelopoulos2023prediction}, the chapter describes the R-AutoEval method of \citep{einbinder2024semi} and the adaptive R-AutoEval+ method of \citep{park2025adaptive}, the latter of which adaptively re-weights a grid of reliance factors and provably matches the better of full-real and full-autoevaluated testing in sample complexity.

\textbf{Chapter~\ref{chapter:conclusions}} concludes the monograph and outlines open research directions.

\textbf{Appendix~\ref{app:evidence}} collects background material on statistical evidence. It provides a self-contained treatment of p-values and e-values, including their definitions, constructions from concentration bounds and randomization tests, and the key compositional and anytime-validity properties that distinguish them.

\textbf{Appendix~\ref{app:ch2_pvalue_constructions}} collects technical details on p-value and e-value constructions complementing Chapter~\ref{chapter:mht}: a unified moment generating function viewpoint covering bounded, sub-Gaussian, and sub-exponential losses; a proof of superuniformity for the Hoeffding-based p-value; and refined finite-sample constructions including exact binomial and variance-sensitive Bernstein p-values.

\textbf{Appendix~\ref{app:additional_mht}} gathers additional multiple testing procedures: Holm, Hochberg, \v{S}id\'{a}k, and Westfall--Young FWER procedures, and a proof of the FDR guarantee for the e-BH procedure.



\chapter{Hyperparameter Selection via Multiple Hypothesis Testing}
\label{chapter:mht}

This chapter introduces a general framework for risk-controlling hyperparameter selection. As illustrated in Fig. \ref{fig:ch1_reliability_guarantee}, the goal is to identify a configuration $\hat{\lambda}$ for which the true risk $R(\hat{\lambda})$ is guaranteed to lie below a user-chosen threshold with high probability with respect to the randomness of the calibration data. The framework follows the post-training calibration methodology introduced by \citep{angelopoulos2025learn}. This approach is also referred to as LTT, since it is based on training and calibration on separate datasets \citep{grunwald2020safetest, vovk2021evalues, chakraborty2026comparing}.

The chapter is organized as follows. Sec.~\ref{sec:mht_risk_control} introduces the risk-control objective and empirical risk estimation. Sec.~\ref{sec:ch2_mht} reviews the multiple hypothesis testing framework and the family-wise error rate. Sec.~\ref{sec:ch2_risk_control} shows how to connect MHT to hyperparameter certification and proves the main certified-set guarantee. Sec.~\ref{sec:ch2_pval} introduces p-values as test statistics and derives the Hoeffding-based p-value for bounded losses. Sec.~\ref{sec:ch2_eval} introduces e-values as an alternative evidence measure and describes their key properties. Sec.~\ref{sec:FDR} develops false discovery rate control via the BH and e-BH procedures, which are less conservative alternatives to FWER control. Sec.~\ref{sec:ch2_summary} summarizes the chapter. Technical material on p-value and e-value constructions from moment generating function bounds is collected in Appendix~\ref{app:ch2_pvalue_constructions}.

\section{Risk Control}
\label{sec:mht_risk_control}

Throughout this monograph, as discussed in Sec. \ref{sec:ch1_problem}, we fix a finite candidate set
\begin{equation}
\Lambda = \{\lambda_1,\dots,\lambda_{K}\},
\end{equation}
which may arise from any upstream hyperparameter search or model design process. Each hyperparameter $\lambda\in\Lambda$ determines the operation of a pre-trained model
\begin{equation}
f_\lambda:\ \mathcal{X}\to\mathcal{Y},
\end{equation}
where $\mathcal{X}$ denotes the input space and $\mathcal{Y}$ denotes the model output space. The model $f_\lambda$ may be predictive, generative, or it may prescribe a control policy. In the former case, the output space $\mathcal{Y}$ may encompass point estimates for the given target variables, probability distributions on the target variable space, or structured predictions such as class labels in multi-class classification \citep{hastie2009elements}, and bounding boxes or segmentation masks in object detection \citep{ren2015faster, he2017mask}. With generative models $f_\lambda$, the output space $\mathcal{Y}$ may include images \citep{ho2020denoising}, videos \citep{ho2022video}, or text \citep{vaswani2017attention}, often in the form of token sequences. Finally, control policies $f_\lambda$ output action sequences \citep{sutton1998reinforcement}.

The key quantity of interest is the average risk of model $f_\lambda$, which is defined as the expectation of a loss function under the (unknown) data-generating distribution. Concretely, let $(X,Y)\sim \mathbb{P}$ denote a generic test-time input-output pair, and let
\begin{equation}
L_\lambda:\ \mathcal{X}\times\mathcal{Y}\to[0,1]
\end{equation}
be a bounded loss function that encodes a deployment-relevant notion of (negatively oriented) performance. Note that any bounded performance measure can be constrained within the standardized interval $[0,1]$ by rescaling. Furthermore, unbounded losses can be handled via asymptotic or truncated constructions as discussed in \citep{angelopoulos2025learn}.

We define the risk of hyperparameter $\lambda$ as
\begin{equation}
\label{eq:ch2_risk_def}
R(\lambda)\ =\ \mathbb{E}\!\left[L_\lambda\!\big(X,Y\big)\right],
\end{equation}
where the expectation is taken over the unknown test-time data distribution $\mathbb{P}$.

The user specifies a target level $\alpha\in(0,1)$ for the average risk (\ref{eq:ch2_risk_def}), as well as a failure probability $\delta\in(0,1)$. The goal is to select a (data-dependent) hyperparameter $\hat{\lambda}\in\Lambda$ such that the deployed model $f_{\hat{\lambda}}$ satisfies the following risk control guarantee.
\begin{definition}
\label{def:ch2_rcp}
A hyperparameter choice $\hat{\lambda}\in\Lambda$ is said to be \emph{$(\alpha,\delta)$-risk-controlling} if the inequality
\begin{equation}
\label{eq:ch2_rcp_def}
\mathbb{P}( R(\hat{\lambda}) \le \alpha )\ \ge\ 1-\delta
\end{equation}
holds, where the probability $\mathbb{P}(\cdot)$ is taken over the randomness in the calibration data used to optimize the hyperparameter $\hat{\lambda}$.
\end{definition}

The key challenge in ensuring condition (\ref{eq:ch2_rcp_def}) is that the risk $R(\lambda)$ depends on the unknown distribution $\mathbb{P}$, while hyperparameters $\hat{\lambda}$ must be chosen based on a finite calibration dataset. To elaborate, let
\begin{equation}
\mathcal{D}^{\mathrm{cal}} = \{(X_1,Y_1),\dots,(X_n,Y_n) \},
\end{equation}
where $(X_i, Y_i)\overset{\text{i.i.d.}}{\sim}\ \mathbb{P}$ for all $i\in\{1,\ldots, n\}$, denote a calibration dataset. For each candidate hyperparameter $\lambda\in\Lambda$, define the corresponding empirical risk
\begin{equation}
\label{eq:ch2_emp_risk}
\widehat{R}_n(\lambda)\ =\ \frac{1}{n}\sum_{i=1}^n L_\lambda(X_i,Y_i).
\end{equation}
Conventional E-HPO chooses the hyperparameter $\lambda$ that minimizes the empirical estimates $\widehat{R}_n(\lambda)$, i.e.,
\begin{equation}
\label{alg:ch2_ehpo}
\hat{\lambda}_\text{HPO} = \arg\min_{\lambda\in \Lambda}\widehat{R}_n(\lambda).
\end{equation}
As discussed in Sec. \ref{sec:Ch1_no_reliability}, E-HPO does not provide any reliability guarantees in general (see Fig. \ref{fig:conventional_fail}).

\section{Multiple Hypothesis Testing}
\label{sec:ch2_mht}

Following \citep{angelopoulos2025learn}, we address this problem by treating the risk-control objective as a multiple hypothesis testing (MHT) problem.

To explain MHT, consider the standard A/B testing methodology applied by an online retailer, as illustrated in Fig.~\ref{fig:mht_illustration}. The online retailer wishes to explore possible modifications relative to the current website configurations. For example, it may experiment with a new homepage layout, different product image sizes, alternative call-to-action button colors, revised pricing display formats, and personalized recommendation banners.

\begin{figure}[t]
    \centering
    \includegraphics[width=0.9\linewidth]{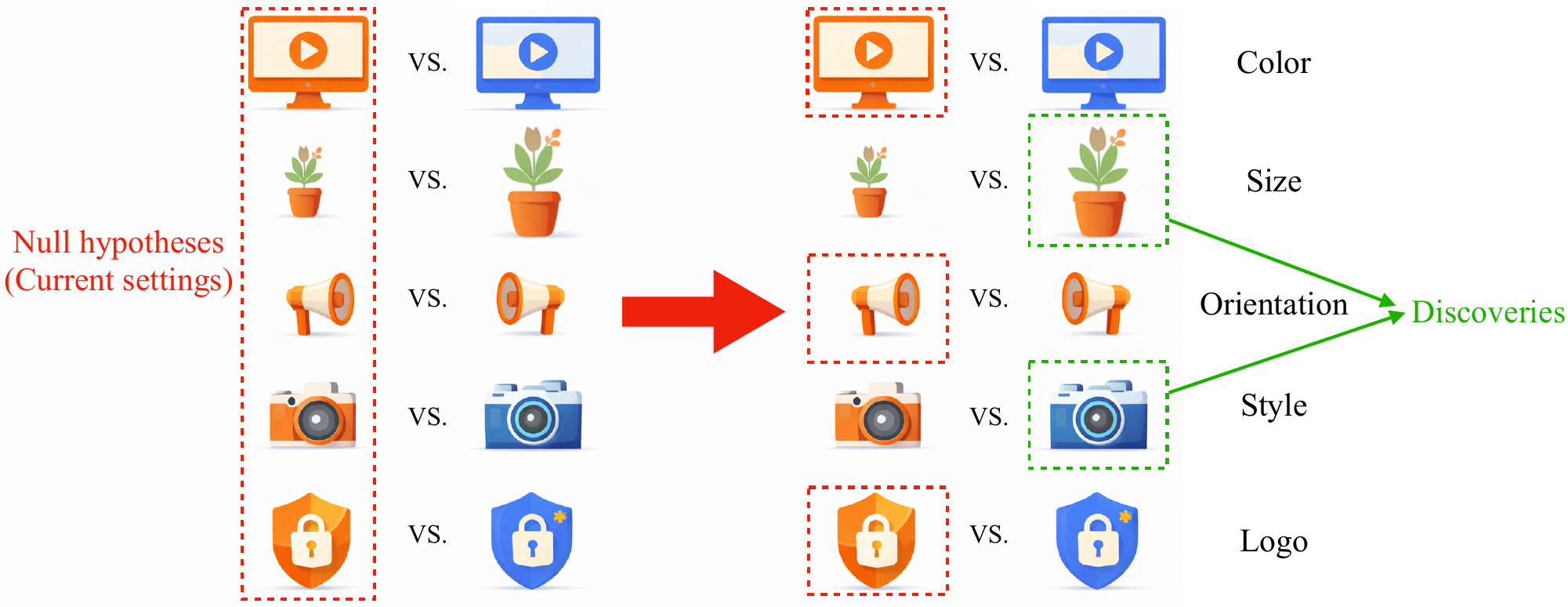}
    \caption{Illustration of multiple hypothesis testing (MHT) for the example of A/B testing implemented by an online retailer. Multiple changes to the current settings of the retailer's website are considered via separate tests, and a number of them are discovered as effective changes.}
    \label{fig:mht_illustration}
\end{figure}

For each proposed modification, the retailer formulates a baseline claim, called the \emph{null hypothesis}, stating that the modification does not improve conversion relative to the current configuration. The retailer then collects user-interaction data, and evaluates whether the observed improvement is strong enough to contradict this baseline claim. When the data provide sufficient statistical evidence against the null hypothesis, the retailer declares a \emph{discovery}, meaning that the modification is judged to produce a genuine improvement.

However, when many modifications are tested simultaneously, the probability of making at least one incorrect discovery increases with the number of experiments. Even if each individual experiment has a small probability $\delta$ of producing a false discovery, running $K$ experiments inflates the overall chance that at least one ineffective modification is incorrectly declared successful. In fact, under independence, this probability equals
$1 - (1-\delta)^K \approx K\delta$,
which grows with $K$. For instance, if $\delta = 0.01$ and $K=50$, the probability of making at least one incorrect claim is $(1-0.99^{50}) \approx 0.395$.

This observation motivates the adoption of MHT tools. As illustrated in Fig. \ref{fig:mht_illustration}, in MHT, we are generally faced with a set $\mathcal{H} = \{\mathcal{H}_1, \ldots, \mathcal{H}_K\}$ of hypotheses that need to be tested simultaneously. An MHT procedure returns a subset $\hat{\mathcal{K}}\subseteq \{1, \ldots, K\}$ of the hypotheses as discoveries, i.e., as null hypotheses that have been \textit{rejected} by the MHT procedure.

Rather than controlling the error rate separately for each hypothesis, MHT controls a \emph{global error criterion} across the entire family of tests. One such criterion is the \emph{family-wise error rate} (FWER), defined as the probability of making even a single incorrect claim among all tested modifications. Procedures that control the FWER ensure that, with high probability, none of the selected discoveries is actually ineffective.

To formalize this setting, for each hypothesis $\mathcal{H}_k \in \mathcal{H}$, MHT methods construct a \textit{test statistic} $T_k$ based on the available data.
The statistic $T_k$ is ideally designed to measure how surprising the data are under the null $\mathcal{H}_k$, so that large values provide evidence against $\mathcal{H}_k$ (see Sec. \ref{sec:ch2_pval}, Sec. \ref{sec:ch2_eval}, and Appendix \ref{app:evidence} for more details on how to construct valid test statistics).
An individual hypothesis $\mathcal{H}_k$ is rejected whenever its test statistic $T_k$ exceeds a suitable rejection threshold.

When testing multiple hypotheses simultaneously, the central statistical risk is that of \textit{false discovery}, i.e., of rejecting a hypothesis $\mathcal{H}_k$ even though it is in fact true. Let
\begin{equation}
\label{eq:ch2_true_nulls}
\mathcal{K}_0 \;=\; \{\, k \;:\; \mathcal{H}_k \text{ is true} \,\}
\end{equation}
denote the indices of all true null hypotheses. A false discovery occurs if hypothesis $\mathcal{H}_k$ is rejected for some $k \in \mathcal{K}_0$, i.e., if the intersection between the set of rejected hypotheses (discoveries) $\hat{\mathcal{K}}$, and the set of true nulls $\mathcal{K}_0$, is non-empty, i.e., $\hat{\mathcal{K}}\cap\mathcal{K}_0 \neq \emptyset$.

To ensure reliability after selection, we seek to control the probability of making any such false discovery.
This leads naturally to the notion of FWER, which is formally defined as
\begin{align}
\label{eq:ch2_fwer_def}
\mathrm{FWER}
&=
\mathbb{P}\!\left(
\exists\, k\in\mathcal{K}_0 \text{ such that } \mathcal{H}_k \text{ is rejected}
\right)\notag\\
& = \mathbb{P}(\hat{\mathcal{K}}\cap\mathcal{K}_0\neq \emptyset).
\end{align}
i.e., the probability of having \emph{any} true null hypothesis among the rejected hypotheses in set $\hat{\mathcal{K}}$.

Fig.~\ref{fig:fwer} illustrates two multiple testing outcomes, which may correspond to two distinct realizations of the calibration dataset $\mathcal{D}^{\mathrm{cal}}$. In the case illustrated in the left panel, the set $\hat{\mathcal{K}}$ of rejected hypotheses contains no true null hypotheses, while in the second at least one true null hypothesis is erroneously rejected, contributing to an increase of the FWER. Controlling FWER at level $\delta$, i.e.,
\begin{equation}
\label{eq:fwer_condition}
\mathrm{FWER} \le \delta,
\end{equation}
guarantees that, with probability at least $1-\delta$, no true null hypothesis is rejected, as in the left panel of Fig. \ref{fig:fwer}.

\begin{figure}[t]
    \centering
    \includegraphics[width=0.95\linewidth]{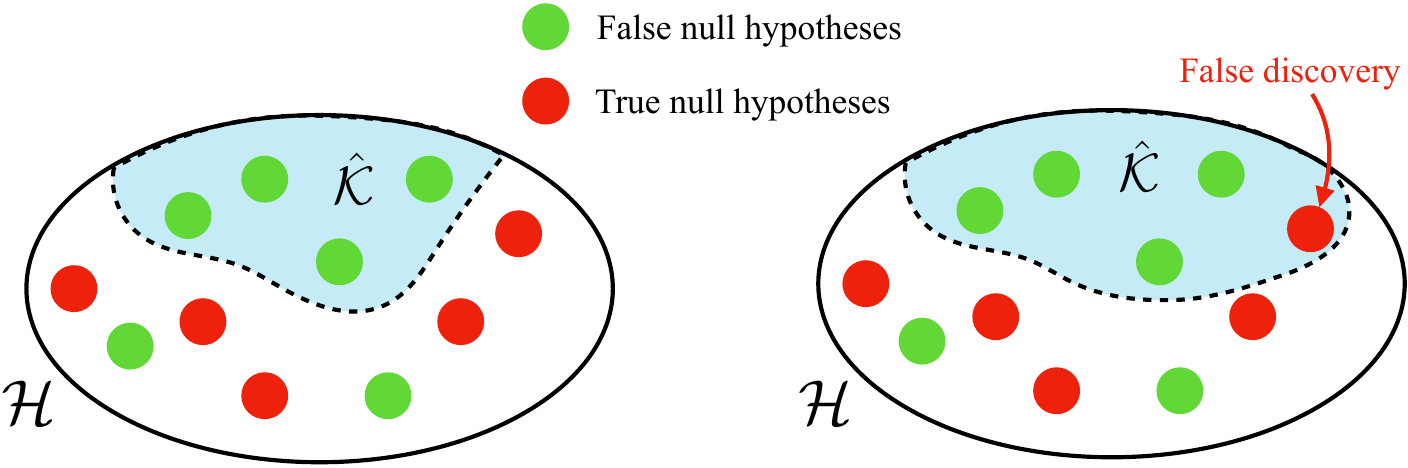}
    \caption{Schematic illustration of family-wise error rate (FWER) control in MHT.
    Each point represents a null hypothesis, with green markers denoting false null hypotheses
    and red markers denoting true null hypotheses.
    The shaded region corresponds to the subset $\hat{\mathcal{K}}$ of hypotheses rejected by a test, i.e., of discoveries.
    \emph{Left:} The rejected set contains no true null hypotheses, a condition that occurs with probability at least $1-\delta$ if the
    FWER condition (\ref{eq:fwer_condition}) is satisfied.
    \emph{Right:} The rejected set includes at least one true null hypothesis, an event that occurs with probability smaller than $\delta$ under the FWER guarantee (\ref{eq:fwer_condition}).}
    \label{fig:fwer}
\end{figure}

\begin{definition}[FWER-controlling procedure]
\label{def:ch2_fwer}
A selection procedure
\begin{equation}
\label{eq:fwer_controlling}
\mathcal{A}_{\text{FWER}}:\mathbb{R}^{K}\to 2^{K}
\end{equation}
that maps a collection of test statistics $(T_1,\dots,T_K)$ to a set of rejected hypotheses (discoveries), i.e.,
\begin{equation}
\hat{\mathcal{K}}=\mathcal{A}_{\text{FWER}}(T_1,\dots,T_K),
\end{equation}
is said to control the FWER at level $\delta \in (0,1]$ if the probability of rejecting any null hypothesis is smaller than $\delta$, i.e.,
\begin{equation}
\label{eq:ch2_fwer_set}
\mathbb{P}\!\left(
\mathcal{K}_0 \cap \hat{\mathcal{K}} = \emptyset
\right)
\;\ge\;
1-\delta.
\end{equation}
\end{definition}
In other words, an FWER-controlling procedure ensures that the entire rejected set consists only of false null hypotheses, with the exception of events having probability at most~$\delta$.

\section{Risk Control via Multiple Hypothesis Testing}
\label{sec:ch2_risk_control}

We now show how to connect MHT to the problem of reliable hyperparameter selection. To start, for each candidate hyperparameter $\lambda_k\in\Lambda$, we define the null hypothesis
\begin{equation}
\label{eq:ch2_null}
\mathcal{H}_k:\quad R(\lambda_k) > \alpha
\end{equation}
that the average risk exceeds the target level $\alpha$. Rejecting $\mathcal{H}_k$ therefore corresponds to certifying that hyperparameter $\lambda_k$ meets the risk requirement (\ref{eq:risk_requirement}).

Assume we have a valid test statistic $T_k$ for each null hypothesis $\mathcal{H}_k$ (see Sec. \ref{sec:ch2_pval} and Sec. \ref{sec:ch2_eval}), as well as an MHT procedure $\mathcal{A}_{\text{FWER}}$ that controls the FWER at level $\delta$. The following guarantee is immediate by the definition of FWER-controlling procedures.

\begin{theorem}
\label{thm:ch2_certified_set}
Any FWER-controlling procedure at level $\delta$ for the null hypotheses (\ref{eq:ch2_null}) returns a subset $\hat{\Lambda}\subseteq \Lambda$ of hyperparameters satisfying the reliability condition
\begin{equation}
\label{eq:ch2_set_guarantee}
\mathbb{P}\left(\sup_{\lambda\in\hat{\Lambda}} R(\lambda)\le \alpha\right)\ \ge\ 1-\delta.
\end{equation}
\end{theorem}
\begin{proof}
Let
\begin{equation}
E = \{\nexists \lambda \in \hat{\Lambda}:R(\lambda)>\alpha\} = \left\{\sup_{\lambda\in\hat{\Lambda}} R(\lambda)\le \alpha\right\}
\end{equation}
denote the event that all the hyperparameters in set $\hat{\Lambda}$ are reliable. By the definition (\ref{eq:ch2_null}) of the null hypotheses, this is equivalent to the event $E = \{\hat{\Lambda}\cap \Lambda_0 = \emptyset\}$, where $\Lambda_0$ is the subset of hyperparameters for which the null hypothesis (\ref{eq:ch2_null}) holds, i.e., $\Lambda_0 = \{\lambda\in\Lambda:R(\lambda)>\alpha\}$. Hence, by the FWER property \eqref{eq:ch2_fwer_set}, we have
\begin{equation}
\mathbb{P}\!\left(\sup_{\lambda\in\hat{\Lambda}} R(\lambda)\le \alpha\right) = \mathbb{P}(\hat{\Lambda}\cap \Lambda_0 = \emptyset)
\;\ge\; 1-\delta,
\end{equation}
concluding the proof.
\end{proof}

This result has the following important consequence.

\begin{corollary}
\label{cor:fwer}
Any specific hyperparameter $\hat{\lambda}\in \hat{\Lambda}$, possibly selected via a data-dependent way within the subset $\hat{\Lambda}$, is $(\alpha,\delta)$-risk-controlling in the sense of Definition~\ref{def:ch2_rcp} as long as the selection procedure controls the FWER at level $\delta$.
\end{corollary}

By Corollary~\ref{cor:fwer}, once the certified set $\hat{\Lambda}$ is obtained under FWER control, we may select any specific hyperparameter configuration $\hat{\lambda}\in\hat{\Lambda}$ that optimizes any secondary objectives such as accuracy, efficiency, or interpretability, without invalidating the risk guarantee (\ref{eq:ch1_guarantee}).

\begin{figure}[t]
    \centering
    \includegraphics[width=0.86\linewidth]{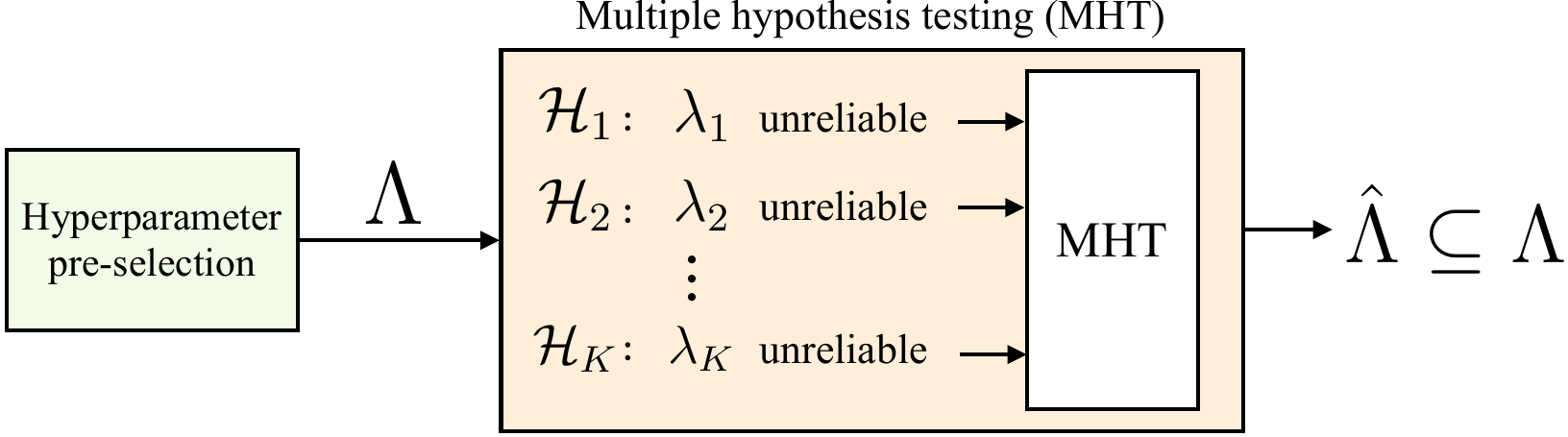}
    \caption{Reliable hyperparameter selection as MHT. Each hyperparameter $\lambda_k\in\Lambda$ is assigned to the null hypothesis $\mathcal{H}_k:R(\lambda_k)>\alpha$ that the hyperparameter is unreliable for a given target maximum risk level $\alpha$. An MHT procedure outputs subset of hyperparameters $\hat{\Lambda}$ that provides reliability guarantees on the number of unreliable hyperparameters (with $R(\lambda)>\alpha$) in subset $\hat{\Lambda}$.}
    \label{fig:ch2_mht_overview}
\end{figure}

Overall, as summarized in Fig. \ref{fig:ch2_mht_overview}, the workflow of reliable hyperparameter selection via MHT is as follows:
\begin{enumerate}
    \item For each candidate hyperparameter $\lambda_k \in \Lambda$, use the calibration data $\mathcal{D}^\mathrm{cal}$ to estimate the empirical risks $\widehat{R}_n(\lambda_k)$ for all $\lambda_k\in \Lambda$ via (\ref{eq:ch2_emp_risk}).

    \item Turn the empirical risks $\{\widehat{R}_n(\lambda_k)\}_{k = 1}^{K}$ in (\ref{eq:ch2_emp_risk}) into valid statistics, i.e., p-values (see Sec. \ref{sec:ch2_pval}) or e-values (see Sec. \ref{sec:ch2_eval}).
    \item Apply an MHT procedure to the derived statistics to obtain a certified set $\hat{\Lambda}\subseteq\Lambda$ of reliable hyperparameters.
    \item Select hyperparameter $\hat{\lambda}\in\hat{\Lambda}$ for deployment by following any, possibly data-dependent, criterion.
\end{enumerate}

\section{P-values}
\label{sec:ch2_pval}

This section describes standard statistics $T_k$ to test the null hypotheses $\mathcal{H}_k$, namely p-values. A self-contained broader background treatment of p-values, including their definition, power, and general constructions via the probability integral transform and concentration bounds, is provided in Appendix \ref{app:evidence}.

\subsection{Definition}

For each candidate hyperparameter $\lambda_k\in\Lambda$, a p-value $p_k$ for the null hypothesis
$\mathcal{H}_k$ in (\ref{eq:ch2_null}) is a function of the calibration data $\mathcal{D}^{\mathrm{cal}}$ such that, for all $u\in[0,1]$, the inequality
\begin{equation}
\label{eq:ch2_pvalue_validity}
\mathbb{P}\!\left( p_k \le u \mid \mathcal{H}_k\right) \;\le\; u
\end{equation}
holds, where the probability is taken over the randomness of the calibration dataset when conditioning on the null hypothesis $\mathcal{H}_k$. Condition (\ref{eq:ch2_pvalue_validity}) is known as the \emph{superuniformity} property. It states that, under the null $\mathcal{H}_k$, the p-value is stochastically no smaller than a uniform random variable in the interval $(0,1)$, i.e., small p-values occur no more often than they would under a uniform draw.

The CDF of a uniform distribution $\text{Unif}(0,1)$ is given by $F_U(u)=u$ for $u\in[0,1]$. Therefore, letting $F_{p_k}(u\mid \mathcal{H}_k)=\mathbb{P}(p_k\le u\mid \mathcal{H}_k)$ denote the CDF of the p-value $p_k$, the superuniformity property (\ref{eq:ch2_pvalue_validity}) corresponds equivalently to the inequality
\begin{equation}
F_{p_k}(u\mid \mathcal{H}_k)\le F_U(u) \quad \text{for all }u\in[0,1],
\end{equation}
so that the null CDF of a valid p-value lies everywhere below the line $F_U(u) = u$ in the interval $u\in [0,1]$. This property is illustrated in Fig. \ref{fig:superuniform}.

\begin{figure}[t]
    \centering
    \includegraphics[width=0.72\linewidth]{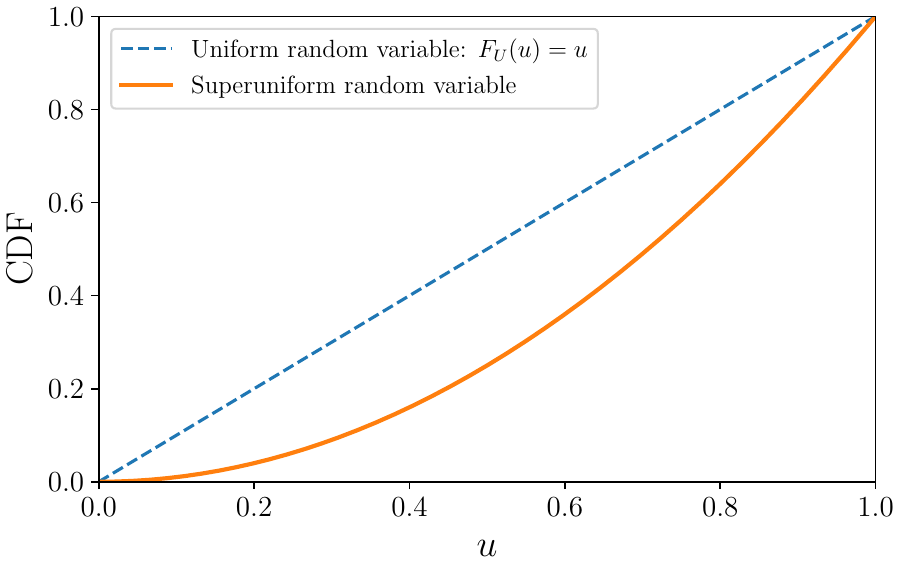}
    \caption{Illustration of the superuniformity property (\ref{eq:ch2_pvalue_validity}): The CDF $F_U(u)=u$ of a uniform random variable $\mathrm{Unif}(0,1)$ is shown as a dashed line, and a superuniform random variable has a CDF $F_p(u)\le u$ for all $u\in [0,1]$, lying on or below the dashed diagonal line.}
    \label{fig:superuniform}
\end{figure}

P-values are often constructed from concentration inequalities \citep{boucheron2003concentration}. Notably, in the bounded-loss setting $L_\lambda(\cdot, \cdot)\in[0,1]$, a simple and widely used construction of p-values is based on Hoeffding’s inequality \citep{hoeffding1963probability}.

\begin{lemma}[Hoeffding's inequality]
\label{lem:hoeffding}
Fix a hyperparameter $\lambda_k\in\Lambda$ and assume that the losses
$L_{\lambda_k}(X_i,Y_i)$ are independent and bounded in $[0,1]$.
Then, for any $t>0$, the empirical risk
$\widehat{R}_n(\lambda_k)$
satisfies
\begin{equation}
\label{eq:hoeffding_unit_interval}
\mathbb{P}\!\left(
\widehat{R}_n(\lambda_k)
\le
R(\lambda_k) - t
\right)
\;\le\;
\exp\!\left(-2n t^2\right).
\end{equation}
\end{lemma}

\begin{lemma}
    The statistic
\begin{equation}
\label{eq:ch2_hoeffding_pvalue}
p_k
\;=\;
\exp\!\left(
-2n\big(\alpha - \widehat{R}_n(\lambda_k)\big)_+^2
\right)
\end{equation}
is a p-value for the null hypothesis $\mathcal{H}_k$ in (\ref{eq:ch2_null}), where $(x)_+ = \max\{x,0\}$.
\end{lemma}
\begin{proof}
Under the null hypothesis $\mathcal{H}_k$ in (\ref{eq:ch2_null}), i.e., when $R(\lambda_k) = \mathbb{E}[\widehat{R}_n(\lambda_k)]> \alpha$, Hoeffding’s inequality (\ref{eq:hoeffding_unit_interval}) implies the condition
\begin{equation}
\label{eq:Hoeffding_eq}
\mathbb{P}\!\left(
\widehat{R}_n(\lambda_k \mid \mathcal{H}_k) \le \alpha - t
\right)
\;\le\;
\exp\!\left(-2n t^2\right),
\qquad t>0.
\end{equation}
The following chain of inequalities establishes the superuniformity property:
\begin{align}
F_{p_k}(u \mid \mathcal{H}_k)
&=
\mathbb{P}\!\left[
\exp\!\left(
-2n\big(\alpha-\widehat{R}_n(\lambda_k)\big)_+^2
\right)\le u
\ \middle|\ \mathcal{H}_k
\right] \nonumber\\
&=
\mathbb{P}\!\left[
\big(\alpha-\widehat{R}_n(\lambda_k)\big)_+
\ge \sqrt{\frac{-\log u}{2n}}
\ \middle|\ \mathcal{H}_k
\right] \nonumber\\
&=
\mathbb{P}\!\left[
\widehat{R}_n(\lambda_k)
\le
\alpha-\sqrt{\frac{-\log u}{2n}}
\ \middle|\ \mathcal{H}_k
\right] \nonumber\\
&\le
\exp\!\left(
-2n\left(\sqrt{\frac{-\log u}{2n}}\right)^2
\right)\nonumber\\
&= u.
\end{align}

\end{proof}

To interpret the p-value \eqref{eq:ch2_hoeffding_pvalue}, for a given hyperparameter $\lambda_k$, let
\begin{equation}
\Delta_{i,k} = \alpha - L_{\lambda_k}(X_i,Y_i)
\end{equation}
be the instantaneous reliability margin with respect to the target $\alpha$ for hyperparameter $\lambda_k$ given the test pair $(X_i,Y_i)$.
Under the null hypothesis $\mathcal{H}_k$, we have the inequality
\begin{equation}
\mathbb{E}[\Delta_{i,k} \mid \mathcal{H}_k] = \alpha - R(\lambda_k) < 0,
\end{equation}
i.e., the average margin is negative.
Hence, under the null hypothesis $\mathcal{H}_k$, the random variables $\Delta_{i,k}$ have negative mean, and large positive deviations of their sum provide evidence against the null hypothesis $\mathcal{H}_k$. Accordingly, the p-value \eqref{eq:ch2_hoeffding_pvalue} decreases exponentially quickly as the quantity $(\alpha - \widehat{R}_n(\lambda_k))_+$ becomes increasingly positive.

As further discussed in Appendix~\ref{app:ch2_pvalue_constructions}, more refined p-values based on more powerful concentration inequalities such as exact binomial tails for 0-1 losses or empirical Bernstein bounds based on knowledge or estimates of the variance of the loss can substantially improve test performance in terms of capacity to identify true reliable hyperparameters.

\subsection{P-values for Single-Hypothesis Testing}
\label{sec:p_Val_Single}

Before turning to MHT, it is useful to recall how p-values are used for a \emph{single} hypothesis test. To this end, consider one candidate hyperparameter $\lambda_k$ and the null hypothesis $\mathcal{H}_k:R(\lambda_k)>\alpha$.

The primary validity requirement for a single test is \emph{type-I error control}. This requires controlling the probability of falsely deeming hyperparameter $\lambda_k$ as reliable when in fact it is not. Given the p-value $p_k$ for hypothesis $\mathcal{H}_k$, consider the following test
\begin{equation}
    \label{eq:single_test}
    \text{if}\;p_k\leq \delta,\;\text{reject }\mathcal{H}_k,\;\text{while otherwise failing to reject the null}.
\end{equation}
The type-I probability of error is the probability of a false rejection, or false discovery. For the test (\ref{eq:single_test}), this probability is bounded as
\begin{equation}
\label{eq:ch2_type1}
\mathbb{P}\!\left(\text{reject } \mathcal{H}_k\mid \mathcal{H}_k\right)
=
\mathbb{P}(p_k\le \delta \mid \mathcal{H}_k)
\;\le\;
\delta,
\end{equation}
which is an immediate consequence of the superuniformity property \eqref{eq:ch2_pvalue_validity} that defines a p-value. Accordingly, the test (\ref{eq:single_test}) gives a type-I error probability no larger than $\delta$, and is said to be a level-$\delta$ test.

Attaining an arbitrarily low type-I error probability is straightforward, as it suffices to accept the hypothesis $\mathcal{H}_k$ irrespective of the available data. It is thus necessary to evaluate also the \textit{type-II error} probability, i.e., the probability of incorrectly identifying hyperparameter $\lambda_k$ as unreliable when in fact it is reliable. Equivalently, one can evaluate the \textit{power}, i.e., the complement of the type-II error. Writing the alternative hypothesis, i.e., that hyperparameter $\lambda_k$ is reliable, as
\begin{equation}
\mathcal{H}_k^c:\ R(\lambda_k)\le \alpha,
\end{equation}
the power of the level-$\delta$ test (\ref{eq:single_test}) is
\begin{equation}
\label{eq:ch2_power}
\mathrm{Power}(\delta)
\;=\;
\mathbb{P}\!\left(\text{reject } \mathcal{H}_k\mid \mathcal{H}_k^c\right)
=
\mathbb{P}(p_k\le \delta \mid R(\lambda_k)\le \alpha).
\end{equation}

The power of test (\ref{eq:single_test}) depends on the separation between the true risk $R(\lambda_k)$ and the target level $\alpha$, the sample size $n$, and the tightness of the p-value. For example, suppose that hyperparameter $\lambda_k$ is truly reliable by a margin $\varepsilon$, i.e.,
$R(\lambda_k) =  \alpha-\varepsilon$ for some $\varepsilon>0$.
Then, the p-value (\ref{eq:ch2_hoeffding_pvalue}) decreases at the rate
$\exp(-cn\varepsilon^2)$ for a constant $c>0$.
Consequently, the test (\ref{eq:single_test}) rejects the null hypothesis $\mathcal{H}_k$ with higher probability, i.e., has higher power,
when we have more calibration samples $n$, or the gap $\varepsilon$
between the true risk and the target $\alpha$ is larger.

In summary, for a single hypothesis, with the test (\ref{eq:single_test}), the threshold $\delta$ directly controls the type-I error, i.e., the probability of selecting an unreliable hyperparameter, while the chosen p-value construction governs how quickly the test can improve power, i.e., select truly reliable hyperparameters.

\subsection{P-values in MHT}
\label{sec:ch2_pval_mht}

Once p-values $\{p_k\}_{k=1,\ldots, K}$ have been constructed for all the $K$ null hypotheses $\mathcal{H}$, an FWER-controlling certified set
$\hat{\Lambda}$ can be obtained by applying any MHT procedure that controls the FWER
\eqref{eq:ch2_fwer_def} at level $\delta$. We briefly review several widely used MHT procedures in the following.

\textbf{1) Bonferroni correction:} The most classical FWER procedure is the Bonferroni correction \citep{bonferroni1936teoria}, which rejects each null hypothesis $\mathcal{H}_k$ whenever the inequality
\begin{equation}
\label{eq:ch2_bonferroni}
p_k \le \frac{\delta}{K}
\end{equation}
holds.

\begin{proposition}[FWER control of the Bonferroni correction]
\label{prop:ch2_bonferroni_fwer}
The Bonferroni correction \eqref{eq:ch2_bonferroni} controls the FWER at level $\delta$ under arbitrary dependence among the p-values.
\end{proposition}

\begin{proof}
Under the Bonferroni rule \eqref{eq:ch2_bonferroni}, the event of violating the FWER condition (\ref{eq:ch2_fwer_def}) can be written as
\begin{equation}
\left\{\exists\, k\in\mathcal{K}_0:\ p_k \le \frac{\delta}{K}\right\},
\end{equation}
i.e., that there exists a true null hypothesis $k$ that satisfies the Bonferroni condition (\ref{eq:ch2_bonferroni}).
By the union bound,
\begin{equation}
\label{eq:app_union_bound}
\mathbb{P}\!\left(\exists\, k\in\mathcal{K}_0:\ p_k \le \frac{\delta}{K}\right)
\le
\sum_{k\in\mathcal{K}_0}
\mathbb{P}\!\left(p_k \le \frac{\delta}{K}\right).
\end{equation}
For each $k\in\mathcal{K}_0$, the hypothesis $\mathcal{H}_k$ is true, and hence, by the superuniformity condition (\ref{eq:ch2_pvalue_validity}) with $u=\delta/K$,
\begin{equation}
\mathbb{P}\!\left(p_k \le \frac{\delta}{K}\right)
=
\mathbb{P}\!\left(p_k \le \frac{\delta}{K}\middle| \mathcal{H}_k\right)
\le
\frac{\delta}{K}.
\end{equation}
Substituting into (\ref{eq:app_union_bound}) yields
\begin{equation}
\mathrm{FWER}
\le
\sum_{k\in\mathcal{K}_0} \frac{\delta}{K}
=
\frac{|\mathcal{K}_0|}{K}\,\delta
\le
\delta,
\end{equation}
since $|\mathcal{K}_0|\le K$.
\end{proof}

Note that, equivalently, one may define Bonferroni-adjusted p-values
\begin{equation}
\label{eq:ch2_bonferroni_adj}
\tilde p_k \;=\; \min\{1,\ Kp_k\},
\end{equation}
and reject $\mathcal{H}_k$ if the inequality $\tilde p_k \le \delta$ holds.
Bonferroni is conservative when $K$ is large, but it provides a simple baseline with
minimal assumptions.

\textbf{2) Fixed sequence testing (FST):}
An alternative FWER-controlling procedure is fixed sequence testing, which exploits a pre-specified ordering of the hypotheses.

Suppose that, prior to examining the calibration data used to compute the p-values, the hypotheses in $\mathcal{H}$ are arranged in a fixed sequence from most promising to least promising to be reliable. This ordering may reflect prior information or structural considerations, such as increasing model complexity, decreasing validation performance on a separate dataset, or expert knowledge suggesting that certain hyperparameters are more promising. Note the importance of not using calibration data in this process.

Denote the ordered list of hypotheses by $\mathcal{H}_{(1)}, \ldots, \mathcal{H}_{(K)}$, and let $p_{(1)}, \ldots, p_{(K)}$ be the corresponding p-values arranged according to this predetermined sequence. FST starts by testing $\mathcal{H}_{(1)}$ at level $\delta$ using the test (\ref{eq:single_test}).
If $p_{(1)} \le \delta$, the hypothesis $\mathcal{H}_{(1)}$ is rejected and the procedure proceeds to test $\mathcal{H}_{(2)}$ at the same level $\delta$.
In general, testing continues sequentially for $k = 1, \ldots, K$ as long as the inequality
$p_{(k)} \le \delta$ holds.
FST stops at the first index
\begin{equation}
\label{eq:ch2_fst_kstar}
k^\star
=
\min\{k:\ p_{(k)} > \delta\},
\end{equation}
with the convention $k^\star = K+1$ if all hypotheses are rejected.
The rejected set of hypotheses is then
\begin{equation}
    \hat{\mathcal{K}} = \{\mathcal{H}_{(1)},\ldots,\mathcal{H}_{(k^\star-1)}\}.
\end{equation}

\begin{proposition}[FWER control of fixed sequence testing]
\label{prop:ch2_fst_fwer}
Fixed sequence testing \eqref{eq:ch2_fst_kstar} controls the FWER at level $\delta$.
\end{proposition}

\begin{proof}
Let $k_0$ be the smallest index such that $\mathcal{H}_{(k_0)}$ is a true null hypothesis, i.e., $\mathcal{H}_{(k_0)}\in \mathcal{K}_0$.
A family-wise error occurs only if $\mathcal{H}_{(k_0)}$ is rejected, since the procedure stops immediately after the first non-rejection.
Because $\mathcal{H}_{(k_0)}$ is tested at level $\delta$,
\begin{equation}
\mathrm{FWER}
=
\mathbb{P}\!\left(p_{(k_0)} \le \delta \middle| \mathcal{H}_{(k_0)}\right)
\le
\delta,
\end{equation}
where the inequality follows from the superuniformity property \eqref{eq:ch2_pvalue_validity}.
\end{proof}

Unlike Bonferroni, FST does not divide the level $\delta$ across the $K$ hypotheses. When the ordering places truly reliable hyperparameters early in the sequence, the procedure can therefore be substantially more powerful. However, its performance depends critically on the quality of the predetermined ordering: if unreliable hyperparameters appear early, the procedure may stop prematurely and fail to identify reliable configurations appearing later in the sequence.

\section{E-values}
\label{sec:ch2_eval}

This section introduces e-values as alternative statistics $T_k$ for the null hypotheses $\mathcal{H}_k$. A self-contained background treatment of e-values, including their interpretation as generalized likelihood ratios and the proofs of the properties summarized below, is provided in Appendix \ref{app:evidence}.

\subsection{Definition}

A conceptually distinct way to construct valid statistics is through \emph{e-values}. E-values quantify evidence against a null hypothesis via an expectation constraint, rather than a tail-probability constraint like p-values
\citep{vovk2021evalues,grunwald2020safetest,shafer2019gamefoundation, ramdas2025hypothesis, shafer2019language}. They can be thought of as a generalization of likelihood ratios, as a betting score, or as a strong type of p-value with special properties \citep{ramdas2025hypothesis}.

For each candidate hyperparameter $\lambda_k\in\Lambda$, an e-value $e_k$ for the null hypothesis $\mathcal{H}_k$ in \eqref{eq:ch2_null} is a nonnegative random variable satisfying the inequality
\begin{equation}
\label{eq:ch2_evalue_validity}
\mathbb{E}\!\left[e_k\mid \mathcal{H}_k\right] \;\le\; 1,
\end{equation}
where the expectation is taken over the randomness in the calibration data used to evaluate, conditioned on the null hypothesis $\mathcal{H}_k$.
Large values of the statistic $e_k$ provide evidence against the null hypothesis $\mathcal{H}_k$.

The defining property \eqref{eq:ch2_evalue_validity} implies that
\begin{equation}
\label{eq:pval_eval}
p_k = \frac{1}{e_k}
\end{equation}
is a p-value for hypothesis $\mathcal{H}_k$. In fact, by Markov's inequality we have
\begin{equation}
\label{eq:ch2_markov_test}
\mathbb{P}\!\left(e_k \ge \frac{1}{u}\middle|\mathcal{H}_k\right) = \mathbb{P}\left(\frac{1}{e_k}\geq u \middle| \mathcal{H}_k\right) \leq \mathbb{E}[e_k\mid\mathcal{H}_k]\cdot u\leq u.
\end{equation}
Therefore, by the definition (\ref{eq:ch2_evalue_validity}), the reciprocal (\ref{eq:pval_eval}) of an e-value is a p-value. However, the converse does not hold: For a p-value $p_k$, it is generally not true that the reciprocal $1/p_k$ is an e-value. E-values hence correspond to special types of p-values.

As further discussed in Sec. \ref{sec:eval_properties} and Appendix \ref{app:evidence}, e-values have useful properties that general p-values do not possess. Notably, e-values enable the post-hoc selection of the level of the test \citep{koning2024continuous}; and multiple e-values can be combined via averaging or through products (under independence), supporting optional continuation and stopping in sequential testing \citep{ramdas2025hypothesis}.

\subsection{Construction}
\label{sec:e_value_construction}

There are several principled constructions of e-values. We summarize here on the most relevant ones for hyperparameter selection. Other principled strategies include the likelihood-based universal inference approach \citep{wasserman2020universal}, which has been leveraged in \citep{dey2026generalized} to obtain confidence intervals on minimizers of empirical risks. In this regard, we note that the e-value in \citep{dey2026generalized} applies to the null that a given hyperparameter $\lambda$ is the risk minimizer, and rejection establishes only that the hyperparameter is suboptimal. Using this statistic therefore supports the identification of candidate hyperparameters to discard, but it provides no risk bound for the candidates that are not discarded.

\paragraph{1. E-values from concentration inequalities:} Like p-values, e-values can be constructed directly from concentration bounds. Consider again the bounded-loss setting $L_\lambda(\cdot,\cdot)\in[0,1]$. We first recall Hoeffding's lemma \citep{hoeffding1963probability}.

\begin{lemma}[Hoeffding's lemma]
\label{lem:hoeffding_lemma}
Let $X$ be a random variable such that $a \le X \le b$ almost surely.
Then for any $\eta\in\mathbb{R}$, the inequality
\begin{equation}
\label{eq:hoeffding_lemma}
\mathbb{E}\!\left[
\exp\!\bigl(\eta (X-\mathbb{E}[X])\bigr)
\right]
\;\le\;
\exp\!\left(
\frac{\eta^2 (b-a)^2}{8}
\right)
\end{equation}
holds.
\end{lemma}

Fix any $\eta>0$ and define the statistic
\begin{equation}
\label{eq:hoeffding_evalue}
e_k(\eta)
\;=\;
\exp\!\Bigl(
\eta n(\alpha-\widehat{R}_n(\lambda_k))
-
\frac{\eta^2 n}{8}
\Bigr).
\end{equation}
To verify that $e_k(\eta)$ is an e-value, note that under the null $\mathcal{H}_k$ we have the inequality $\mathbb{E}[\widehat{R}_n(\lambda_k)]\ge\alpha$. Since the losses are bounded in $[0,1]$, Lemma \ref{lem:hoeffding_lemma} implies that, for any $\eta>0$,
\begin{equation}
\label{eq:eval_from_bounds}
\mathbb{E}\!\left[
\exp\!\bigl(
\eta n(\alpha-\widehat{R}_n(\lambda_k))
\bigr)
\ \middle|\ \mathcal{H}_k
\right]
\le
\exp\!\left(\frac{\eta^2 n}{8}\right).
\end{equation}
Dividing both sides by $\exp(\eta^2 n/8)$ yields
\begin{equation}
\mathbb{E}[e_k(\eta)\mid\mathcal{H}_k]\le 1,
\end{equation}
so that $e_k(\eta)$ satisfies the defining property \eqref{eq:ch2_evalue_validity}.

Large positive deviations of $\alpha-\widehat{R}_n(\lambda_k)$ produce exponentially large e-values $e_k(\eta)$, thereby providing strong evidence against the null. The parameter $\eta$ controls the trade-off between sensitivity to small deviations and robustness to variability. In practice, the parameter $\eta$ may be optimized, discretized over a grid and averaged, or selected via mixture constructions to improve power while preserving validity \citep{howard2021time, ramdas2023gametheoretic, waudby2024estimating}.

\paragraph{2. Calibrating P-values into E-values:}

\noindent Any p-value can be transformed into an e-value using an \emph{e-calibrator}. Since e-values have additional properties as compared to p-values, this conversion typically yields statistics that are less powerful than the original p-values in detecting true null hypotheses. In other words, the price to pay for benefiting from the special properties of e-values (see Sec. \ref{sec:eval_properties}) is a decrease in power.

\begin{definition}
\label{def:calibrator}
Given a superuniform random variable $U$, function $f:[0,1]\to[0,\infty)$ is an e-calibrator if the following inequality holds
\begin{equation}
\label{eq:ch2_ecalibrator}
\mathbb{E}[f(U)] \le 1.
\end{equation}
\end{definition}
By this definition, if $p_k$ is a p-value, then
\begin{equation}
\label{eq:ch2_p_to_e}
e_k = f(p_k)
\end{equation}
is an e-value for the same null hypothesis.

A common family of e-calibrators can be expressed as
\begin{equation}
\label{eq:ch2_simple_calibrator}
f_\kappa(p) = (1-\kappa) \, p^{-\kappa}, \qquad \kappa \in (0,1),
\end{equation}
which can be shown to satisfy the condition \eqref{eq:ch2_ecalibrator} \citep{vovk2021evalues}. An optimal e-calibrator would yield the largest e-value $f_\kappa(p)$ across all possible p-values $p\in[0,1]$. However, as illustrated in Fig. \ref{fig:ecalibrators}, within the family (\ref{eq:ch2_simple_calibrator}), no single value of $\kappa$ yields a larger e-value $f_\kappa(p)$, for all possible values of $p$ \citep{vovk2021evalues}. Therefore, the optimal choice of the parameter $\kappa$ is non-trivial, and a common selection is $\kappa=1/2$ \citep{yanchenko2025hypothesis}.

\begin{figure}[t]
\centering
\includegraphics[width=0.7\linewidth]{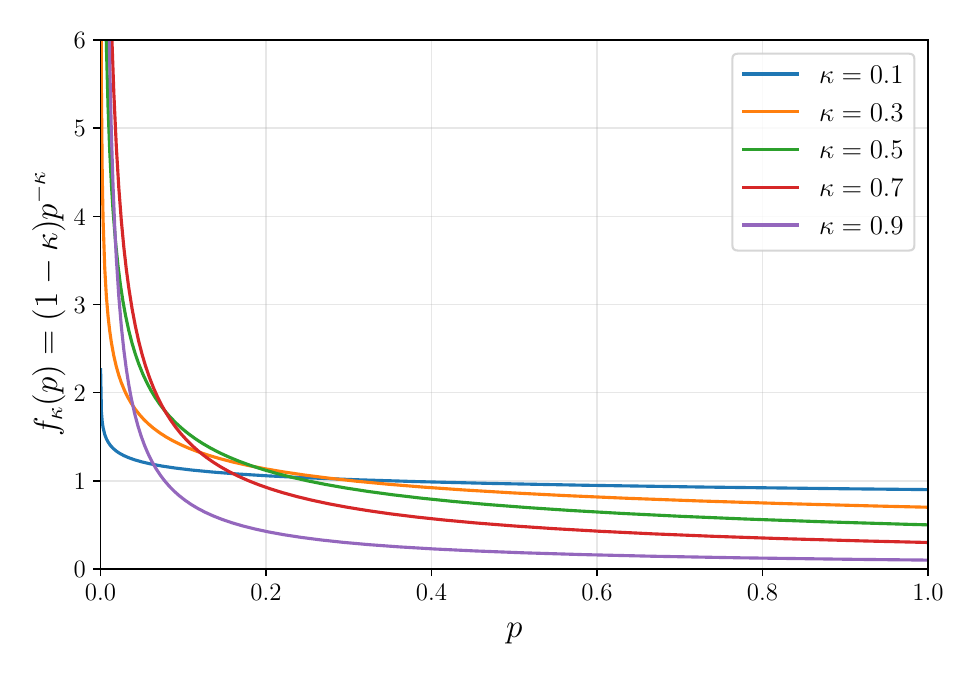}
\caption{
E-calibrators from the family (\ref{eq:ch2_simple_calibrator}) for several
values of $\kappa\in(0,1)$.
Each curve maps a p-value $p$ to an e-value according to
(\ref{eq:ch2_simple_calibrator}).
The figure illustrates that no single value of $\kappa$ dominates
uniformly over the entire range $p\in[0,1]$: calibrators with larger
$\kappa$ produce larger e-values for very small p-values, while
smaller $\kappa$ yields larger e-values for moderate p-values.
}
\label{fig:ecalibrators}
\end{figure}

More generally, admissible calibrators can be derived via convex duality arguments \citep{vovk2021evalues}.
Calibrators allow one to leverage existing p-value constructions (Sec.~\ref{sec:ch2_pval}),
while inheriting the sequential and combination advantages of e-values, at the cost of a loss in power.

\subsection{E-values for Single-Hypothesis Testing}

For a single null hypothesis $\mathcal{H}_k:R(\lambda_k)>\alpha$, an e-value $e_k$ can be used to define a level-$\delta$ test as
\begin{equation}
    \label{eq:single_test_eval}
    \text{if}\;e_k\geq \frac{1}{\delta},\;\text{reject }\mathcal{H}_k;\;\text{otherwise accept }\mathcal{H}_k.
\end{equation}
The validity of test (\ref{eq:single_test_eval}) follows immediately from the definition \eqref{eq:ch2_evalue_validity} and from Markov's inequality as
\begin{equation}
\label{eq:ch2_single_e_test}
\mathbb{P}\left(\text{reject }\mathcal{H}_k \middle| \mathcal{H}_k\right)
=
\mathbb{P}\!\left(e_k\ge \frac{1}{\delta}\middle| \mathcal{H}_k\right)
\;\le\;
\delta\,\mathbb{E}[e_k\mid \mathcal{H}_k]
\;\le\;
\delta.
\end{equation}
Thus, e-values provide type-I error control via the test (\ref{eq:single_test_eval}).

\subsection{Properties}
\label{sec:eval_properties}

E-values possess structural properties that distinguish them from p-values and make them particularly attractive in adaptive and post-selection settings. We highlight two key features: post-hoc level selection \citep{grunwald2020safetest} and stability under convex combination \citep{vovk2021evalues}. The corresponding limitations of p-values that motivate these properties, as well as the underlying proofs, are reviewed in Appendix~\ref{app:evidence}.

\medskip
\noindent\textbf{1. Post-hoc selection of the testing level:} A distinctive feature of e-values is that a single realized test statistic simultaneously supports valid tests at all levels $\delta\in(0,1)$.
For the test (\ref{eq:single_test}) based on p-values, the type-I error guarantee is only valid if the level $\delta$ is specified in advance, prior to observing the data, and thus the p-values $p_k$. When the significance level $\delta$ is chosen as a function of the p-value $p$, one has the phenomenon of p-hacking \citep{simonsohn2014p}, causing a violation of the type-I error guarantee. In contrast, with e-values, any rejection rule of the form (\ref{eq:single_test_eval}) retains valid type-I error control even with a data-dependent choice of $\delta$. In this sense, $1/e$ plays the role of a data-determined significance level.

\medskip
\noindent\textbf{2. Additivity and stability under convex combination:} Another major advantage of e-values is their stability under convex combination. If $e_1,\dots,e_M$ are valid e-values for the same null hypothesis $\mathcal{H}$, then the average
\begin{equation}
\bar e=\frac{1}{M}\sum_{m=1}^M e_m
\end{equation}
is again a valid e-value. Indeed, by linearity of expectation we can write
\begin{equation}
\mathbb{E}[\bar e \mid \mathcal{H}]
=
\frac{1}{M}\sum_{m=1}^M \mathbb{E}[e_m \mid \mathcal{H}]
\le
1,
\end{equation}
proving that $\bar e$ is also a valid e-value. More generally, for any nonnegative weights $w_1,\dots,w_M$ satisfying $\sum_{m=1}^M w_m = 1$, the weighted average
\begin{equation}
e_w = \sum_{m=1}^M w_m e_m
\end{equation}
remains a valid e-value \citep{vovk2021evalues}. Importantly, this stability holds without any independence assumptions among the e-values $e_m$.

\subsection{E-values for MHT}

E-values can be used in MHT in two closely related ways. First, since $p_k = 1/e_k$ is a valid p-value by \eqref{eq:ch2_markov_test}, any p-value based FWER-controlling procedure discussed in Sec.~\ref{sec:ch2_pval_mht} can be applied to the induced p-values as $\mathcal{A}_\text{FWER}(1/e_1, \ldots, 1/e_K)$. This immediately yields a certified set $\hat{\Lambda}$ with the same FWER guarantees as in Sec. \ref{sec:ch2_pval_mht}.

For example, one may control FWER using the test (\ref{eq:single_test_eval}) with modified thresholds
\begin{equation}
e_k \ge \frac{|\Lambda|}{\delta}.
\end{equation}
This guarantees the condition $\mathrm{FWER}\le \delta$ as, by the union bound, we can write the inequality
\begin{equation}
\label{eq:eval_fwer_proof}
\mathrm{FWER}
=
\mathbb{P}\!\left(\bigcup_{k\in\mathcal{K}_0}
\left\{e_k \ge \frac{|\Lambda|}{\delta}\right\}\right)
\le
\sum_{k\in\mathcal{K}_0}
\mathbb{P}\!\left(
e_k \ge \frac{|\Lambda|}{\delta}
\middle| \mathcal{H}_k
\right).
\end{equation}
By Markov’s inequality and the property (\ref{eq:ch2_evalue_validity}) we have
$\mathbb{E}[e_k \mid \mathcal{H}_k]\le 1$,
\begin{equation}
\label{eq:eval_modified_property}
\mathbb{P}\!\left(
e_k \ge \frac{|\Lambda|}{\delta}
\middle| \mathcal{H}_k
\right)
\le
\frac{\delta}{|\Lambda|}.
\end{equation}
Inserting (\ref{eq:eval_modified_property}) into (\ref{eq:eval_fwer_proof}) and summing over $k\in\mathcal{K}_0$ yields
\begin{equation}
\mathrm{FWER}
\le
\frac{|\mathcal{K}_0|}{|\Lambda|}\,\delta
\le
\delta,
\end{equation}
which establishes the desired condition $\mathrm{FWER}\le \delta$.

More powerful MHT procedures may be obtained by targeting \emph{false discovery rate} (FDR) instead of the FWER. This is discussed in Sec. \ref{sec:FDR}.

\section{False Discovery Rate Control}
\label{sec:FDR}

While FWER control prevents \emph{any} false certification with high probability via the requirement (\ref{eq:ch2_fwer_def}), it can be overly conservative when the number of candidate hyperparameters $K$ is large. In such cases, it may be acceptable to tolerate a number of false certifications, as long as they constitute only a small fraction of the certified set. This motivates control of the \emph{false discovery rate} (FDR).

\subsection{Definition}

To quantify the amount of false certification within the subset $\hat{\Lambda}$ of selected hyperparameters, let  $ \big|\,\hat{\Lambda}\cap\Lambda_0|$ denote the number of false rejections, i.e., the number of certified hyperparameters that are actually unreliable. We recall that $\Lambda_0\subseteq \Lambda$ denotes the subset of unreliable hyperparameters.

The \textit{false discovery proportion} (FDP) is the ratio
\begin{equation}
\label{eq:ch2_fdp_def}
\mathrm{FDP}
\;=\;
\frac{\big|\,\hat{\Lambda}\cap\Lambda_0|}{\max\{1,|\hat{\Lambda}|\}}
\end{equation}
between the number of false discoveries and the total number of discoveries. The FDR is defined as the average
\begin{equation}
\label{eq:ch2_fdr_def}
\mathrm{FDR}
\;=\;
\mathbb{E}\!\left[\mathrm{FDP}\right]
\;=\;
\mathbb{E}\!\left[
\frac{\big|\,\hat{\Lambda}\cap\Lambda_0|}{\max\{1,|\hat{\Lambda}|\}}
\right],
\end{equation}
where the expectation is over the calibration dataset.

Controlling the FDR at level $\delta$, i.e., $\text{FDR}<\delta$, ensures that, on average, the fraction of unreliable hyperparameters inside the certified set $\hat{\Lambda}$ is at most $\delta$. Unlike FWER control, FDR control does not rule out the presence of any false certifications on a given run. Rather, it bounds their expected proportion.

\begin{definition}[FDR-controlling procedure]
\label{def:ch2_fdr}
A procedure
\begin{equation*}
\mathcal{A}_{\text{FDR}}\!:[0,1]^{K}\to 2^{K},
\end{equation*}
mapping a collection of statistics $T_1, \ldots, T_K$ to a set of selected hyperparameters $\hat{\Lambda}=\mathcal{A}_{\text{FDR}}(T_1,\ldots,T_K)$, is said to \emph{control the FDR} at level $\delta$ if
\begin{equation}
\label{eq:ch2_fdr_set}
\mathbb{E}\!\left[
\frac{\big|\,\hat{\Lambda}\cap\Lambda_0|}{\max\{1,|\hat{\Lambda}|\}}
\right]
\;\le\;
\delta.
\end{equation}
\end{definition}

Fig.~\ref{fig:fdr} illustrates two outcomes corresponding to different realizations of the calibration data. In the first realization (left), the FDP is smaller, while in the second (right) the FDP is larger. The FDR condition \eqref{eq:ch2_fdr_set} ensures that on average, the realized FDP is no larger than the specified outage rate $\delta$.

\begin{figure}[t]
    \centering
    \includegraphics[width=0.95\linewidth]{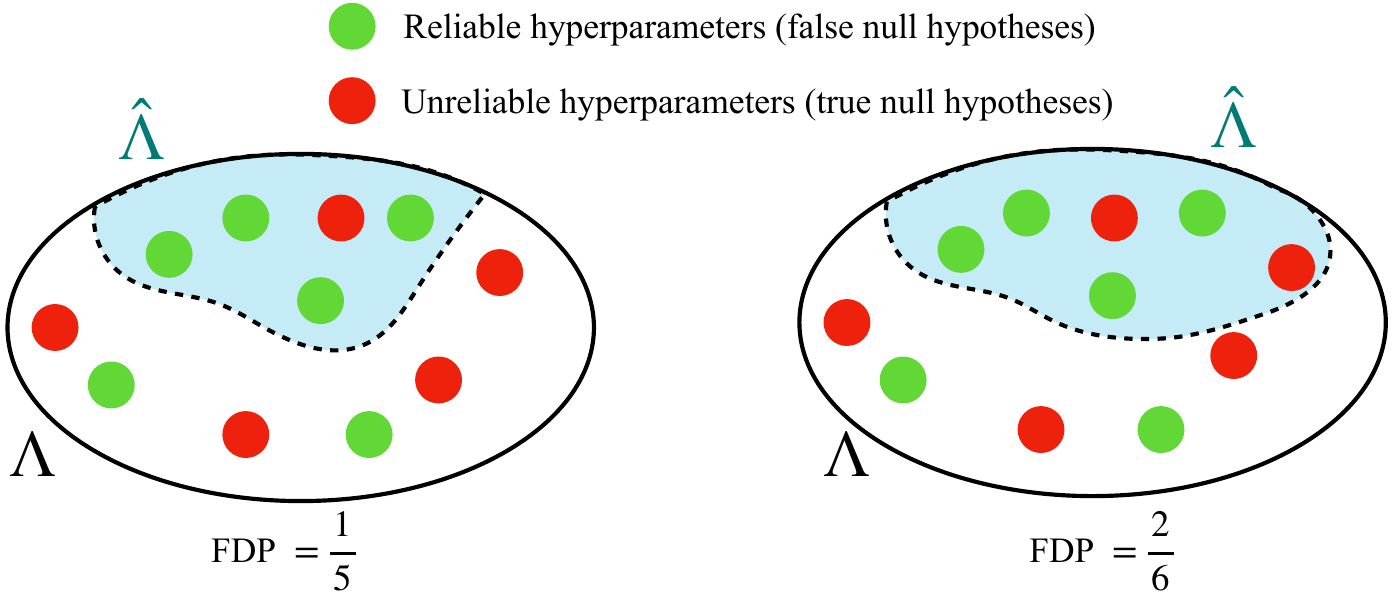}
    \caption{Schematic illustration of FDR control in hyperparameter
    selection. Each point represents a candidate hyperparameter, with green and red markers
    denoting truly reliable and unreliable configurations, respectively. The shaded region
    corresponds to the subset $\hat{\Lambda}$ of selected hyperparameters. The FDR requirement (\ref{eq:ch2_fdr_set}) ensures that the FDP is on average no larger than the target $\delta$.}
    \label{fig:fdr}
\end{figure}

\subsection{FDR Control via P-values}
\label{sec:ch2_fdr_pval}

The most widely used FDR-controlling method for p-values is the Benjamini--Hochberg (BH) procedure \citep{benjamini1995controlling}. The BH procedure is an example of a \textit{step-up} procedure that tests the null hypotheses in increasing order of the corresponding p-values \citep{rice2007mathematical}. To elaborate, let $p_{(1)} \le p_{(2)} \le \cdots \le p_{(K)}$ denote the p-values $\{p_1, \ldots, p_K\}$ sorted in increasing order. BH rejects the $k^\star$ hypotheses with smallest p-values, i.e., those indexed by $\{(1),\dots,(k^\star)\}$, where $k^\star$ is selected as
\begin{equation}
\label{eq:ch2_bh_kstar}
k^\star
=
\max\left\{
k:\ p_{(k)} \le \frac{k}{K}\,\delta
\right\}.
\end{equation}

Accordingly, the smallest p-value $p_{(1)}$ is compared with the same threshold as in the Bonferroni procedure, i.e., $\delta/K$, while larger p-values are compared to an increasing threshold, namely $k\delta/K$ for p-value $p_{(k)}$. Under independence, BH controls the FDR at level $\delta$.

When the p-values are arbitrarily dependent, a classical conservative modification is the Benjamini--Yekutieli (BY) procedure \citep{benjamini2001control}. The BY procedure replaces $\delta$ in (\ref{eq:ch2_bh_kstar}) by $\delta/c_{K}$, where
\begin{equation}
\label{eq:ch2_by_harmonic}
c_{K}
=
\sum_{i=1}^{K}\frac{1}{i}.
\end{equation}
Equivalently, BY uses the tests $p_{(k)} \le k\delta / Kc_{K}$ in (\ref{eq:ch2_bh_kstar}). This ensures FDR control without dependence assumptions, at the cost of generally reducing the number of discoveries.

\subsection{FDR Control via E-values}

The standard FDR-controlling rule for e-values is the e-BH procedure introduced in
\citep{wang2022fdr,ramdas2025hypothesis}, which can be seen as a direct counterpart of the BH procedure for p-values.
Let $e_{(1)} \ge e_{(2)} \ge \cdots \ge e_{(|\Lambda|)}$ be the sorted e-values in decreasing order. The e-BH procedure rejects the $k^\star$ hypotheses with largest e-values, with
\begin{equation}
\label{eq:ch2_eBH_kstar}
k^\star
=
\max\left\{
k:\ e_{(k)} \ge \frac{|\Lambda|}{k\,\delta}
\right\},
\end{equation}
with the convention $k^\star = 0$ if no such $k$ exists.

In contrast to the BH procedure for controlling the FDR via p-values as described in Sec. \ref{sec:ch2_fdr_pval}, e-BH controls FDR at the target level $\delta$ under arbitrary dependence among the e-values, without requiring a correction term as in the BY procedure (see Appendix \ref{app:e-BH} for a proof).

Variants and refinements of e-value FDR control include adaptive and weighted versions,
as well as procedures that combine e-values across groups or along structured graphs
\citep{wang2022fdr}.
These methods inherit robustness to dependence and compatibility with optional stopping (see Sec. \ref{sec:eval_properties}).

\section{Summary}
\label{sec:ch2_summary}

This chapter introduced LTT, a statistically principled approach to hyperparameter
selection based on MHT. The central idea is to
associate each candidate hyperparameter with a null hypothesis expressing
violation of the desired risk constraint, and to certify only those candidates
for which the null can be rejected with controlled error.

Upstream procedures generate a finite candidate set
using arbitrary heuristics, optimization routines, or domain knowledge.
Downstream, the calibration data are used to construct valid marginal
evidence measures, either p-values or e-values, one per candidate hyperparameter.
An MHT procedure is then applied to control global error,
yielding a certified set of hyperparameters.

When the testing rule controls the FWER,
the resulting certified set enjoys a finite-sample post-selection guarantee:
with high probability, every hyperparameter in the certified set satisfies
the prescribed risk requirement.
Any (possibly data-dependent) choice within that set therefore remains
risk-controlling.
If FDR control is used instead, the guarantee
becomes weaker but potentially less conservative, allowing larger certified sets.

We discussed two complementary forms of marginal evidence.
P-values rely on superuniform tail bounds, and connect naturally to classical
concentration inequalities.
E-values rely on expectation bounds, and are particularly well suited to
sequential or adaptive settings due to their stability under optional stopping
and averaging.

Algorithm~\ref{alg:ch2_mht_pipeline} summarizes the full
LTT hyperparameter selection pipeline.
Its modular structure highlights a core message of this chapter:
statistical validity can be layered on top of arbitrary hyperparameter
generation strategies, provided that the final deployment decision
is restricted to the certified set produced by a valid multiple testing rule.

\begin{algorithm}[t]
\caption{Learn-Then-Test (LTT): Reliable hyperparameter selection via multiple hypothesis testing}
\label{alg:ch2_mht_pipeline}
\begin{algorithmic}
\STATE \textbf{Input:} candidate set $\Lambda=\{\lambda_1,\dots,\lambda_{|\Lambda|}\}$; calibration data $\{(X_i,Y_i)\}_{i=1}^n$;
risk threshold $\alpha$; error level $\delta$; error type $\text{err}\in\{\mathrm{FWER},\mathrm{FDR}\}$; error-controlling MHT procedure $\mathcal{A}_{\mathrm{err}}(\cdot)$ at level $\delta$; selection algorithm $\mathcal{A}_{\mathrm{sel}}(\cdot)$, possibly data-dependent

\STATE \textbf{Output:} certified set $\hat{\Lambda}\subseteq\Lambda$ and a deployed hyperparameter $\hat{\lambda}\in\hat{\Lambda}$ (if nonempty)

\FOR{$j = 1,\dots,|\Lambda|$}
    \STATE Compute empirical risk
    \[
    \widehat{R}_n(\lambda_j)
    \;=\;
    \frac{1}{n}\sum_{i=1}^n L_{\lambda_j}\!\big(X_i,Y_i\big)
    \]
    \STATE Compute a p-value $p_j$ or e-value $e_j$ for the null hypothesis $\mathcal{H}_j:R(\lambda_j)>\alpha$
\ENDFOR

\STATE Apply $\mathcal{A}_{\mathrm{err}}$ to $\{p_j\}_{j=1}^{|\Lambda|}$ or $\{e_j\}_{j=1}^{|\Lambda|}$ and obtain the subset
\[
\hat{\Lambda}
=
\mathcal{A}_{\mathrm{err}}(p_1,\dots,p_{|\Lambda|})\;\;\text{or} \;\;
\hat{\Lambda}
=
\mathcal{A}_{\mathrm{err}}(e_1,\dots,e_{|\Lambda|})
\]

\IF{$\hat{\Lambda}\neq\emptyset$}
    \STATE Choose $\hat{\lambda} = \mathcal{A}_{\mathrm{sel}}(\hat{\Lambda})\in \hat{\Lambda}$
\ENDIF
\end{algorithmic}
\end{algorithm}


\chapter{Applications}
\label{chapter:applications}

The previous chapters introduced a statistical framework for reliable hyperparameter selection. The central idea is to move from best-effort tuning methods, which optimize empirical performance, to procedures that provide explicit reliability guarantees under finite calibration data.

This chapter presents representative applications of statistically valid hyperparameter selection. For each application, we describe the underlying system, define the relevant hyperparameters and performance metrics, and show how the statistical framework developed in the previous chapters can be applied to obtain reliability guarantees for deployment. Specifically, Sec.~\ref{sec:app_basic_ml} applies LTT to image classification on the Fashion-MNIST dataset, and Sec.~\ref{sec:app_wireless} considers a wireless scheduling application based on real packet-delay data.

\section{Image Classification}
\label{sec:app_basic_ml}

We begin with an image-classification example based on the Fashion-MNIST dataset \citep{xiao2017fashion}. Fashion-MNIST is a standard benchmark consisting of grayscale images of size $28\times 28$ from ten clothing categories, including T-shirts, pullovers, coats, sandals, bags, and ankle boots. For each input image $X$, the model outputs a class label $\hat Y_\lambda(X)$ determined by hyperparameter configuration $\lambda$. In this experiment, performance is measured through the $0$--$1$ loss
\begin{equation}
L_\lambda(X,Y)=\mathbb{I}\{\hat Y_\lambda(X)\neq Y\}.
\end{equation}
Specifically, the loss for a test point $(X,Y)$ is equal to 0 if the true label $Y$ is estimated correctly, and 1 otherwise.
Throughout this example, we set the target reliability level to $\alpha=0.2$, and the target outage rate to $\delta = 0.2$.

The classifier is a multinomial logistic regression model \citep{hosmer2013applied} applied to principal component analysis (PCA) \citep{abdi2010principal} features extracted from the image. Hence, each candidate hyperparameter is a pair
\begin{equation}
\lambda=(d,C),
\end{equation}
where $d$ denotes the number of retained principal components, and $C$ is the inverse regularization coefficient of the classifier. We consider a small finite candidate set obtained from the grid
\begin{equation}
d\in\{5,20,80\},
\qquad
C\in\{0.02,0.2,2,10\},
\end{equation}
yielding a total of $12$ candidate configurations.

The available data are divided into three disjoint parts. A training split is used to fit all candidate models, and a separate calibration split is used for hyperparameter selection. Finally, an independent evaluation split is reserved for testing the selected hyperparameters. {The training split is used once to fit all $K=12$ candidate models and then remains fixed. The remaining data are then randomly divided into calibration and evaluation sets independently for each of the 100 trials reported below. In each trial, the calibration set is used to select a hyperparameter, which is then assessed on the corresponding evaluation set.} LTT uses Algorithm \ref{alg:ch2_mht_pipeline} by computing p-values using the Hoeffding bound (\ref{eq:ch2_hoeffding_pvalue}) and using the Bonferroni correction (\ref{eq:ch2_bonferroni}) to control the FWER (\ref{def:ch2_fwer}). The conventional E-HPO baseline \eqref{alg:ch2_ehpo} simply selects the hyperparameter with the lowest empirical risk. Among the set $\hat{\Lambda}$ of hyperparameters returned by Algorithm \ref{alg:ch2_mht_pipeline}, we pick the one with the smallest empirical risk as the selected hyperparameter.

Table~\ref{tab:ch3_fmnist_ltt} illustrates the LTT procedure on a single representative calibration run with $n = 10{,}000$ calibration samples. For each of the $K = 12$ candidate configurations $\lambda = (d, C)$, it reports the empirical risk $\hat{R}_k$ on the calibration set and the corresponding Hoeffding p-value $p_k$ in \eqref{eq:ch2_hoeffding_pvalue}. The Bonferroni threshold is $\delta/K = 0.2/12 = 0.016$. The four configurations with $d = 5$ principal components yield empirical risks around $0.297 > \alpha$, and are therefore assigned p-value $p_k = 1$. The four configurations with $d = 20$ components have risks marginally above the target level $\alpha = 0.2$, and are likewise rejected. Only the four configurations with $d = 80$ components pass the Bonferroni test. Among these accepted configurations, the one with the lowest empirical risk, $(d, C) = (80, 2)$, is selected as the final hyperparameter configuration $\hat{\lambda}$.

\begin{table}[t]
\centering
\caption{LTT hyperparameter selection on Fashion-MNIST for a single representative calibration run ($n = 10{,}000$ samples, $\alpha = 0.2$, $\delta = 0.2$). A configuration is accepted if its Hoeffding p-value $p_k$ satisfies $p_k \le \delta/K = 0.016$. Among all accepted configurations, the one with the lowest empirical risk is marked as selected.}
\label{tab:ch3_fmnist_ltt}
\begin{tabular}{cc|cc|c}
\hline
\textbf{$d$} & \textbf{$C$} & \textbf{Empirical risk} $\hat{R}_k$ & \textbf{p-value} $p_k$ & \textbf{LTT Output} \\
\hline
$5$  & $0.02$ & $0.297$ & $1.000$ & not accepted \\
$5$  & $0.20$ & $0.296$ & $1.000$ & not accepted \\
$5$  & $2.00$ & $0.296$ & $1.000$ & not accepted \\
$5$  & $10.0$ & $0.296$ & $1.000$ & not accepted \\
$20$ & $0.02$ & $0.202$ & $1.000$ & not accepted \\
$20$ & $0.20$ & $0.201$ & $1.000$ & not accepted \\
$20$ & $2.00$ & $0.201$ & $1.000$ & not accepted \\
$20$ & $10.0$ & $0.201$ & $1.000$ & not accepted \\
$80$ & $0.02$ & $0.163$ & $7.1 \times 10^{-13}$ & accepted \\
$80$ & $0.20$ & $0.162$ & $1.8 \times 10^{-13}$ & accepted \\
$80$ & $2.00$ & $0.161$ & $6.1 \times 10^{-14}$ & selected \\
$80$ & $10.0$ & $0.161$ & $9.8 \times 10^{-14}$ & accepted \\
\hline
\end{tabular}
\end{table}

Fig.~\ref{fig:fmnist_examples} complements these results with representative images from the independent evaluation set. For each image, we report the ground-truth label together with the predictions produced by the hyperparameters chosen by conventional E-HPO and LTT. Consistent with the distributional comparison in Fig.~\ref{fig:fmnist_risk_distribution}, E-HPO produces models that can make more mistakes on visually ambiguous items.

We conclude this section by elaborating on the training-validation data split used for LTT. The use of a dedicated calibration set for hyperparameter selection
should be distinguished from cross-validation. In conventional supervised
learning, cross-validation is commonly used to select training-time
hyperparameters, particularly when the amount of available data is
limited, since it allows different portions of the data to be used for
training and validation across multiple folds \citep{bishop2006pattern,hastie2009elements,arlot2010survey,simeone2022machine}. When
sufficient data are available, a standard alternative is the adopted approach of partitioning the
data into separate training, validation, and test sets \citep{bishop2006pattern,simeone2022machine, goodfellow2016deep}. This methodology is especially natural for the post-training hyperparameters considered in this monograph, as the
predictive model does not need to be retrained for different candidate
hyperparameters.  We also note that, for very large models, alternative approaches based on empirical
scaling laws may be also used to guide choices such as model
size, training-data size, and compute allocation \citep{kaplan2020scaling,hoffmann2022training}.

Fig.~\ref{fig:fmnist_risk_distribution} reports the distribution of the empirical risk obtained
by the selected hyperparameters over 100 independent random splits
of the data into calibration and evaluation sets. For each split, the hyperparameter-selection procedure is applied
anew to the corresponding calibration set, and the resulting selected
hyperparameter is assessed on the independent evaluation set. Repeating this experiment over different calibration draws therefore quantifies the probability that the selection procedure returns a hyperparameter whose risk exceeds the target level. The dashed vertical line marks the target level $\alpha=0.2$, and the shaded portions show the mass of the distributions corresponding to violations of the reliability requirement.

Conventional E-HPO returns hyperparameters that fail to meet the requirement $\alpha=0.2$ for a fraction $0.53$ of the test runs, whereas LTT achieves an outage rate of $0.03$, well below the target $\delta=0.2$. This disparity reflects selection bias: E-HPO selects the configuration with the lowest empirical risk on the calibration set and hence tends to favor configurations whose calibration risk is optimistically low for the particular realization of the calibration data. Across repeated calibration draws, this selection effect can cause the population risk of the selected configuration to exceed the target level with substantial probability. LTT addresses this issue by requiring statistically significant evidence that the risk associated with a candidate hyperparameter is below the target $\alpha$, thus providing the desired reliability guarantee for the selected hyperparameter.

\begin{figure}[t]
    \centering
    \includegraphics[width=0.8\linewidth]{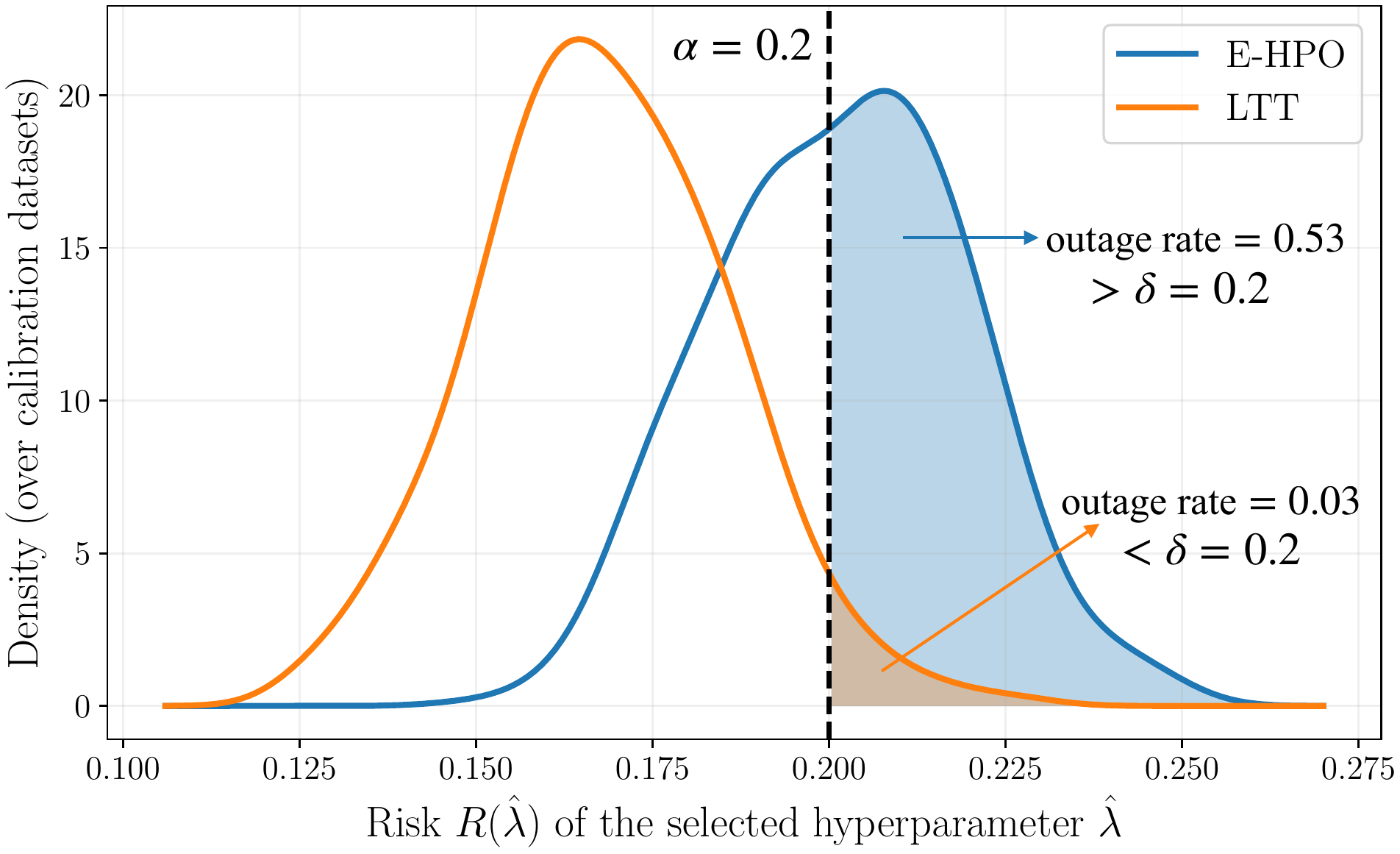}
    \caption{Distribution of the empirical risk of the hyperparameters selected by conventional E-HPO \eqref{alg:ch2_ehpo} and LTT (Algorithm \ref{alg:ch2_mht_pipeline}), evaluated over many random batches from an independent evaluation set on Fashion-MNIST. E-HPO violates the reliability requirement in a fraction 0.53 of the test runs, exceeding the target $\delta = 0.2$. In contrast, LTT successfully guarantees an outage rate of $0.03$, which is well below the target $\delta = 0.2$.}
    \label{fig:fmnist_risk_distribution}
\end{figure}

\begin{figure}[t]
    \centering
    \includegraphics[width=\linewidth]{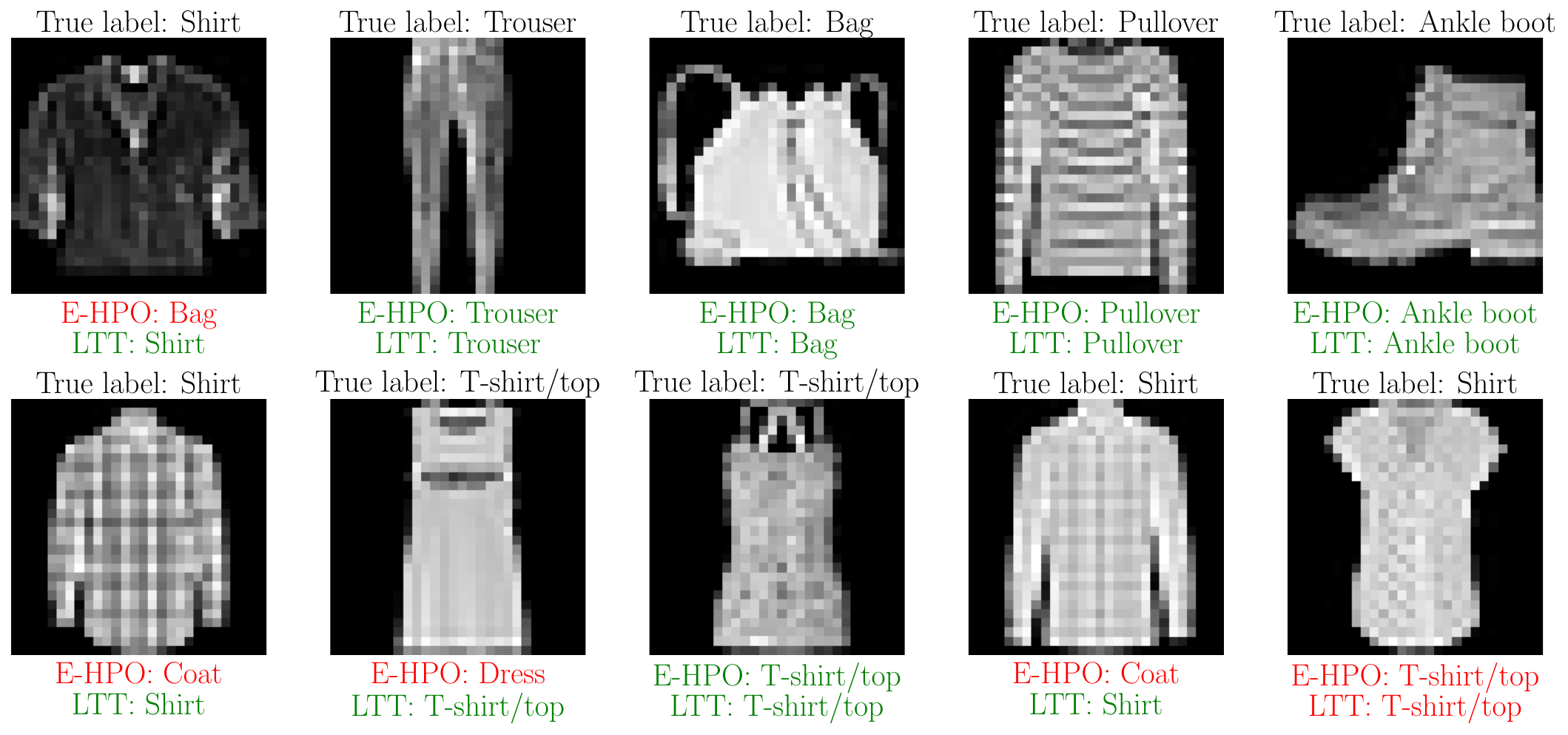}
    \caption{Representative Fashion-MNIST images from the independent evaluation set, shown together with the ground-truth class and the predictions produced by models obtained with the hyperparameters selected via conventional E-HPO and LTT.}
    \label{fig:fmnist_examples}
\end{figure}

\section{Wireless Scheduling}
\label{sec:app_wireless}

We now study an engineering application drawn from the wireless communications domain.
Specifically, we consider a radio access scheduling problem based on the Nokia wireless suite \citep{Nokia} studied in \citep{farzaneh2024quantile}. In this setting, hyperparameters control the behavior of a learning-based scheduler that allocates wireless resources among users.

We consider a downlink wireless system in which a base station must allocate radio resources to a set of user equipments (UEs).
At every transmission time interval, the base station assigns a number of resource blocks to active users according to a scheduling policy. Each episode consists of a sequence of time intervals during which packets arrive, are buffered, and are scheduled for transmission.

The scheduling policy is implemented by the learning-based agent described in \citep{de2020radio} whose behavior depends on a vector of hyperparameters
\begin{equation}
\lambda = (\lambda_1,\lambda_2,\lambda_3,\lambda_4).
\end{equation}

Each component of vector $\lambda$ controls the relative weight assigned to a specific term in the scheduler's reward model. Referring to \citep{de2020radio} for details, hyperparameter $\lambda_1$ weighs a term that accounts for channel-quality;
hyperparameter $\lambda_2$ weighs a term that grows with queue backlog;
hyperparameter $\lambda_3$ weighs a term that depends on waiting time; and hyperparameter $\lambda_4$ weighs a term that encourages fairness among UEs.

To evaluate the performance of a hyperparameter configuration $\lambda$, we define a risk function based on packet delays experienced by UEs belonging to the most stringent quality-of-service (QoS) class. Specifically, the risk is defined as $R(\lambda) = \mathbb{E}[L_\lambda(X,Y)]$, where $L_\lambda(X,Y)$ is the average packet delay (in ms) incurred during an episode under hyperparameter $\lambda$.

In the experiments presented below, we set the maximum tolerated average delay to $\alpha = 10$ ms, reflecting a typical latency requirement for delay-sensitive traffic \citep{series2017minimum}. Moreover, the candidate hyperparameter set is constructed by scaling a baseline configuration $\lambda^\star$ obtained from the optimization procedure proposed in \citep{de2020radio}.
Specifically, we consider all configurations of the form
$\lambda = (a_1\lambda^\star_1,a_2\lambda^\star_2,a_3\lambda^\star_3,a_4\lambda^\star_4)$
with scaling coefficients
$a_i \in \{1/2,1,3/2,2\}$. This results in a candidate set containing $K=256$ hyperparameter configurations.

Using the boundedness of the delay-based loss, we construct p-values via Hoeffding’s inequality as described in Chapter~\ref{chapter:mht}.
We then apply the Bonferroni correction \eqref{eq:ch2_bonferroni} to these p-values to control the FWER at level $\delta = 0.2$.

{The dataset of recorded episodes is divided into a training split and a remaining pool of data. The training split is used once to obtain the scheduling policy for each of the $K=256$ candidate configurations and then remains fixed. The remaining data are then randomly divided into calibration and evaluation sets independently for each trial. In each trial, the calibration set is used to select a hyperparameter configuration, which is then assessed on the corresponding evaluation set.}

Fig. \ref{fig:nokia} illustrates the distribution of the packet delays for models obtained using the hyperparameter chosen by E-HPO \eqref{alg:ch2_ehpo} and by LTT (Algorithm \ref{alg:ch2_mht_pipeline}) over the test dataset. It can be seen that, while the empirical mean of the delay for the hyperparameter chosen by LTT is successfully controlled below the threshold $\alpha = 10$ ms, the average packet delay for the hyperparameter chosen by E-HPO is 10.9 ms, which is larger than the threshold $\alpha = 10$ ms. This highlights the fact that E-HPO does not provide any statistical guarantees on the hyperparameters it outputs.

\begin{figure}
    \centering
    \includegraphics[width=\linewidth]{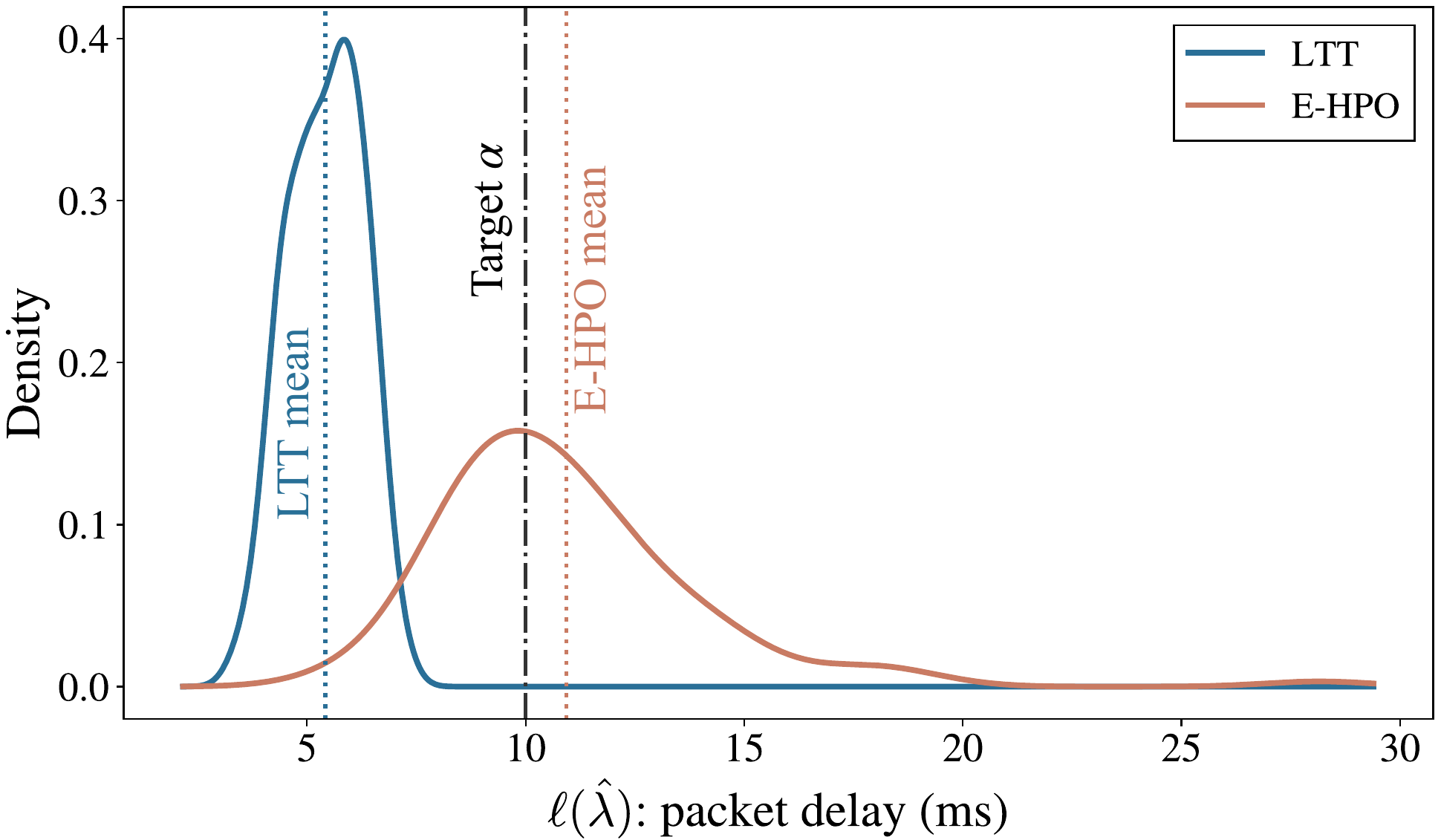}
    \caption{Distribution of individual packet delays observed on a single run obtained using hyperparameters selected via E-HPO and LTT.}
    \label{fig:nokia}
\end{figure}

To further see this, we repeat this experiment over {100 independent random splits into calibration and evaluation sets}, and plot the distribution of the empirical mean of the delay of the hyperparameters chosen by E-HPO and LTT over the test dataset. This distribution, illustrated in Fig. \ref{fig:nokia2}, shows that in a fraction $0.66$ of runs, E-HPO exhibited an empirical average above the target $\alpha = 10$ ms, which is well above the target outage rate of $\delta = 0.2$. On the other hand, LTT does not fail, keeping the outage rate below the target $\delta = 0.2$.

\begin{figure}
    \centering
    \includegraphics[width=\linewidth]{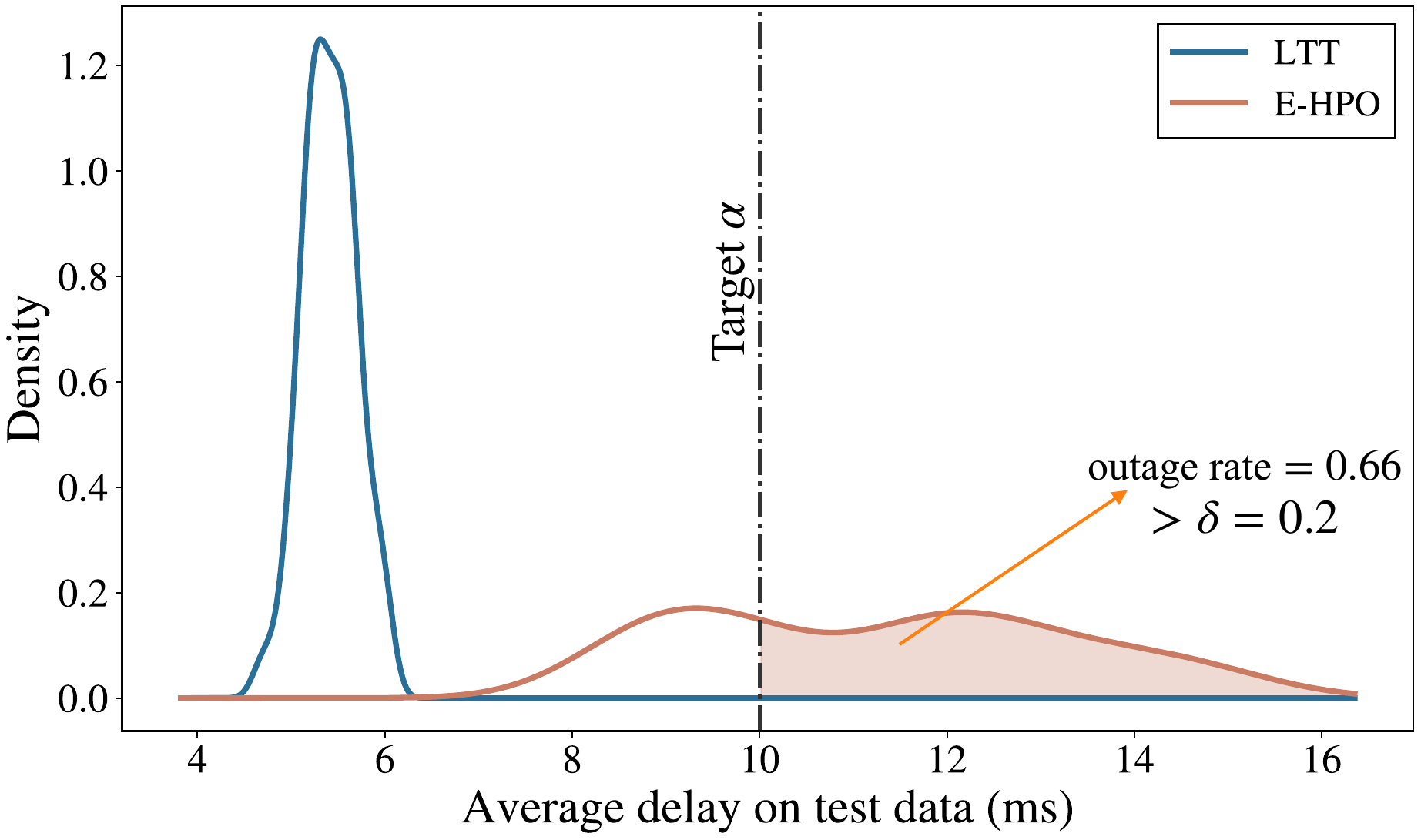}
    \caption{Distribution of the average delay obtained by the hyperparameters selected via E-HPO and LTT on 100 random splits of the dataset.}
    \label{fig:nokia2}
\end{figure}


\chapter{Beyond Average Risk}
\label{chapter:beyond_average}

In Chapter~\ref{chapter:mht}, we described LTT, a formal reliable hyperparameter selection approach aimed at controlling the average risk (\ref{eq:ch2_risk_def}) (see Algorithm \ref{alg:ch2_mht_pipeline}).
A hyperparameter was declared reliable if its expected loss did not exceed a user-specified threshold $\alpha$ as per (\ref{eq:risk_requirement}), and statistical validity was achieved by testing the null hypotheses (\ref{eq:ch2_null}) with control of FWER or FDR.

While the average risk is a natural and widely adopted performance measure, it is often insufficient for modern AI deployments. In fact, in many applications, reliability concerns focus not on the mean behavior of the system, but rather on tail events, constraint violations, or more general functionals of the loss distribution.

In this chapter, we develop this generalization of LTT by first introducing a unified methodological framework, and then exploring it along two main directions. First, we consider quantile risk control, which replaces the average risk with a specified quantile of the loss distribution, thereby enabling guarantees on tail behavior in Sec.~\ref{sec:qltt}. Second, we study reliability constraints involving information-theoretic quantities, with particular emphasis on the information bottleneck framework, where relevance is measured through mutual information in Sec.~\ref{sec:ch4_ib}.

\section{Motivation}

The central question of this chapter is the following:
Can the MHT-based hyperparameter selection procedure introduced in Chapter \ref{chapter:mht} be extended beyond average risk to certify hyperparameters under more general reliability criteria?

For example, settings of interest may include the following:
\begin{itemize}
    \item In safety-critical systems, one may wish to control the probability that the loss exceeds a critical threshold, rather than its average \citep{amodei2016concrete, varshney2016engineering}.
    \item In latency-sensitive applications, guarantees may be required on high-percentile delays, e.g., $95^\text{th}$ or $99^\text{th}$ percentiles, not only on the mean \citep{dean2013tail, cardwell2000modeling, bennis2018ultrareliable}.
    \item In information bottleneck learning, the reliability requirement may involve information-theoretic quantities such as mutual information \citep{tishby2000information, saxe2019information}.
    \item In fairness-aware systems \citep{hardt2016equality, barocas2016big}, constraints may involve group-conditional risks rather than a single global average.
\end{itemize}

These examples share a common structure: reliability is expressed not as a bound on the average loss $\mathbb{E}[L_\lambda(X,Y)]$, but as a constraint on a more general functional
$\mathcal{R}(\lambda)$ of the distribution of the loss $L_\lambda(X,Y)$,
which may represent a quantile, a tail probability, a divergence, or an information-theoretic measure.

The key insight is that the MHT framework of Chapter~\ref{chapter:mht} does not fundamentally depend on the specific form of the risk functional. Rather, it requires only the ability to construct valid test statistics for null hypotheses of the form
\begin{equation}
\label{eq:ch3_general_risk}
\mathcal{H}_k:\ \mathcal{R}(\lambda_k)>\alpha,
\end{equation}
where $\mathcal{R}(\lambda)$ can be any general functional derived from the distribution of the loss $L_\lambda(X,Y)$.

\section{General Reliability Functionals}
\label{sec:general_reliability_functionals}

Let $L_\lambda$ denote the loss induced by hyperparameter $\lambda$. As in the previous chapters, the random variable $L_\lambda$ may be a function of input-output variables $X$ and $Y$ as $L_\lambda = L_\lambda(X,Y)$, but the framework here is more general. It encompasses, e.g., unsupervised or self-supervised learning settings, in which the loss depends only on the input variable $X$. Let also $P_\lambda$ denote the distribution of the random variable $L_\lambda$.
We are interested in reliability measures of the form
\begin{equation}
\mathcal{R}(\lambda)
=
\Phi\!\left(P_\lambda\right),
\end{equation}
where $\Phi$ is a functional mapping a probability distribution to a real number.

The reliability requirement takes the unified form
\begin{equation}
\lambda \text{ is reliable}
\quad \Longleftrightarrow \quad
\mathcal{R}(\lambda) \le \alpha.
\end{equation}
This formulation encompasses many practically relevant notions of reliability, including the average risk $R(\lambda) = \mathbb{E}[L_\lambda]$ studied so far.

\section{Confidence Bounds and Inversion-Based Testing}
\label{sec:confidence_bounds}

This section reviews a useful approach to construct p-values and e-values for the null hypothesis
\begin{equation}
\label{eq:general_hyp}
    \mathcal{H}_\lambda:\quad \mathcal{R}(\lambda) > \alpha,
\end{equation}
where $\mathcal{R}(\lambda)$ is a general risk measure as defined in Sec. \ref{sec:general_reliability_functionals}.

Assume that, based on calibration data $\mathcal{D}^{\mathrm{cal}}$,
we can construct a one-sided upper confidence bound
$U_\varepsilon(\lambda)$ on the reliability measure $\mathcal{R}(\lambda)$
such that the inequality
\begin{equation}
\label{eq:confidence_bound}
\mathbb{P}\!\left(
\mathcal{R}(\lambda)
\le
U_\varepsilon(\lambda)
\right)
\ge
1-\varepsilon
\end{equation}
holds for all probabilities $\varepsilon \in (0,1)$,
where the probability is taken over the randomness of the calibration data.

The specific construction of the upper bound $U_\varepsilon(\lambda)$ depends on the risk measure $\mathcal{R}(\cdot)$, and we will provide specific examples in the next sections.

\begin{lemma}
    Given a family of upper bounds $U_\varepsilon(\lambda)$ for $\varepsilon\in(0,1)$ as in (\ref{eq:confidence_bound}), a p-value $p_\lambda$ for the null hypothesis $\mathcal{H}_\lambda$ in (\ref{eq:general_hyp}) is obtained by inverting the confidence bound $U_\varepsilon(\lambda)$ as

\begin{equation}
\label{eq:inverted_pval}
p_\lambda
=
\inf
\left\{
\varepsilon \in (0,1) :
U_\varepsilon(\lambda) \le \alpha
\right\}.
\end{equation}
In words, as illustrated in Fig. \ref{fig:upper_bound}, the p-value $p_\lambda$ is the smallest confidence level at which
the upper bound falls below $\alpha$.
\end{lemma}

\begin{proof}
Fix the outage rate $\delta \in (0,1)$ and assume that hypothesis $\mathcal{H}_\lambda$ is true, i.e.,
$\mathcal{R}(\lambda) > \alpha$.
By definition of the statistic (\ref{eq:inverted_pval}), the event $\{p_\lambda \leq \delta\}$ is equivalent to the upper bound $U_\delta(\lambda)$ being smaller than $\alpha$, i.e.,
\begin{equation}
\label{eq:ch4_pval_event}
\{ p_\lambda \le \delta \}
\equiv
\{ U_\delta(\lambda) \le \alpha \}.
\end{equation}
Since we assume the inequality $\mathcal{R}(\lambda) > \alpha$, the event $U_\delta(\lambda) < \mathcal{R}(\lambda) $ holds if we have $U_\delta(\lambda) \le \alpha$, i.e.,
\begin{equation}
\label{eq:ch4_pval_event_2}
\{ U_\delta(\lambda) \le \alpha \}
\subseteq
\{  U_\delta(\lambda) <\mathcal{R}(\lambda) \}.
\end{equation}
Therefore, by (\ref{eq:ch4_pval_event}) and (\ref{eq:ch4_pval_event_2}), the probability of the event $p_\lambda \le \delta$ is upper bounded by the probability of the event $\mathcal{R}(\lambda_k) > U_\delta(\lambda_k)$, which in turn is upper bounded by $\delta$ due to (\ref{eq:confidence_bound}), i.e.,
\begin{equation}
\label{eq:ch4_pval_event_3}
\mathbb{P}\!\left(
p_\lambda \le \delta
\right)
=
\mathbb{P}\!\left(
U_\delta(\lambda) \le \alpha
\right)
\le
\mathbb{P}\!\left(
\mathcal{R}(\lambda) > U_\delta(\lambda)
\right)
\le
\delta.
\end{equation}
This establishes the validity of the p-values $p_\lambda$ in (\ref{eq:inverted_pval}).
\end{proof}

\begin{figure}
    \centering
    \includegraphics[width=0.7\linewidth]{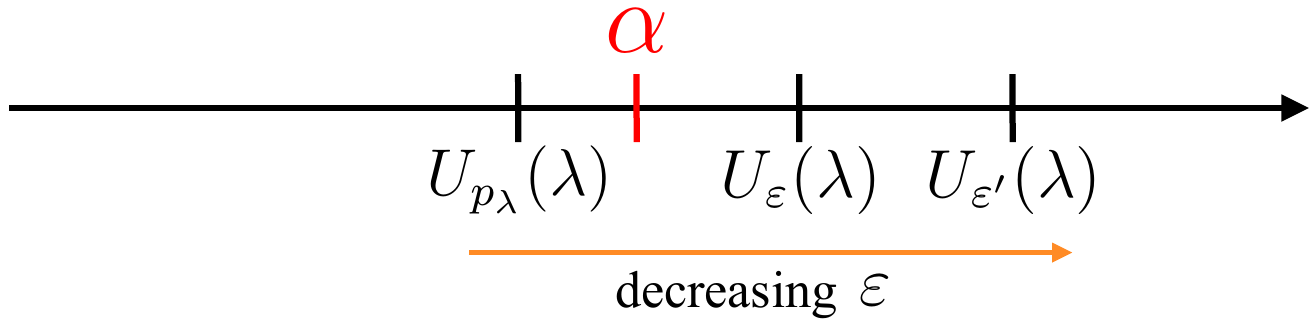}
\caption{
Illustration of the inversion-based p-value construction (\ref{eq:inverted_pval}) from a family of one-sided upper confidence bounds $U_\varepsilon(\lambda)$.
For a fixed hyperparameter $\lambda$, by the definition (\ref{eq:inverted_pval}), the bound $U_\varepsilon(\lambda)$ is nonincreasing in $\varepsilon$ (arrow).
The p-value $p_\lambda$ is defined as the smallest $\varepsilon$ for which the bound drops below the target threshold $\alpha$.
}
\label{fig:upper_bound}
\end{figure}

\section{Quantile Risk Control}
\label{sec:qltt}

In many systems, controlling the mean is insufficient if rare but large failures are unacceptable. In this section, we describe a variant of LTT that can control the quantile risk.

\subsection{Quantile Risk}
\label{sec:quantile_risk}

We define the quantile risk $\mathcal{R}(\lambda)$ as
\begin{equation}
\mathcal{R}(\lambda) = R_q(\lambda)
=
\inf\left\{ t \in \mathbb{R} :
\mathbb{P}(L_\lambda \le t) \ge 1-q
\right\},
\end{equation}
i.e., as the $(1-q)$-quantile of the loss $L_\lambda$.
The objective of quantile risk control is to certify hyperparameters that violate the null hypothesis
\begin{equation}
\label{eq:ch4_quantile_null}
\mathcal{H}_\lambda:\quad R_q(\lambda) > \alpha,
\end{equation}
where $\alpha$ is the target maximum latency.

As an example, consider the setting in Fig. \ref{fig:quantile_latency} in which a base station of a wireless system provides radio access to a group of users (see also Sec. \ref{sec:app_wireless}).
Consider as the loss $L_\lambda$ the delay experienced by one of the users.
A system that only ensures a mean delay below $30$ ms does not prevent occasional larger delay spikes.
This situation is illustrated in Fig.~\ref{fig:quantile_latency}(a), where the average delay, across multiple users, is below $30$ ms despite two users experiencing a delay larger than $30$ ms.

Instead, a system that controls the quantile risk as $R_{0.1}(\lambda) \le 30 \text{ ms}$ ensures that at least $90\%$ of users experience delays below $30$ ms.
As illustrated in Fig.~\ref{fig:quantile_latency}(b), this constraint explicitly limits the proportion of users experiencing high latency.

\begin{figure}[t]
\centering
\begin{subfigure}{0.3\textwidth}
    \centering
    \includegraphics[width=\linewidth]{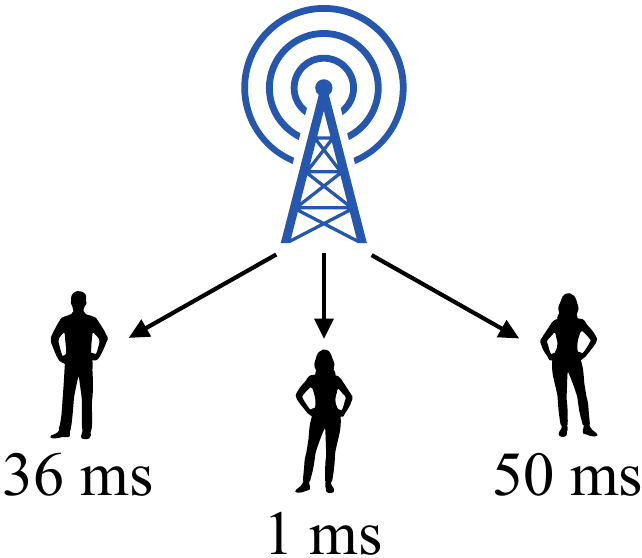}
    \caption{}
\end{subfigure}
\hspace{0.03\textwidth}
\begin{subfigure}{0.3\textwidth}
    \centering
    \includegraphics[width=\linewidth]{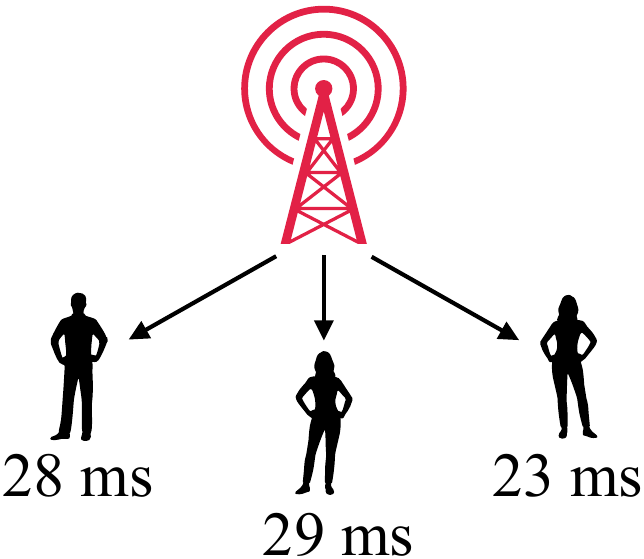}
    \caption{}
\end{subfigure}

\caption{
Illustration of mean versus quantile risk control for a wireless system.
(a) The average delay across the users is below $30$ ms, but two users experience a larger latency, showing that mean control alone does not prevent extreme outcomes.
(b) Controlling the $0.1$-quantile of the latency ensures that at least $90\%$ of users experience delays below $30$ ms, thereby limiting the proportion of high-latency events.
}
\label{fig:quantile_latency}
\end{figure}

\subsection{Confidence Bound and P-value for the Quantile Risk}
\label{sec:qltt_conf}

In order to construct a p-value for the null hypothesis \eqref{eq:ch4_quantile_null}, we now obtain a one-sided upper confidence bound $U_\varepsilon(\lambda)$ as in \eqref{eq:confidence_bound} for the quantile risk $\mathcal{R}(\lambda)=R_q(\lambda)$, i.e.,
\begin{equation}
\label{eq:ch4_quantile_cb_goal}
\mathbb{P}\!\left(
R_q(\lambda)
\le
U_\varepsilon(\lambda)
\right)
\ge
1-\varepsilon,
\qquad \forall\,\varepsilon\in(0,1).
\end{equation}

For each hyperparameter $\lambda\in\Lambda$, write the losses evaluated on the calibration datasets as
$\ell_i(\lambda)$, $ i=1,\ldots,n$,
and let $\ell_{(1)}(\lambda)\le \cdots \le \ell_{(n)}(\lambda)$ denote the corresponding sorted values.
An upper bound $U_\varepsilon(\lambda)$ can then be obtained as the empirical $(1-q^\star(\varepsilon))$-quantile of the calibration losses \citep{howard2022sequential}, i.e.,
\begin{equation}
\label{eq:ch4_quantile_U_def}
U_\varepsilon(\lambda)
=
\ell_{(\lceil n(1-q^\star(\varepsilon))\rceil)}(\lambda),
\end{equation}
with an adjusted outage level $q^\star(\varepsilon)$ \citep{howard2022sequential}. Given the confidence bound (\ref{eq:ch4_quantile_U_def}), a p-value for the null hypothesis \eqref{eq:ch4_quantile_null} is obtained as in \eqref{eq:inverted_pval}.

\subsection{Example}

We illustrate the difference between average-risk control and quantile-risk control using the wireless scheduling application introduced in Sec.~\ref{sec:app_wireless}. In this example, the performance metric of interest is the packet delay experienced by the highest-priority QoS class. As discussed in Sec. \ref{sec:quantile_risk}, the goal of the network operator is to ensure that packet delays remain below a target threshold with high probability.

We compare LTT, reviewed in Chapter~\ref{chapter:mht}, which controls the average delay by enforcing the constraint
$\mathbb{E}[L_\lambda] \le \alpha$,
with target $\alpha=10$ ms, with the quantile-risk control method described in this section. The latter, referred to as quantile LTT (QLTT), enforces the requirement
$R_{0.1}(\lambda) \le \alpha$,
ensuring that at least $90\%$ of packets experience delays below $10$ ms.

Fig.~\ref{fig:quantile_wireless_example} shows the empirical distribution of packet delays obtained using the hyperparameters selected by the two procedures. The histogram is evaluated over test-time realizations of packet delays for a fixed calibration dataset. Note that this differs from the plot in Fig. \ref{fig:fmnist_risk_distribution}, which describes the variability of the average risk across realizations of the calibration data. The hyperparameter selected using LTT achieves a mean delay below the target threshold. However, the distribution exhibits a heavy right tail, resulting in a $0.9$-quantile that exceeds the desired latency target. In contrast, the hyperparameter selected using QLTT produces a distribution with significantly reduced tail behavior, ensuring that the $90^\text{th}$ percentile of the delay remains below the target value $\alpha = 10$ ms. \textcolor{myblue}{It is also worth noting that the mean delay of the QLTT-selected configuration is lower than that of the LTT-selected configuration. This is a secondary consequence of the different certification constraints. QLTT's quantile requirement admits only configurations whose delay distributions are tightly concentrated below the threshold, and such light-tail schedulers also tend to incur lower average delays. LTT's certified set $\hat{\Lambda}$, by contrast, spans a wider range of tail behaviors, and among these LTT selects the configuration that maximises the scheduler's empirical reward, not necessarily the one with the smallest mean delay. The configuration chosen by QLTT may indeed lie within $\hat{\Lambda}$, but LTT's secondary selection rule passes it over in favor of a more aggressively tuned scheduler that operates closer to the mean-delay threshold.}

\begin{figure}[t]
\centering
\includegraphics[width=0.8\linewidth]{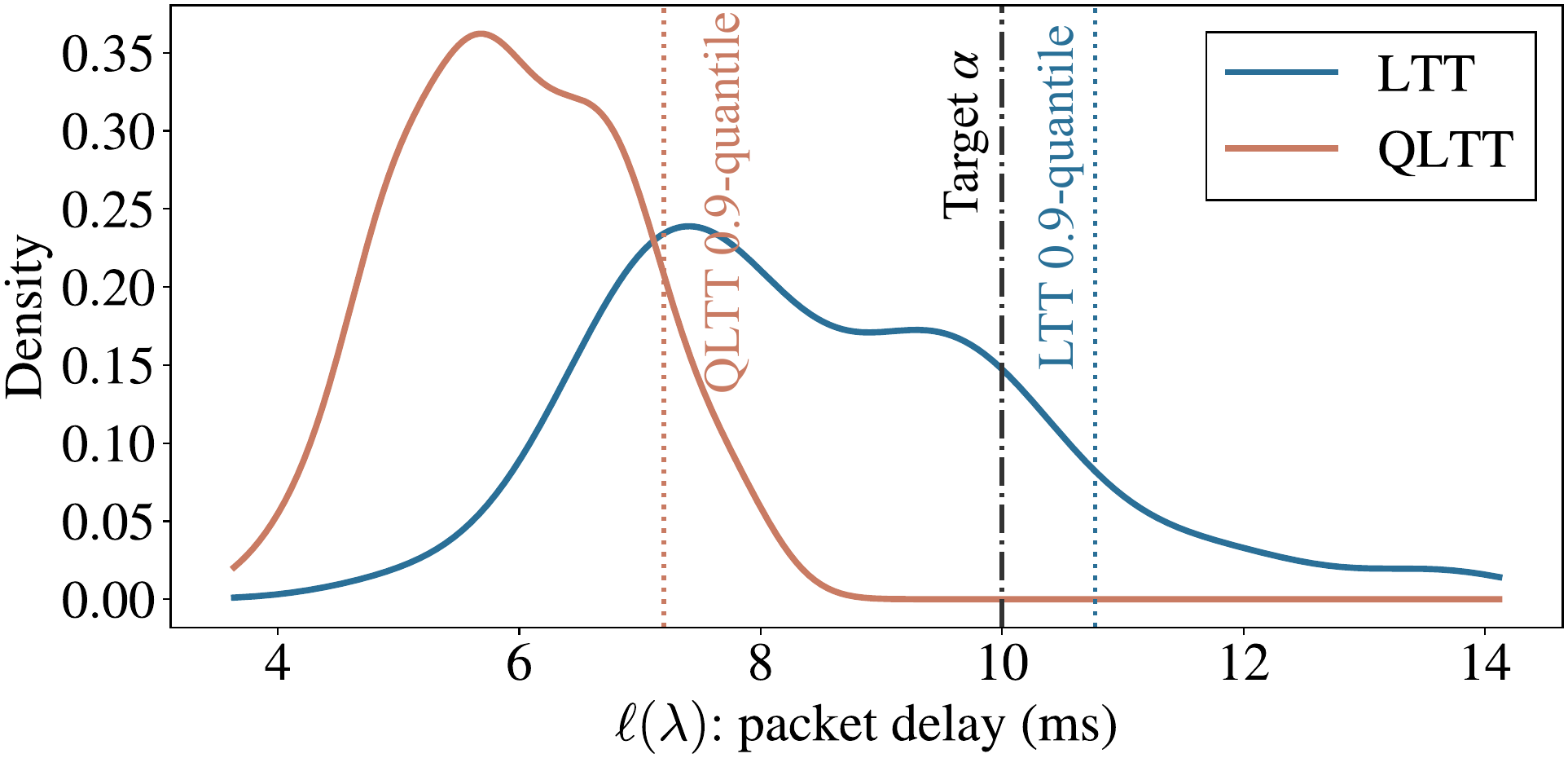}
\caption{
Distribution of packet delays obtained using hyperparameters selected by two statistically valid procedures in the wireless scheduling setting of Sec.~\ref{sec:app_wireless}.
The blue histogram corresponds to average-risk control, while the red histogram corresponds to quantile-risk control.
The dashed vertical lines indicate the corresponding $0.9$-quantiles of the delay distributions. }
\label{fig:quantile_wireless_example}
\end{figure}

\section{Information-Theoretic Constraints}
\label{sec:ch4_ib}

In this section, we describe another variant of LTT that addresses the information bottleneck problem.

\subsection{The Information Bottleneck Problem}

The information bottleneck (IB) problem \citep{tishby2000information,saxe2019information} formalizes representation learning as a trade-off between compression and relevance.
To describe it, consider an input random variable $X$ and a target variable $Y$.
The goal is to extract from input $X$ a compressed representation $T$ that preserves information about the target variable $Y$. For example, input $X$ may be a text, $Y$ a topic of interest, and $T$ a summary of the text that preserves information about topic $Y$.

The IB problem leverages the information-theoretic measure of statistical dependence known as mutual information. Given two discrete random variables $A$ and $B$ with joint distribution $P(A,B)$, the mutual information is defined as
\begin{equation}
I(A;B)=\sum_{a,b} P(a,b)\log \frac{P(a,b)}{P(a)P(b)},
\end{equation}
where $P(A)$ and $P(B)$ are the marginal distributions obtained from the joint distribution $P(A,B)$.

With this definition, the IB problem seeks a mechanism $\mathbb{P}_{T|X}$ that generates a representation $T$ from input $X$ by balancing two competing objectives:
\begin{itemize}
    \item \emph{Compression:} Minimize the mutual information $I(X;T)$ so as to compress the representation $T$;
    \item \emph{Relevance:} Maximize the mutual information $I(T;Y)$, so that the representation $T$ retains information about the target $Y$.
\end{itemize}

Specifically, the IB problem is defined as the optimization problem in which the mutual information terms $I(X;T)$ and $I(T;Y)$ are replaced by their estimates $\hat{I}(X;T)$ and $\hat{I}(T;Y)$, respectively:
\begin{equation}
\label{eq:IB_problem}
    \min_{P_{T\mid X}} \; \hat{I}(X;T)-\lambda \hat{I}(T;Y)
\end{equation}
over the conditional distribution $\mathbb{P}_{T|X}$, where $\lambda\geq 0$ is a hyperparameter. These can be obtained using a dataset of pairs $(X,Y)$, e.g., using plug-in or neural estimators \citep{belghazi2018mutual, poole2019variational}.

The choice of the hyperparameter $\lambda$ in \eqref{eq:IB_problem} is dictated by the desired trade-off between relevance and compression. Ideally, hyperparameter $\lambda$ should ensure that the solution of problem \eqref{eq:IB_problem}, which we denote as $T_\lambda \sim P^\lambda_{T\mid X}$, ensures the relevance constraint
\begin{equation}
\label{eq:ch4_ib_req}
I(T_\lambda;Y)\ge \alpha.
\end{equation}
Accordingly, for every hyperparameter \(\lambda\in\Lambda\), we define the null hypothesis
\begin{equation}
\label{eq:ch4_ib_null}
\mathcal{H}_\lambda:\quad I(T_\lambda;Y)<\alpha.
\end{equation}
Rejecting the null hypothesis \eqref{eq:ch4_ib_null} means certifying that hyperparameter \(\lambda\) satisfies the desired relevance requirement \eqref{eq:ch4_ib_req}.

As shown in \citep{farzaneh2025ibmht}, statistically valid testing for \eqref{eq:ch4_ib_null} can be obtained by constructing a one-sided \emph{lower} confidence bound on the mutual information \(I(T_\lambda;Y)\) \citep{stefani2014confidence}, and then inverting this bound to form valid p-values or e-values (see Sec. \ref{sec:confidence_bounds}). This confidence bound is obtained by using a calibration dataset which is distinct from the dataset used for estimates in \eqref{eq:IB_problem}.

Once a subset $\hat{\Lambda}$ of reliable hyperparameters has been identified using this procedure, we choose the final hyperparameter as
\begin{equation}
\label{eq:ch4_ib_selection}
\lambda^\star
\in
\arg\min_{\lambda\in\hat{\Lambda}}
\widehat I(X;T_\lambda),
\end{equation}
where \(\widehat I(X;T_\lambda)\) is the estimate of the compression term \(I(X;T_\lambda)\) in \eqref{eq:IB_problem}. We refer to this approach as IB-LTT.

\subsection{Example}
\label{sec:ib_example}

We illustrate the IB-LTT using a variational IB model \citep{alemi2016deep}, following the experimental setting detailed in \citep{farzaneh2025ibmht}.
We set a target relevance level \(\alpha=2.28\) and outage level \(\delta=0.1\).
To visualize the performance of the different selection strategies, Fig.~\ref{fig:ib_mht_example} reports the joint distribution of the mutual informations $I(X;T)$ and $I(T;Y)$ evaluated across 100 realizations of the calibration dataset.

In the figure, the relevance constraint \(I(T;Y)\ge \alpha\) is visualized by a vertical threshold at \(\alpha\).
Outage events correspond to points that lie to the left of this threshold, i.e., trials (calibration dataset realizations) in which the selected representation fails to meet the required relevance level. E-HPO yields an empirical outage probability of approximately \(0.27\), indicating frequent violations of the relevance requirement \eqref{eq:ch4_ib_req}.
In contrast, IB-LTT concentrates the mass of the joint distribution in the figure to the right of the threshold, reducing the outage probability to about \(0.07\), thereby satisfying the desired reliability level.

\begin{figure}[t]
    \centering
    \includegraphics[width=0.8\linewidth]{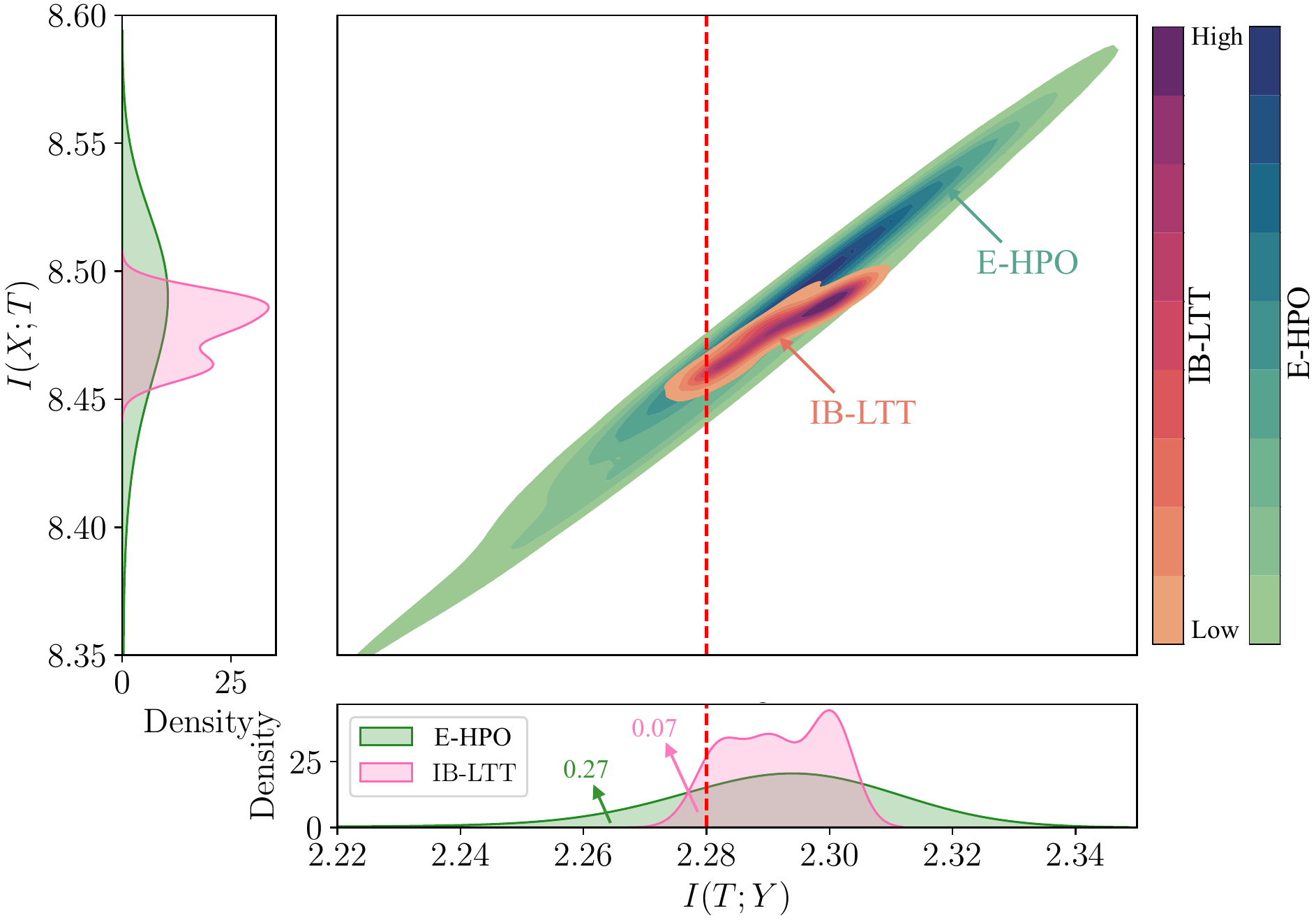}
    \caption{
    Joint distribution of relevance and compression for the hyperparameter \(\lambda^\star\) selected by E-HPO and IB-LTT across 100 independent trials. The horizontal axis shows the achieved relevance \(I(T;Y)\) and the vertical axis shows the compression \(I(X;T)\).
    The vertical dashed line indicates the target relevance level \(\alpha\).}
    \label{fig:ib_mht_example}
\end{figure}

\section{Summary}

This chapter extended the LTT framework beyond the classical setting of average risk control to accommodate more general and practically relevant notions of reliability. We showed that the MHT perspective introduced in Chapter~\ref{chapter:mht} naturally generalizes to arbitrary risk functionals of the loss distribution, provided that suitable confidence bounds can be constructed.

The key methodological insight is that statistical validity can be preserved by formulating reliability requirements in terms of hypotheses of the form $\mathcal{R}(\lambda) > \alpha$, and by constructing p-values via the inversion of one-sided confidence bounds. This yields a unified and modular approach, in which different reliability criteria can be handled by designing appropriate confidence bounds for the corresponding functional $\mathcal{R}(\lambda)$.

We instantiated this framework along two main directions. First, we considered quantile risk control, which enables guarantees on tail behavior of the loss distribution. The resulting QLTT procedure provides protection against rare but severe failures, as illustrated in the wireless scheduling example, where it enforces constraints on high-percentile latency rather than on average delay.

Second, we addressed reliability constraints involving information-theoretic quantities through the information bottleneck framework. The proposed IB-LTT method certifies relevance in terms of mutual information with statistical guarantees, and then selects among reliable configurations by optimizing compression. This yields a principled approach to hyperparameter selection that mirrors the underlying IB objective while ensuring reliability.

Overall, this chapter demonstrates that LTT is not limited to average-risk guarantees, but can serve as a general-purpose framework for statistically valid hyperparameter selection under a broad class of reliability criteria, including tail risks, constraint violations, and information-theoretic objectives.

\chapter{Multi-Objective Hyperparameter Selection}
\label{chapter:multi_objective}

The previous chapters focused on selecting hyperparameters that satisfy a single reliability constraint, namely that the risk $R(\lambda)$ lies below a user-specified threshold $\alpha$. In many practical deployments, however, a single scalar objective is insufficient to capture the full range of system requirements.

Modern AI systems often operate under multiple, potentially conflicting performance criteria. For instance, an LLM may need to be accurate, well-calibrated, and computationally efficient \citep{10.1098/rsta.2024.0522}; a wireless scheduler must balance latency, throughput, and fairness \citep{kuo2007utility}; and a computer vision system must jointly optimize detection recall and segmentation quality \citep{he2017mask}. In such settings, the practitioner faces not one but several objectives, some of which are hard constraints to be satisfied with statistical guarantees, while others are soft objectives to be optimized in a best-effort manner.

This chapter introduces a principled framework for multi-objective hyperparameter selection that extends the LTT methodology introduced in Chapter~\ref{chapter:mht}. As detailed in Sec.~\ref{sec:ch5_problem}, the central challenge is to simultaneously satisfy a number of reliability constraints with statistical guarantees, while making a best effort at optimizing a number of additional unconstrained objectives. The chapter introduces \textit{Pareto testing} (PT) \citep{laufer-goldshtein2023efficiently} in Sec.~\ref{sec:ch5_pareto_testing}. PT addresses this challenge by combining multi-objective optimization with sequential hypothesis testing to efficiently identify configurations on the reliability-performance frontier. Then, it describes \textit{reliability graph-based Pareto testing} (RG-PT) \citep{farzaneh2025multiobjective} in Sec.~\ref{sec:ch5_rg}. RG-PT further incorporates structured prior knowledge about expected reliability relationships among hyperparameters via a directed acyclic graph \citep{ramdas2017dagger}.

\section{Multi-Objective Problem Formulation}
\label{sec:ch5_problem}

We begin by formalizing the multi-objective hyperparameter selection problem. As in Chapter~\ref{chapter:mht}, let $\Lambda = \{\lambda_1, \ldots, \lambda_K\}$ denote a finite candidate set, and let $f_\lambda$ denote a pre-trained model operating under hyperparameter $\lambda$.

In a multi-objective setting, we associate with each hyperparameter $\lambda$ a collection of $L$ risk functions
\begin{equation}
\label{eq:ch5:risk_def}
    R_l(\lambda) = \mathbb{E}[L_\lambda^l(Z)], \quad l = 1, \ldots, L,
\end{equation}
where $Z$ denotes a generic test-time data point and $L_\lambda^l(Z) \in [0,1]$ is a bounded per-sample loss for objective $l$. Note that, in a supervised learning setup, data point $Z$ includes input $X$ and output $Y$, i.e., $Z = (X,Y)$.

Given a dataset $\mathcal{D}$, the empirical estimate of $R_l(\lambda)$ is computed as
\begin{equation}
    \widehat{R}_l(\lambda \mid \mathcal{D}) = \frac{1}{|\mathcal{D}|}\sum_{Z\in\mathcal{D}} L_\lambda^l(Z).
\end{equation}

We partition the $L$ objectives into two groups:
\begin{itemize}
    \item \textbf{Reliability risk functions} $\{R_l(\lambda)\}_{l=1}^{L_C}$: The first set of $L_C$ risk functions must be controlled below user-specified thresholds $\{\alpha_l\}_{l=1}^{L_C}$. Accordingly, a hyperparameter $\lambda$ is said to be \emph{reliable} if the following inequalities hold.
    \begin{equation}
    \label{eq:ch5_reliability}
        R_l(\lambda) \le \alpha_l \quad \text{for all } l = 1, \ldots, L_C.
    \end{equation}
    \item \textbf{Auxiliary risk functions} $\{R_l(\lambda)\}_{l=L_C+1}^{L}$: The remaining $L - L_C$ risk functions are unconstrained objectives that the practitioner wishes to minimize as a best-effort secondary goal.
\end{itemize}

Fig.~\ref{fig:ch5_wireless_multi_obj} illustrates a concrete example of
multi-objective deployment in a wireless network. A base station serves multiple devices (UEs) and must simultaneously satisfy three system-level requirements $(L_C = 3)$, namely a packet error rate below
$10^{-5}$, a coverage above $95\%$, and a delay below $30$ ms.
Violating any of these three criteria can render the system unreliable for its intended use. Under these requirements, it is also desired to make a best effort at maximizing the data rate ($L = 4$, and hence $L - L_C= 1$).

\begin{figure}[t]
    \centering
    \includegraphics[width=0.7\linewidth]{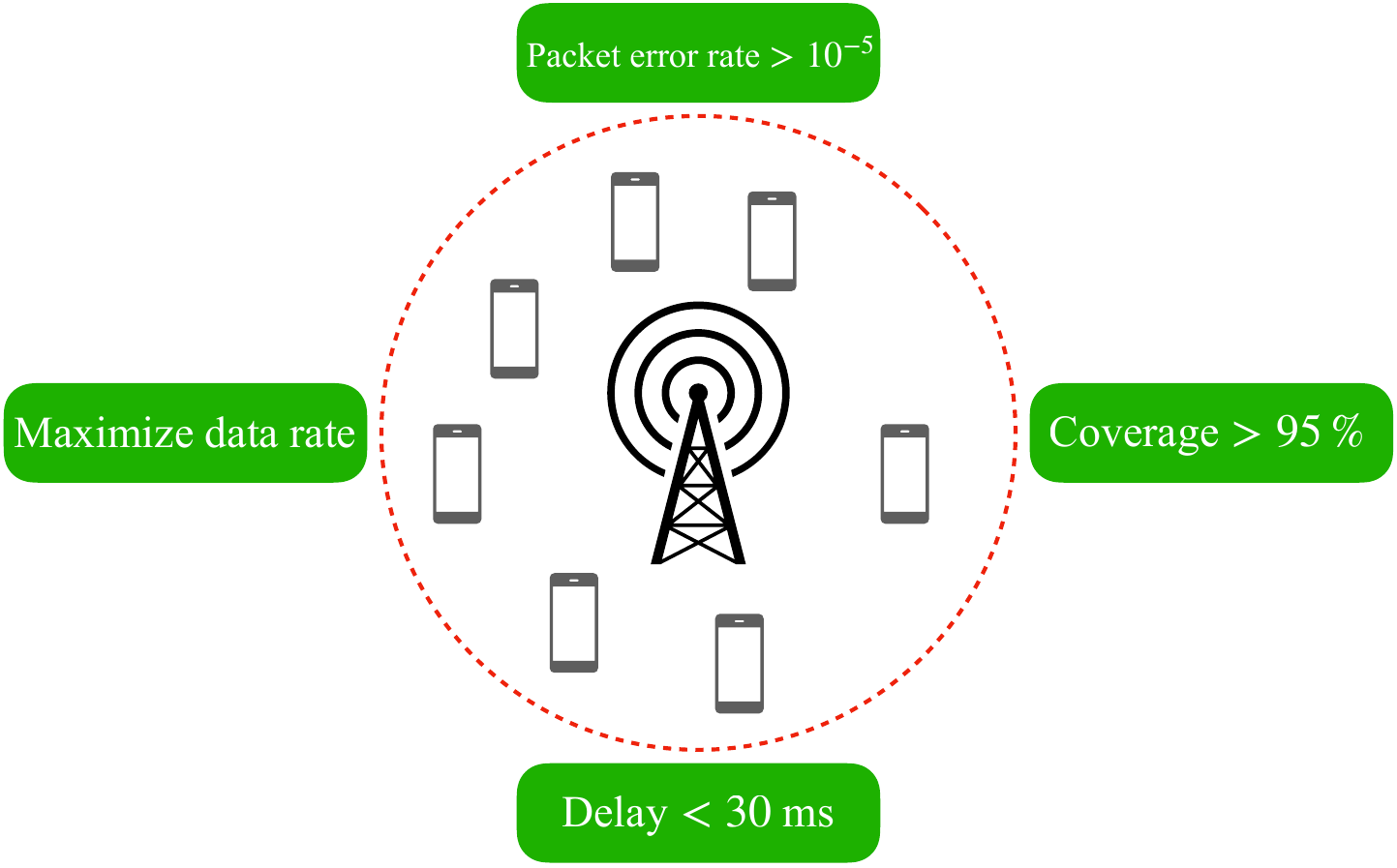}
    \caption{Example of multi-objective reliability requirements in a wireless network.
    A base station serving multiple devices (UEs) must simultaneously satisfy three
    system-level constraints, namely packet error rate, coverage, and delay, while maximizing the data rate.
    Multi-objective hyperparameter selection must certify configurations that satisfy
    all three constraints simultaneously with high probability, while making a best effort at maximizing the data rate.}
    \label{fig:ch5_wireless_multi_obj}
\end{figure}

Multi-objective hyperparameter selection is formalized by the optimization problem
\begin{equation}
\label{eq:ch5_mo_problem}
\begin{aligned}
    &\min_{\lambda \in \Lambda} \; \{R_{L_C+1}(\lambda), \ldots, R_L(\lambda)\} \\
    &\text{subject to} \; R_l(\lambda) \le \alpha_l \;\; \forall\, l = 1, \ldots, L_C.
\end{aligned}
\end{equation}
As with the single-objective case studied in the previous chapters, the risks $R_l(\lambda)$ in \eqref{eq:ch5:risk_def} depend on the unknown data distribution and can only be estimated from finite calibration data. The key challenge is therefore to address problem (\ref{eq:ch5_mo_problem}) using only calibration data $\mathcal{D}^\mathrm{cal}$, while providing formal statistical guarantees on the reliability constraints (\ref{eq:ch5_reliability}) in the optimization \eqref{eq:ch5_mo_problem}.

\begin{definition}[$(\alpha, \delta)$-risk-controlling configuration]
\label{def:ch5_mo_rcp}
A hyperparameter $\hat{\lambda} \in \Lambda$ selected using calibration data is said to be $(\alpha, \delta)$-risk-controlling for the risks $\{R_l(\lambda)\}_{l = 1}^{L_C}$, where $\alpha = (\alpha_1, \ldots, \alpha_{L_C})$, if the following condition holds.
\begin{equation}
\label{eq:ch5_mo_guarantee}
    \mathbb{P}\!\left( R_l(\hat{\lambda}) \le \alpha_l \;\; \forall\, l = 1, \ldots, L_C \right) \ge 1 - \delta,
\end{equation}
where the probability is over the randomness in the calibration data.
\end{definition}

Definition~\ref{def:ch5_mo_rcp} generalizes Definition~\ref{def:ch2_rcp} from Chapter~\ref{chapter:mht} to the case of multiple simultaneous reliability constraints.

\section{Pareto Optimality and the Pareto Frontier}
\label{sec:ch5_pareto}

Before introducing the testing procedures, we recall the notion of Pareto optimality that underlies multi-objective optimization \citep{deb2011multi}.

Given a set of risk functions $\{R_l(\lambda)\}_{l = 1}^L$, for any two hyperparameters $\lambda\;\text{and}\; \lambda' \in \Lambda$, we say that hyperparameter $\lambda'$ \emph{dominates} hyperparameter $\lambda$, written $\lambda' \prec \lambda$, if
\begin{equation}
    R_l(\lambda') \le R_l(\lambda) \;\; \text{for all } l,
    \quad \text{and} \quad
    R_l(\lambda') < R_l(\lambda) \;\; \text{for at least one } l.
\end{equation}
That is, hyperparameter $\lambda'$ is at least as good as hyperparameter $\lambda$ on every objective and strictly better on at least one objective $l\in \{1, \ldots, L\}$.

The \emph{Pareto optimal set} $\Lambda^{\mathrm{Par}}$ consists of all non-dominated hyperparameters, i.e.,
\begin{equation}
\label{eq:ch5_pareto_set}
    \Lambda^{\mathrm{Par}} = \left\{ \lambda \in \Lambda \;:\; \nexists\, \lambda' \in \Lambda \;\text{with}\; \lambda' \ne \lambda \text{ such that } \lambda' \prec \lambda \right\}.
\end{equation}
Accordingly, the hyperparameter settings in the set $\Lambda^{\mathrm{Par}}$ capture the best achievable trade-offs among all objectives. Any hyperparameter outside the set $\Lambda^{\mathrm{Par}}$ is dominated and can be discarded without loss of optimality.

Fig.~\ref{fig:ch5_pareto_frontier} illustrates the Pareto frontier for a two-objective setting with $R_1(\lambda)$ and $R_2(\lambda)$. Dots represent the risk pairs $(R_1(\lambda), R_2(\lambda))$ achievable with some choice of hyperparameter $\lambda \in \Lambda$. Configurations on the frontier, displayed as green dots, represent the best achievable trade-offs between the two risks: no other configuration simultaneously improves both objectives.

\begin{figure}[t]
    \centering
    \includegraphics[width=0.65\linewidth]{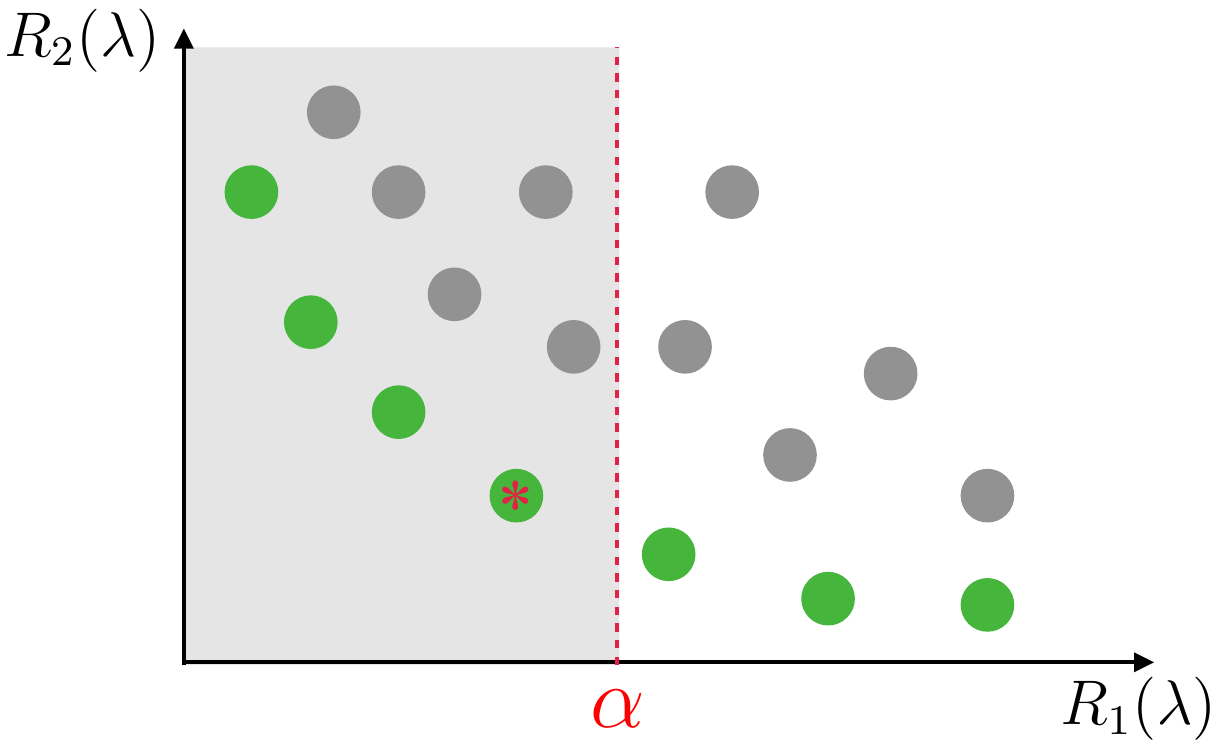}
    \caption{Illustration of the Pareto frontier in a two-objective setting with one constrained objective $R_1(\lambda)$ (to be controlled below $\alpha$) and one free objective $R_2(\lambda)$ (to be minimized). Each dot represents the risk pair $(R_1(\lambda), R_2(\lambda))$ attained with some hyperparameter choice $\lambda$. The Pareto frontier (green dots) identifies the hyperparameters $\lambda$ yielding the best achievable trade-offs. The shaded region marks configurations that satisfy the reliability constraint $R_1(\lambda) \le \alpha$.}
    \label{fig:ch5_pareto_frontier}
\end{figure}

Returning to the multi-objective optimization formulation in problem \eqref{eq:ch5_mo_problem}, assume now that $R_1(\lambda)$ is a reliability risk function, while $R_2(\lambda)$ is an auxiliary risk function, i.e., we have $L_C = 1$ and $L = 2$. Then, as shown in Fig. \ref{fig:ch5_pareto_frontier}, only the configurations with $R_1(\lambda)\le \alpha$ are feasible for problem \eqref{eq:ch5_mo_problem} (shaded area), and the solution of problem \eqref{eq:ch5_mo_problem} corresponds to the dot marked with an asterisk.

\section{Multi-Objective Hypothesis Testing}
\label{sec:ch5_combined_pval}

 As discussed in the previous section, problem \eqref{eq:ch5_mo_problem} must be addressed using only estimates of the risks based on calibration data. To account for this, in order to enforce the constraints in \eqref{eq:ch5_mo_problem}, we adopt an MHT framework in which the null hypothesis for each candidate hyperparameter $\lambda_k \in \Lambda$ generalizes the single-risk hypothesis in (\ref{eq:ch2_null}) as
\begin{equation}
\label{eq:ch5_null}
    \mathcal{H}_k: \quad \exists\, l \in \{1, \ldots, L_C\} \text{ such that } R_l(\lambda_k) > \alpha_l.
\end{equation}
Rejecting hypothesis $\mathcal{H}_k$ means certifying that hyperparameter $\lambda_k$ satisfies all reliability constraints \eqref{eq:ch5_reliability} simultaneously.

As in Chapter \ref{chapter:mht}, testing for hypothesis $\mathcal{H}_k$ can be carried out using p-values or e-values. A p-value for the combined null hypothesis (\ref{eq:ch5_null}) can be obtained by taking the maximum of the individual p-values for the constituent null hypotheses $R_l(\lambda_k) > \alpha_l$ for all $l \in \{1, \ldots, L_C\}$. Conversely, an e-value for the null hypothesis \eqref{eq:ch5_null} can be evaluated as the minimum of the individual e-values.

Focusing on the case of p-values, let $p_{l,k}$ be a valid p-value for the hypothesis
\begin{equation}
\label{eq:ch5:single_hypothesis}
\mathcal{H}_{k,l}: R_l(\lambda_k) > \alpha_l
\end{equation}
for each $l = 1, \ldots, L_C$, as described in Chapter~\ref{chapter:mht}. Then, the combined statistic
\begin{equation}
\label{eq:ch5_combined_pval}
    p_k = \max_{1 \le l \le L_C} p_{l,k}
\end{equation}
is a valid p-value for the null hypothesis $\mathcal{H}_k$ in (\ref{eq:ch5_null}), although it may be conservative.

\begin{lemma}[Combined p-value validity]
\label{lem:ch5_combined_pval}
If $p_{l,k}$ is a valid p-value for hypothesis $\mathcal{H}_{k,l}$ in \eqref{eq:ch5:single_hypothesis} for each $l \in \{1, \ldots, L_C\}$, then the statistic (\ref{eq:ch5_combined_pval}) satisfies the superuniformity condition
\begin{equation}
    \mathbb{P}(p_k \le u \mid \mathcal{H}_k) \le u \quad \text{for all } u \in [0,1],
\end{equation}
and is thus a p-value for the null hypothesis $\mathcal{H}_k$ in \eqref{eq:ch5_null}.
\end{lemma}
\begin{proof}
Let $I \subseteq \{1, \ldots, L_C\}$ denote the set of indices $l$ for which the null hypothesis $\mathcal{H}_{k,l}$
is true, i.e., for which $R_l(\lambda_k) > \alpha_l$. Under the combined null $\mathcal{H}_k$,
the set $I$ is non-empty. For any $u \in [0,1]$, we have
\begin{subequations}\label{eq:proof_combined}
\begin{align}
    \mathbb{P}(p_k \le u \mid \mathcal{H}_k)
    &= \mathbb{P}\!\left(\max_{1 \le l \le L_C} p_{l,k} \le u \;\middle|\; \mathcal{H}_k\right)
    \label{eq:proof_combined_a}\\
    &\le \mathbb{P}\!\left( p_{l^*,k} \le u \;\middle|\; \mathcal{H}_{k,l^*}\right)
    \label{eq:proof_combined_b}\\
    &\le u,
    \label{eq:proof_combined_c}
\end{align}
\end{subequations}
where $l^*\in I$ indicates any true null hypothesis $\mathcal{H}_{k,l^*}$, i.e., $l^* \in I$.
The inequality (\ref{eq:proof_combined_b}) holds because the event
$\{p_{l^*,k} \le u\}$ includes the event $\{\max_l p_{l,k} \le u\}$, and hence
has larger or equal probability. The final inequality (\ref{eq:proof_combined_c}) follows
from the superuniformity property (\ref{eq:ch2_pvalue_validity}) of the p-value $p_{l^*,k}$ under the null hypothesis
$\mathcal{H}_{k,l^*}$.
\end{proof}

\section{Pareto Testing}
\label{sec:ch5_pareto_testing}

By Lemma~\ref{lem:ch5_combined_pval}, the combined p-values $\{p_k\}_{k=1}^K$ can be fed directly into any of the MHT procedures reviewed in Sec.~\ref{sec:ch2_pval_mht}, such as Bonferroni correction \eqref{eq:ch2_bonferroni} or FST \eqref{eq:ch2_fst_kstar}, to obtain a certified set $\hat{\Lambda}$ of hyperparameters that are simultaneously reliable across all $L_C$ constraints with probability at least $1 - \delta$.

However, testing all $K$ hyperparameter configurations via Bonferroni or standard FST can result in a small certified set $\hat{\Lambda}$, especially when the number of candidates $K$ and the number of constrained objectives $L_C$ are large. A small set $\hat{\Lambda}$, or even an empty set $\hat{\Lambda}$, offers little room for optimizing the $L - L_C$ auxiliary objectives.

PT \citep{laufer-goldshtein2023efficiently} addresses this efficiency challenge by focusing only on the configurations that are estimated to lie on the Pareto frontier in the joint objective space (see Fig. \ref{fig:ch5_pareto_frontier}). The Pareto frontier is estimated via held-out data, and the restriction of candidate hyperparameters on the estimated Pareto frontier can reduce the number of hypotheses to test, thereby improving statistical power.

\subsection{Overview}

PT splits the calibration dataset $\mathcal{D}^{\mathrm{cal}}$ into two disjoint subsets $\mathcal{D}^{\mathrm{opt}}$ and $\mathcal{D}^{\mathrm{MHT}}$, of respective sizes $n^\mathrm{opt}$ and $n^\mathrm{MHT}$, as
\begin{equation}
    \mathcal{D}^{\mathrm{cal}} = \mathcal{D}^{\mathrm{opt}} \cup \mathcal{D}^{\mathrm{MHT}}.
\end{equation}
Then, PT proceeds along the following two stages.

\textbf{Stage 1 (Optimization):} As illustrated in Fig. \ref{fig:ch5_pareto_testing_procedure}, first, PT addresses the multi-objective optimization problem \eqref{eq:ch5_mo_problem} by replacing all risks, both constrained and unconstrained, with empirical estimates obtained using the dataset $\mathcal{D}^{\mathrm{opt}}$. This can be done with any multi-objective optimizer, including exhaustive grid search for small candidate sets or methods such as Bayesian multi-objective optimization for larger spaces \citep{JMLR:v23:21-0888}. Assuming that the dataset $\mathcal{D}^{\mathrm{opt}}$ is sufficiently large, the resulting estimated Pareto set $\hat{\Lambda}^{\mathrm{Par}} \subseteq \Lambda$ contains only the most promising configurations, and is typically a small fraction of the full set $\Lambda$.

Considering only the estimated reliability risk functions $\{R_l(\lambda)\}_{l = 1}^{L_C}$, PT orders the candidate hyperparameters $\lambda_k \in \hat{\Lambda}^\mathrm{Par}$ from most likely to be reliable to least likely to be reliable. This step can be done using different measures of reliability based on empirical risk estimates obtained from dataset $\mathcal{D}^\mathrm{opt}$. Specifically, PT builds a p-value $p_{l,k}^\mathrm{opt}$ for each configuration $\lambda_k \in \hat{\Lambda}^\mathrm{Par}$ using dataset $\mathcal{D}^\mathrm{opt}$. The combined p-values \eqref{eq:ch5_combined_pval} for the null hypothesis in \eqref{eq:ch5_null}, i.e.,
\begin{equation}
\label{eq:ch5_ordering}
    p^{\mathrm{opt}}_k = \max_{1 \le l \le L_C} p_{l,k}^\mathrm{opt},
\end{equation}
are then used to order the hyperparameters in set $\hat{\Lambda}^\mathrm{Par}$ from the smallest p-value \eqref{eq:ch5_ordering} to the largest.

\textbf{Stage 2 (Testing):} Using the disjoint dataset $\mathcal{D}^{\mathrm{MHT}}$, as shown in Fig. \ref{fig:ch5_pareto_testing_procedure}, PT applies FST, as per \eqref{eq:ch2_fst_kstar}, on the candidates $\lambda_k \in \hat{\Lambda}^\mathrm{Par}$ using the order identified in Stage 1. To this end, combined p-values
\begin{equation}
\label{eq:ch5:test_pval}
    p^{\mathrm{MHT}}_k = \max_{1 \le l \le L_C} p_{l,k}^\mathrm{MHT}
\end{equation}
are computed for each $\lambda_k \in \hat{\Lambda}^{\mathrm{Par}}$ using dataset $\mathcal{D}^\mathrm{MHT}$, and FST is applied using the statistics \eqref{eq:ch5:test_pval}. This yields a certified subset $\hat{\Lambda} \subseteq \hat{\Lambda}^{\mathrm{Par}}$.

Finally, within the certified set $\hat{\Lambda}$, any hyperparameter
$\hat{\lambda} \in \hat{\Lambda}$ can be selected for deployment while satisfying the $(\alpha, \delta)$ guarantee \eqref{eq:ch5_mo_guarantee}. In order to address problem \eqref{eq:ch5_mo_problem}, the configuration $\hat{\lambda}$ is then selected by optimizing a suitable combination of the empirical estimates for the auxiliary risks $\{R_l(\lambda)\}_{l = L_C+1}^L$.

The PT procedure is summarized in Algorithm~\ref{alg:ch5_pareto_testing}. Overall, the statistical efficiency of PT over naive application of LTT to the
full candidate set $\Lambda$ stems from the fact that the Pareto set $\hat{\Lambda}^{\mathrm{Par}}$
    is typically much smaller than the full candidate set $\Lambda$, and that ordering configurations on the Pareto set places the most reliably performing configurations at the
    beginning of the testing sequence before the procedure stops at the first failure.

\begin{algorithm}[h!]
\caption{Pareto Testing (PT) \citep{laufer-goldshtein2023efficiently}}
\label{alg:ch5_pareto_testing}
\begin{algorithmic}
\STATE \textbf{Input:} Candidate set $\Lambda$; calibration data $\mathcal{D}^{\mathrm{cal}} = \mathcal{D}^{\mathrm{opt}} \cup \mathcal{D}^{\mathrm{MHT}}$; risk thresholds $\alpha = (\alpha_1, \ldots, \alpha_{L_C})$; error level $\delta$

\STATE \textbf{Output:} Certified set $\hat{\Lambda} \subseteq \Lambda$

\medskip
\STATE \textit{Stage 1: Optimization}
\STATE Compute empirical risks $\widehat{R}_l(\lambda \mid \mathcal{D}^{\mathrm{opt}})$ for all $\lambda \in \Lambda$ and $l = 1, \ldots, L$
\STATE Compute Pareto frontier $\hat{\Lambda}^{\mathrm{Par}}$ using the empirical risks $\{\widehat{R}_l(\lambda \mid \mathcal{D}^{\mathrm{opt}})\}_{l = 1}^L$
\FOR{each $\lambda_k \in \hat{\Lambda}^{\mathrm{Par}}$}
    \STATE Compute combined p-value $p^{\mathrm{opt}}_k = \max_{l=1}^{L_C} p_{l,k}^\mathrm{opt}$
\ENDFOR
\STATE Order the hyperparameters in set $\hat{\Lambda}^{\mathrm{Par}}$ by increasing
$p^{\mathrm{opt}}_k$ to obtain the ordered sequence $(\lambda^{(1)}, \lambda^{(2)}, \ldots,
\lambda^{(|\hat{\Lambda}^{\mathrm{Par}}|)})$

\medskip
\STATE $k \leftarrow 1$
\FOR{$k = 1,\ldots, |\hat{\Lambda}^{\mathrm{Par}}|$}
    \STATE Compute $p^{\mathrm{MHT}}_{(k)} = \max_{l=1}^{L_C} p_{l,(k)}^{\mathrm{MHT}}$
    for the $k$-th configuration $\lambda^{(k)}$ in the ordered sequence
    \IF{$p^{\mathrm{MHT}}_{(k)} \leq \delta$}
        \STATE Add $\lambda^{(k)}$ to $\hat{\Lambda}$
    \ELSE
        \STATE \textbf{break}
    \ENDIF
\ENDFOR
\STATE \textbf{return} the certified set $\hat{\Lambda}$ of hyperparameters
\end{algorithmic}
\end{algorithm}

\begin{figure}[t]
    \centering
    \includegraphics[width=\linewidth]{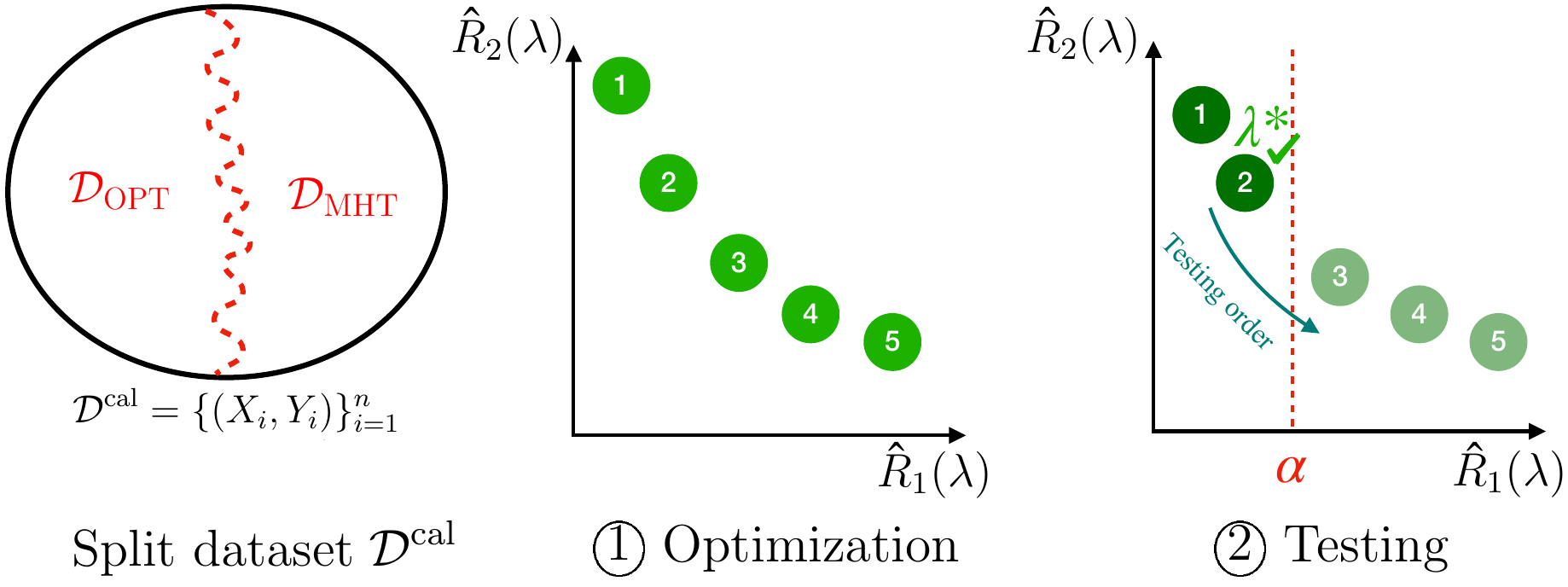}
    \caption{Illustration of the two stages of Pareto Testing (PT) for two risk functions, i.e., $L = 2$, with one reliability risk function $R_1(\lambda)$ and one auxiliary risk function $R_2(\lambda)$, i.e., $L_C = 1$. Having partitioned the dataset $\mathcal{D}^\mathrm{cal}$ into splits $\mathcal{D}^{\mathrm{opt}}$ and $\mathcal{D}^{\mathrm{MHT}}$, the empirical Pareto frontier is estimated using dataset $\mathcal{D}^{\mathrm{opt}}$,
    yielding a set of candidate configurations ordered from most to least likely to be reliable.
Sequential MHT via FST is applied to the ordered Pareto frontier
    using dataset $\mathcal{D}^{\mathrm{MHT}}$. Testing proceeds along the frontier in the predetermined
    order and stops at the first failure to reject the null hypothesis \eqref{eq:ch5_null}. Configurations tested before the first failure (dark green) are certified as reliable, and the final deployed hyperparameter
    $\lambda^*$ is selected among them by considering secondary objectives.}
    \label{fig:ch5_pareto_testing_procedure}
\end{figure}

\subsection{Statistical Guarantee}

The following result establishes that PT achieves the $(\alpha, \delta)$-risk-controlling guarantee of Definition~\ref{def:ch5_mo_rcp}.

\begin{theorem}[Validity of Pareto Testing]
\label{thm:ch5_pt_validity}
Let $\mathcal{D}^{\mathrm{opt}}$ and $\mathcal{D}^{\mathrm{MHT}}$ be disjoint i.i.d.\ subsets of the calibration data. Then, every hyperparameter $\hat{\lambda} \in \hat{\Lambda}$ returned by PT (Algorithm~\ref{alg:ch5_pareto_testing}) is $(\alpha, \delta)$-risk-controlling in the sense of Definition~\ref{def:ch5_mo_rcp}.
\end{theorem}
\begin{proof}
Since datasets $\mathcal{D}^{\mathrm{opt}}$ and $\mathcal{D}^{\mathrm{MHT}}$ are disjoint and i.i.d., the ordering of the hyperparameters in set $\hat{\Lambda}^{\mathrm{Par}}$ computed over dataset $\mathcal{D}^{\mathrm{opt}}$ is independent of the p-values $p^{\mathrm{MHT}}_k$ computed over $\mathcal{D}^{\mathrm{MHT}}$. The combined p-values $p^{\mathrm{MHT}}_k$ are therefore valid for the composite null hypotheses $\mathcal{H}_k$ in (\ref{eq:ch5_null}) by Lemma~\ref{lem:ch5_combined_pval}. Applying FST at level $\delta$ over the fixed predetermined ordering then controls the FWER at level $\delta$ as proved in Section~\ref{sec:ch2_pval_mht}. Hence, with probability at least $1 - \delta$, every $\lambda \in \hat{\Lambda}$ satisfies all reliability constraints (\ref{eq:ch5_reliability}).
\end{proof}

\subsection{Example: Multi-Objective Wireless Scheduling}
\label{sec:ch5_app_wireless}

We revisit the wireless scheduling application introduced in
Sec.~\ref{sec:app_wireless} to illustrate the benefits of PT over single-objective LTT
in a multi-objective setting. Recall that the system consists of a base station allocating
radio resources to UEs across multiple QoS classes,
with the hyperparameter vector $\lambda = (\lambda_1, \lambda_2, \lambda_3, \lambda_4)$
controlling the weights of different scheduling criteria.

We consider a two-objective problem with $L_C = 2$ and $L = 2$. The first
reliability risk function $R_1(\lambda)$ is the average packet delay of UEs in QoS
class 1, which must be controlled below the threshold $\alpha_1 = 10$ ms. The second
reliability risk function $R_2(\lambda)$ is the average packet delay of UEs in QoS
class 2, which must be controlled below the level $\alpha_2 = 12$ ms. The target failure
probability is $\delta = 0.1$.

We compare three procedures: PT, which controls both constraints
simultaneously; LTT controlling only $R_1(\lambda) \le \alpha_1$; and LTT controlling only $R_2(\lambda) \le \alpha_2$. Fig.~\ref{fig:ch5_wireless_pt}
shows the joint distribution of the risks $(R_1(\hat{\lambda}), R_2(\hat{\lambda}))$
of the selected hyperparameter over 100 independent calibration splits for each method,
with the reliability thresholds $\alpha_1$ and $\alpha_2$ shown as dashed lines.

Reflecting their theoretical performance, the first version of LTT successfully controls the risk $R_1(\lambda)$ below the target value
$\alpha_1$, but provides no guarantee on risk $R_2(\lambda)$, with a fraction $0.83$
of runs exceeding the target $\alpha_2$. Similarly, the second instantiation of LTT controls the risk $R_2(\lambda)$ but not
$R_1(\lambda)$, on which it exhibits a violation rate of $0.87$. In contrast, PT controls both risks simultaneously, with the joint
distribution concentrated in the safe region defined by the constraints $R_1(\lambda) \le \alpha_1$
and $R_2(\lambda) \le \alpha_2$.

\begin{figure}[t]
    \centering
    \includegraphics[width=\linewidth]{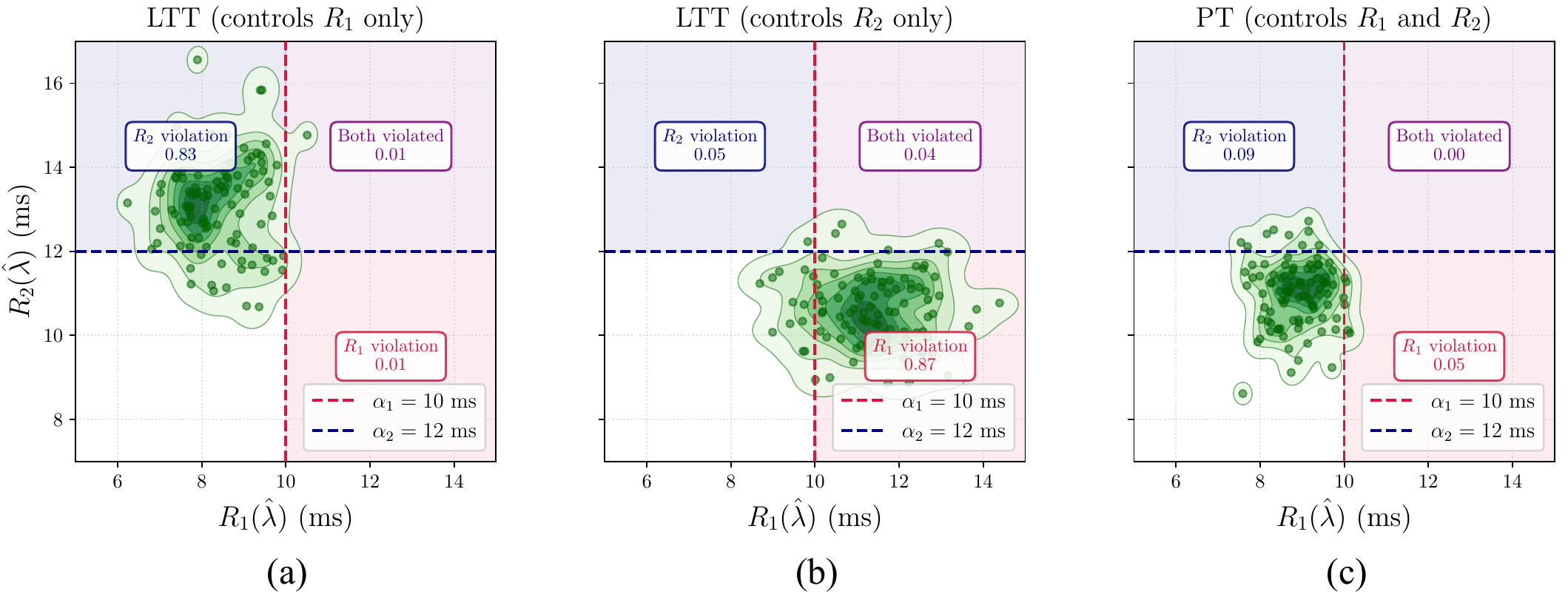}
    \caption{Joint distribution of the risks $(R_1(\hat{\lambda}), R_2(\hat{\lambda}))$
    of the selected hyperparameter over 100 calibration splits in the wireless scheduling
    setting of Sec.~\ref{sec:app_wireless}, for three methods: (a) LTT controlling only
    $R_1(\lambda)$; (b) LTT controlling only $R_2(\lambda)$; and (c) PT
    controlling both $R_1(\lambda)$ and $R_2(\lambda)$ simultaneously. Dashed lines mark the reliability thresholds
    $\alpha_1$ and $\alpha_2$. The shaded region beyond the threshold lines corresponds
    to reliability violations. The fraction of runs violating each reliability constraint is annotated within the corresponding shaded region.}
    \label{fig:ch5_wireless_pt}
\end{figure}

\section{Reliability Graph-Based Pareto Testing}
\label{sec:ch5_rg}

As explained in the previous section, PT orders hyperparameters
based on held-out data. This ordering may be imprecise
when such data are limited, and it fails to capture richer structural
relationships that may be known to exist among hyperparameters.

In fact, in many practical settings, structural knowledge about
relative reliability is available before any calibration data are observed.
For example, longer, or more detailed, prompt templates in LLMs are generally
expected to be more reliable; less sparse model configurations are typically
more expressive; and hyperparameter configurations with larger regularization
coefficients tend to be more conservative. Such relationships define
a partial ordering over the hyperparameter space in terms of expected reliability that can be exploited to
improve the efficiency of the testing procedure.

\emph{Reliability Graph-Based Pareto Testing} (RG-PT)
\citep{farzaneh2025multiobjective} extends PT by encoding this structural
knowledge in a directed acyclic graph (DAG), referred to as the \emph{reliability
graph} (RG). In an RG, each node represents a
candidate hyperparameter $\lambda_i$ and a directed edge from node $\lambda_i$ to node $\lambda_j$
encodes the expectation that hyperparameter $\lambda_i$ is more reliable than hyperparameter $\lambda_j$.
The acyclicity of the graph is a natural requirement: reliability relationships
cannot be circular. Moving from root nodes towards leaves, one encounters
hyperparameters of decreasing expected reliability. All nodes at the same depth
level of the DAG are considered to have approximately the same expected
reliability. An example is shown in Fig. \ref{fig:ch5_reliability_graph}, and is detailed in Sec. \ref{sec:ch5_app_prompt}. An outline of the algorithm can be found in Fig. \ref{fig:ch5_rgpt_steps}.

\begin{figure}[t]
    \centering
    \includegraphics[width=0.85\linewidth]{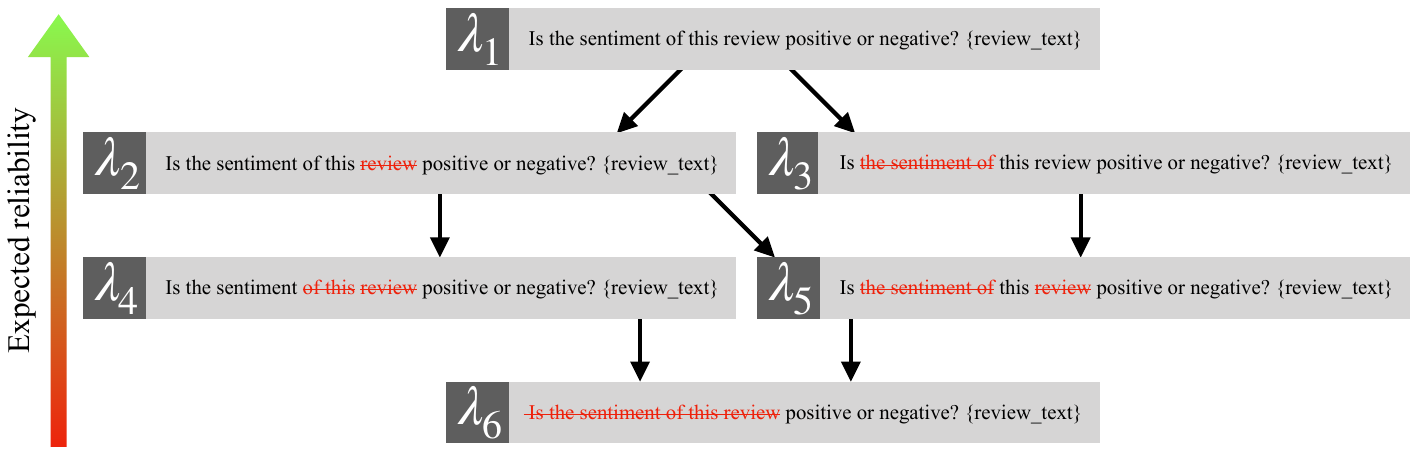}
    \caption{Illustration of a reliability graph (RG) for prompt template
    selection in an LLM-based sentiment analysis task
    \citep{farzaneh2025multiobjective}. Each node $\lambda_i$ represents a
    candidate prompt template, and a directed edge from $\lambda_i$ to
    $\lambda_j$ encodes the prior expectation that $\lambda_i$ is more
    reliable than $\lambda_j$, i.e., more likely to elicit a correct response
    from the LLM. Prompt templates closer to the root nodes are longer and
    more detailed, and are thus expected to be more reliable. Templates closer
    to the leaves are shorter and cheaper in terms of token cost, but are
    expected to be less reliable.}
    \label{fig:ch5_reliability_graph}
\end{figure}

As detailed in \citep{farzaneh2025multiobjective}, the RG is constructed from the Pareto-optimal subset
$\hat{\Lambda}^{\mathrm{Par}}$ using the held-out data split
$\mathcal{D}^{\mathrm{opt}}$ and optional prior information. Prior information is expressed in terms of preferences between pairs of hyperparameters. Depth levels in the RG
are assigned by combining data-driven pairwise p-value comparisons
with the prior pairwise preferences into a scalar reliability score per
hyperparameter via the Bradley-Terry (BT) model \citep{hunter2004mm}. Directed edges between adjacent depth levels are then learned
via non-negative Lasso regression \citep{tibshirani1996regression}, selecting
as parents of each node the hyperparameters at the previous depth level that
are most predictive of its reliability.

As shown in Fig. \ref{fig:ch5_rgpt_steps}, once the RG is constructed, RG-PT performs
MHT via the DAGGER algorithm \citep{ramdas2017dagger}, which operates on DAGs.
DAGGER begins testing at the root nodes and proceeds level by level. If a
hyperparameter is deemed unreliable, all of its descendants are automatically
excluded from further testing. This allows multiple promising paths to be
explored in parallel, unlike the single linear
sequence followed by PT. DAGGER can control the FDR, as defined in Sec. \ref{sec:FDR}, at any given target level $\delta$.

The procedure is summarized in Algorithm~\ref{alg:ch5_rgpt}
and illustrated in Fig.~\ref{fig:ch5_rgpt_steps}. The following result
states its formal guarantee.

\begin{algorithm}[t]
\caption{Reliability Graph-Based Pareto Testing (RG-PT)
\citep{farzaneh2025multiobjective}}
\label{alg:ch5_rgpt}
\begin{algorithmic}
\STATE \textbf{Input:} Candidate set $\Lambda$; calibration data
$\mathcal{D}^{\mathrm{cal}} = \mathcal{D}^{\mathrm{opt}} \cup
\mathcal{D}^{\mathrm{MHT}}$; risk thresholds $\alpha$;
FDR level $\delta$; number of DAG levels $D$; prior parameters
$(n_p, \{\eta_{ij}\})$
\STATE \textbf{Output:} Certified set $\hat{\Lambda} \subseteq \Lambda$
\medskip
\STATE \textit{Stage 1:} Compute $\hat{\Lambda}^{\mathrm{Par}} \subseteq
\Lambda$ via multi-objective optimization using $\mathcal{D}^{\mathrm{opt}}$
\medskip
\STATE \textit{Stage 2:} Learn the RG over set $\hat{\Lambda}^{\mathrm{Par}}$
using dataset $\mathcal{D}^{\mathrm{opt}}$ and prior preferences via
the Bradley-Terry model and non-negative Lasso
\medskip
\STATE \textit{Stage 3:} Run DAGGER on the RG using dataset $\mathcal{D}^{\mathrm{MHT}}$
with combined p-values (\ref{eq:ch5_combined_pval}) to obtain
$\hat{\Lambda}$
\end{algorithmic}
\end{algorithm}

\begin{figure}[t]
    \centering
    \includegraphics[width=\linewidth]{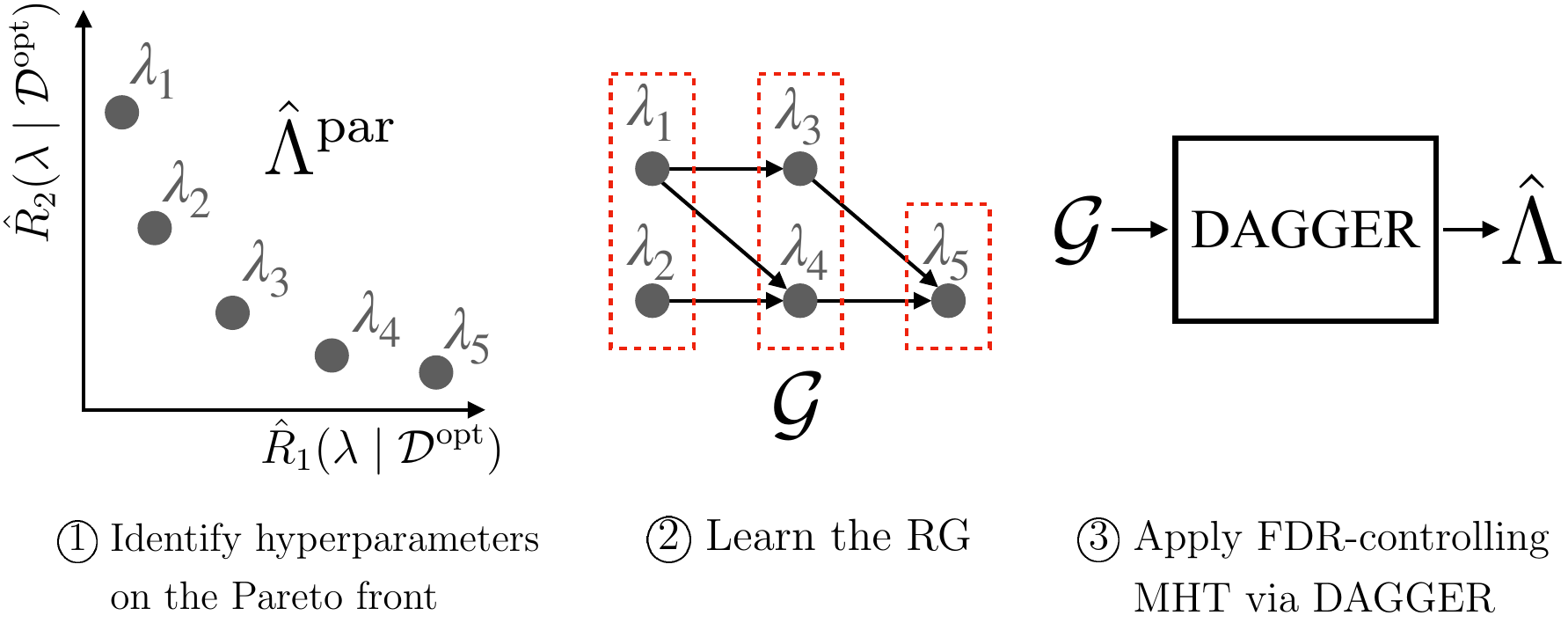}
    \caption{Illustration of the three stages of RG-PT.
    \textcircled{1} The Pareto-optimal subset $\hat{\Lambda}^{\mathrm{Par}}$
    is identified from the full candidate set $\Lambda$ using
    dataset $\mathcal{D}^{\mathrm{opt}}$, retaining only the hyperparameters that
    represent the best achievable trade-offs between the empirical risk
    estimates $\widehat{R}_1(\lambda \mid \mathcal{D}^{\mathrm{opt}})$ and
    $\widehat{R}_2(\lambda \mid \mathcal{D}^{\mathrm{opt}})$.
    \textcircled{2} The reliability graph $\mathcal{G}$ is learned over set
    $\hat{\Lambda}^{\mathrm{Par}}$ using dataset $\mathcal{D}^{\mathrm{opt}}$ and
    optional prior information. Nodes correspond to candidate hyperparameters,
    and directed edges encode inferred reliability dependencies (see Fig. \ref{fig:ch5_reliability_graph}).
    Hyperparameters at the same depth level, indicated by dashed red boxes,
    are considered to have approximately the same expected reliability, with
    reliability decreasing from left to right.
    \textcircled{3} The DAGGER algorithm \citep{ramdas2017dagger} is applied to the reliability graph
    $\mathcal{G}$ using dataset $\mathcal{D}^{\mathrm{MHT}}$, returning the certified
    set $\hat{\Lambda}$ of hyperparameters controlling the FDR \eqref{eq:ch2_fdr_def} at the target level $\delta$.}
    \label{fig:ch5_rgpt_steps}
\end{figure}

\begin{theorem}[FDR control of RG-PT \citep{farzaneh2025multiobjective}]
\label{thm:ch5_rgpt_fdr}
The set $\hat{\Lambda}$ returned by RG-PT meets the FDR constraint
\begin{equation}
\label{eq:ch5_rgpt_fdr}
\mathbb{E}\!\left[
\frac{|\hat{\Lambda} \cap \Lambda_0|}{\max\{1, |\hat{\Lambda}|\}}
\right] \le \delta,
\end{equation}
where $\Lambda_0 = \{\lambda \in \Lambda : \exists\, l \text{ s.t. }
R_l(\lambda) > \alpha_l\}$ is the set of unreliable hyperparameters. This result holds regardless of the accuracy of the prior
information or the quality of the learned RG structure.
\end{theorem}

The proof follows directly from the FDR-controlling
property of DAGGER \citep{ramdas2017dagger}.

\subsection{Example: Reliable Prompt Engineering via RG-PT}
\label{sec:ch5_app_prompt}

We illustrate RG-PT on a prompt engineering problem for
LLM-based sentiment analysis, following \citep{farzaneh2025multiobjective}.
Each data point $Z = (X, Y)$ consists of a movie review $X$ and a sentiment
label $Y \in \{\text{positive}, \text{negative}\}$ from the Stanford
Sentiment Treebank \citep{socher2013recursive}. A prompt template $\lambda$
is prepended to the review and passed to the LLaMA3-8B-Instruct model
\citep{grattafiori2024llama}. The reliability constraint requires the
average 0-1 classification loss to satisfy the inequality $R_{\mathrm{prompt}}(\lambda)
\le \alpha = 0.2$, while the auxiliary objective is to minimize prompt
length, reflecting inference cost under pay-per-token billing. The candidate
set $\Lambda$ contains 100 instruction-style templates of varying length, which are
generated by LLaMA3-70B-Instruct \citep{grattafiori2024llama} using the
forward generation procedure of \citep{zhou2022large}. The calibration
dataset contains $|\mathcal{D}^{\mathrm{opt}}| = |\mathcal{D}^{\mathrm{MHT}}|
= 1000$ examples each, and an independent test set of 1000 examples is used
for evaluation. The target FDR level is $\delta = 0.1$.

Prior information is incorporated via the LLM-as-a-judge
framework \citep{gu2024survey}: GPT-4 Turbo is queried to assess which of
each pair of prompt templates is more likely to elicit a correct response,
yielding pairwise preferences for the Bradley-Terry model. Within
the certified set $\hat{\Lambda}$, the shortest prompt is selected for
deployment.

Fig.~\ref{fig:ch5_prompt_length} reports the distribution
of the length of the shortest certified prompt over 100 independent calibration
splits. RG-PT consistently identifies shorter reliable prompts than both LTT
and PT, while all three methods satisfy the FDR constraint at the target level
$\delta = 0.1$. This improvement stems from RG-PT's more efficient exploration
of the hyperparameter space along reliability-informed paths encoded in the RG:
by testing hyperparameters in an order informed by the prior reliability graph,
RG-PT can certify shorter prompts that LTT and PT, which either ignore structure
entirely or use only a linear ordering, fail to reach within the same testing
budget.

\begin{figure}[t]
    \centering
    \includegraphics[width=0.5\linewidth]{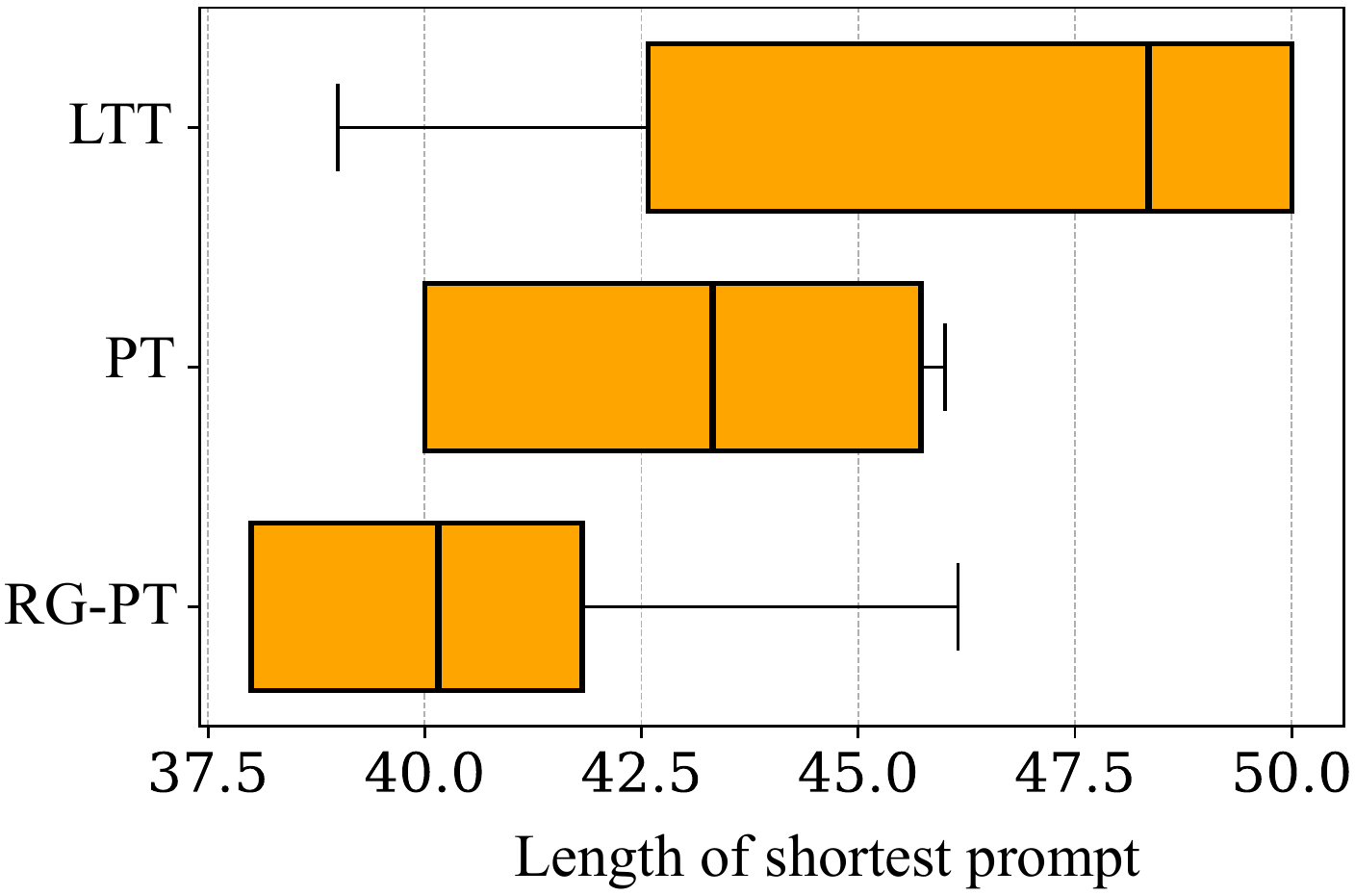}
    \caption{Distribution of the length of the shortest certified prompt
    template returned by LTT, PT, and RG-PT over 100 independent calibration
    splits for the LLM-based sentiment analysis task on the Stanford Sentiment
    Treebank dataset \citep{socher2013recursive}. All methods satisfy the FDR
    constraint at level $\delta = 0.1$. RG-PT identifies shorter reliable
    prompts by leveraging the prior reliability structure encoded in the
    reliability graph, constructed using the LLM-as-a-judge framework
    \citep{gu2024survey}.}
    \label{fig:ch5_prompt_length}
\end{figure}

\section{Summary}
\label{sec:ch5_comparison}

This chapter extended the LTT framework to multi-objective hyperparameter selection, where the practitioner must simultaneously satisfy $L_C$ reliability constraints and optimize $L - L_C$ auxiliary objectives.

The three methods presented in this chapter---LTT, PT, and RG-PT--- differ along several dimensions, as summarized next and in Table \ref{tab:ch5_comparison}.

\paragraph{Testing scope:} LTT applies an MHT procedure such as Bonferroni or BH to all $K = |\Lambda|$ candidate hyperparameters, without any ordering or restriction. PT restricts testing to the empirical Pareto frontier $\hat{\Lambda}^{\mathrm{Par}} \subseteq \Lambda$ and applies FST along a single linear ordering. RG-PT further structures the testing over a DAG, enabling parallel exploration of multiple reliability paths.

\paragraph{Use of prior information:} LTT and PT do not use prior structural information about relative reliability. RG-PT explicitly incorporates such information via prior pairwise preferences in the Bradley-Terry model, while the validity guarantee is maintained regardless of the quality of these priors.

\paragraph{Error control:} All three methods provide formal error guarantees. LTT and PT, as described in \citep{laufer-goldshtein2023efficiently}, can be configured to control either FWER or FDR. RG-PT controls the FDR at the specified level $\delta$. FDR control is generally less conservative than FWER control, enabling larger certified sets at the cost of allowing a small expected fraction of unreliable configurations in $\hat{\Lambda}$.

\paragraph{Multi-objective optimization:} LTT can accommodate
multiple reliability constraints via the combined p-value (\ref{eq:ch5_combined_pval}),
but does not exploit the multi-objective structure of the problem. In particular, it
tests all $K$ candidates without regard for dominance relationships among them, and
does not actively optimize any auxiliary objectives beyond what a secondary criterion
applied to the certified set $\hat{\Lambda}$ can achieve. PT and RG-PT address both
of these limitations: by restricting testing to the Pareto frontier, they eliminate
dominated configurations before testing begins, and by ordering or structuring the
remaining tests, they recover certified configurations that are also efficient with
respect to the auxiliary objectives.

\begin{table}[t]
\centering
\renewcommand{\arraystretch}{1.3}
\resizebox{\columnwidth}{!}{%
\begin{tabular}{lllll}
\hline
\textbf{Property} & \textbf{LTT} & \textbf{PT} & \textbf{RG-PT} \\
\hline
Multi-objective support & No$^*$ & Yes & Yes \\
Prior structure & No & Linear order & DAG \\
Testing scope & Full $\Lambda$ & Pareto frontier & Pareto frontier \\
Error criterion & FWER or FDR & FWER or FDR & FDR \\
\hline
\end{tabular}%
}
\caption{Comparison of LTT, PT, and RG-PT along key dimensions. ($^*$LTT can
accommodate multiple reliability constraints via the combined p-value
(\ref{eq:ch5_combined_pval}), but does not exploit multi-objective structure
to improve efficiency or optimize auxiliary objectives.)}
\label{tab:ch5_comparison}
\end{table}


\chapter{Adaptive and Sequential Hyperparameter Selection}
\label{chapter:adaptive}

The LTT framework introduced in Chapter~\ref{chapter:mht}
and its multi-objective extensions in Chapter~\ref{chapter:multi_objective} share a common
pipeline: (\textit{i}) the calibration data are collected in advance; (\textit{ii}) all candidate
hyperparameters are evaluated on the same fixed dataset $\mathcal{D}^\mathrm{cal}$; and (\textit{iii}) the testing procedure is
applied once to produce a certified set $\hat{\Lambda}$ of hyperparameters.
This batch-based operation is well-suited to settings where data are
plentiful and testing is cheap, but it becomes costly or impractical when data acquisition is expensive or the number of candidates is large.

Consider, for example, the problem of selecting a control policy
for a robotic system via offline reinforcement learning \citep{levine2020offline}.
Policies trained on offline data must be validated through online interaction with the
real environment before deployment \citep{paine2020hyperparameter}. Each such
interaction episode consumes physical resources, may cause wear on hardware, and in
safety-critical settings such as autonomous driving or robotic surgery can pose direct
risk of harm \citep{garcia2015comprehensive}. Reducing the number of episodes used for policy selection is thus of critical importance. Instead of evaluating all $K$ candidate policies on
a fixed budget of episodes, as LTT would require, it would be useful to evaluate policies adaptively, allocating evaluation episodes only to the most promising policies.

As another example, consider automated prompt engineering for an LLM deployed
under pay-per-token billing \citep{zhou2022large}. A set of $K$ candidate prompt
templates is generated, and each template must be evaluated on held-out queries by
executing the LLM, with each execution incurring a direct and measurable monetary cost
\citep{bommasani2021opportunities}. With $K$ ranging into the hundreds in realistic
pipelines \citep{zhou2022large}, and with each evaluation potentially consuming
thousands of tokens, the total cost of exhaustively testing all candidates on a fixed
calibration dataset can be prohibitive. It is far more economical to concentrate
evaluations on the templates that have shown the most promise in early rounds, and to
stop as soon as a sufficiently reliable subset has been identified.

The principle of adaptive evaluation is already implemented, in different forms, by several existing HPO methods. Successive halving \citep{jamieson2016non} and Hyperband \citep{li2018hyperband} allocate computational resources adaptively across configurations, terminating poorly performing candidates early and concentrating the training budget on the most promising ones. Multi-fidelity Bayesian optimization extends this idea by leveraging a surrogate model of the validation performance to decide where to evaluate next \citep{shahriari2016taking}, while \emph{freeze-thaw} Bayesian optimization explicitly models partial learning curves to decide which training runs to pause, resume, or discard \citep{franceschi2025hpo_fnt}.

The guarantees offered by these methods, when available, concern the efficient identification of the empirical optimum, rather than the reliability of the deployed configuration. For example, successive halving and Hyperband come with best-arm-identification and simple-regret bounds under suitable assumptions on the loss curves, ensuring that, with enough budget, a near-empirically-optimal configuration is returned \citep{jamieson2016non,li2018hyperband}. Freeze-thaw and related learning-curve-based schemes are largely heuristic, and their convergence properties typically rest on modelling assumptions about the learning curves rather than on finite-sample concentration arguments \citep{franceschi2025hpo_fnt}. Crucially, however, as in conventional HPO discussed in Chapter~\ref{chapter:intro}, none of these adaptive methods provides finite-sample statistical guarantees that the selected hyperparameter satisfies a user-specified reliability requirement of the form (\ref{eq:ch1_guarantee}).

This chapter introduces an adaptive extension of LTT, called
\emph{adaptive LTT} (aLTT), introduced in \citep{zecchin2024adaptive}. aLTT replaces
the static, batch hypothesis testing of LTT with a sequential, data-dependent procedure
that (\textit{i}) selects which hyperparameters to test next based on evidence accumulated in
prior rounds, and (\textit{ii}) terminates early as soon as a target number of reliable
configurations have been certified. The statistical validity of the procedure, in terms
of FWER or FDR control, is maintained throughout due to e-process-based testing, which supports anytime-valid
inference.

The chapter is organized as follows. Section~\ref{sec:ch6_sequential_setting} motivates and formalizes the
sequential calibration setting.
Sec.~\ref{sec:ch6_eprocess} introduces e-processes and anytime-valid p-values, the
central statistical tools underlying aLTT.
Sec.~\ref{sec:ch6_altt} presents the aLTT algorithm.
Sec.~\ref{sec:ch6_app_wireless} illustrates aLTT on the wireless scheduling
application of Sec.~\ref{sec:app_wireless}, adapted to the sequential calibration setting.

\section{Sequential Hyperparameter Selection}
\label{sec:ch6_sequential_setting}

As illustrated in Fig. \ref{fig:ch6_motivation}(a), under the LTT framework reviewed in Chapter~\ref{chapter:mht}, all
$K = |\Lambda|$ candidate hyperparameters are fully tested using the entire
calibration dataset $\mathcal{D}^{\mathrm{cal}}$.
The MHT procedure is then applied once to the resulting p-values to produce a certified
set $\hat{\Lambda}$.

As discussed, this batch design may be excessively costly in scenarios where one wishes to reduce data usage, as well as the number of evaluation rounds. To address such settings, as shown in Fig. \ref{fig:ch6_motivation}(b), aLTT \citep{zecchin2024adaptive} operates in a feedback loop: at each
round $t$, it evaluates the evidence accumulated so far and asks whether that
evidence is sufficient to produce a certified set of the desired size. If not,
it requests a new batch of calibration data, updates the evidence, and
repeats. When sufficient evidence has been gathered, the loop exits and
$\hat{\Lambda}$ is returned.

\begin{figure}[t]
    \centering
    \includegraphics[width=0.75\linewidth]{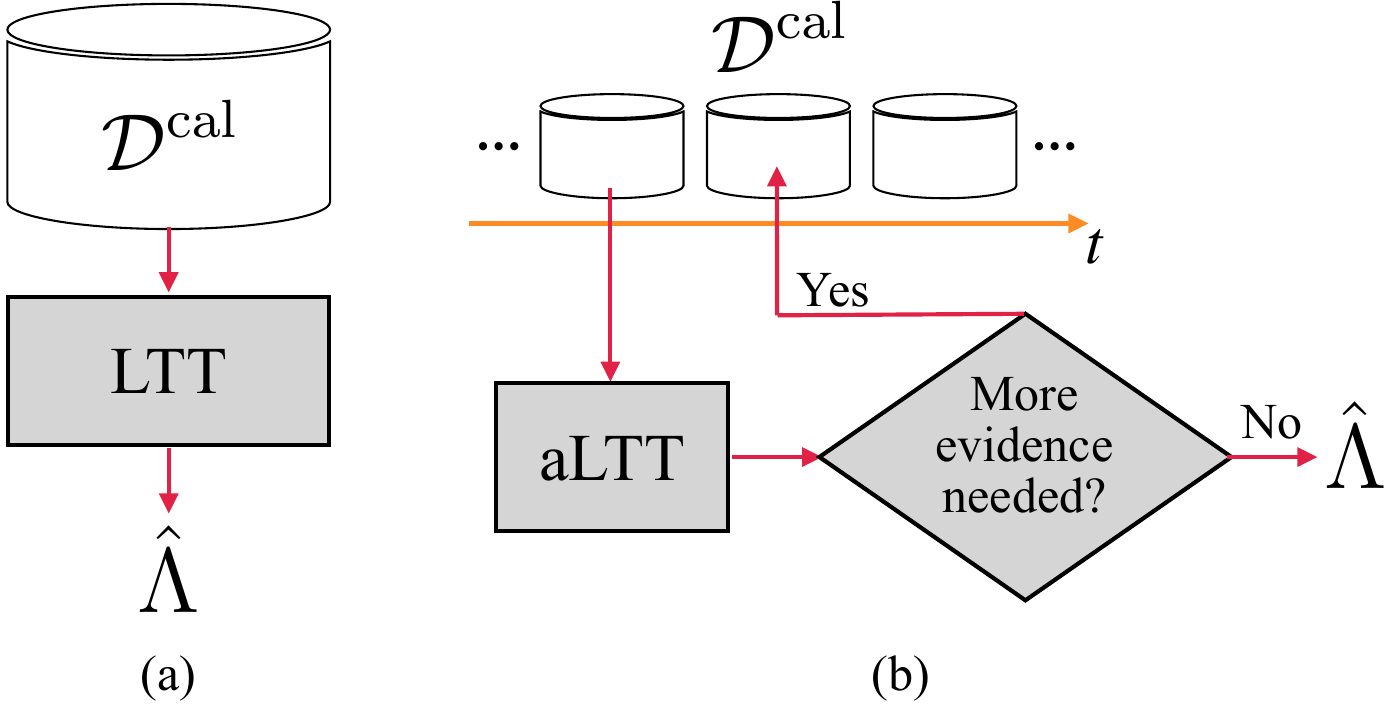}
    \caption{Conceptual comparison of LTT and aLTT.
    (a) LTT is a single-pass batch procedure: a fixed calibration dataset
    $\mathcal{D}^{\mathrm{cal}}$ is collected once, processed by the LTT
    procedure, and a certified set $\hat{\Lambda}$ is returned.
    (b) aLTT is a sequential procedure in which calibration data arrive
    incrementally over rounds $t = 1, 2, \ldots$ At each round, aLTT
    evaluates the accumulated evidence and decides whether more data are
    needed. If so, a new data batch is requested and the available evidence is updated.
    The procedure terminates and returns $\hat{\Lambda}$ as soon as sufficient
    evidence has been gathered.}
    \label{fig:ch6_motivation}
\end{figure}

To formalize, as throughout this monograph, let
$\Lambda = \{\lambda_1, \ldots, \lambda_K\}$
be a finite candidate set.  For each hyperparameter $\lambda \in \Lambda$, the risk
\begin{equation}
R(\lambda) = \mathbb{E}[L_\lambda(Z)]
\end{equation}
is defined with respect to an unknown data-generating distribution $\mathbb{P}$,
where $Z$ is a generic test-time data point. As in
Definition~\ref{def:ch2_rcp}, the reliability threshold $\alpha$ and
failure probability $\delta$ are specified by the user.

Sequential hyperparameter selection proceeds in discrete rounds
$t = 1, 2, \ldots$ In each round $t$, the procedure executes the following five steps.

\textbf{\textcircled{1} Hyperparameter selection:} An \emph{acquisition
policy} $\mathcal{Q}_t$ selects a subset $\mathcal{I}^t \subseteq \Lambda$ of candidates to be
tested in round $t$. The policy $\mathcal{Q}_t$ may be \emph{adaptive}, meaning that it may
depend on the evidence $\mathcal{D}^{t-1}$ collected in all previous rounds, or
\emph{non-adaptive}, meaning that the set $\mathcal{I}^t$ is fixed in advance.

\textbf{\textcircled{2} Testing:} A batch of data samples $\mathcal{Z}^t \sim \mathbb{P}$ are drawn independently across rounds, and the empirical loss
$L_{\lambda_i}(\mathcal{Z}^t)$ is evaluated.

\textbf{\textcircled{3} Evidence update:} The evidence accumulated for
each candidate hyperparameter is updated using the newly observed empirical losses.
\begin{equation}
\label{eq:evidence_t}
\mathcal{D}^t = \mathcal{D}^{t-1} \cup
\bigl\{(\mathcal{I}^t,\, \{L_{\lambda_i}(Z_i^t)\}_{\lambda_i \in \mathcal{I}^t})\bigr\},
\end{equation}

\textbf{\textcircled{4} Certification:} A decision rule $\mathcal{A}$
is applied to the current evidence to produce a candidate certified set $\hat{\Lambda}^t$.
Depending on the desired error criterion, the decision rule $\mathcal{A}$ can be either FWER-controlling or FDR-controlling.

\textbf{\textcircled{5} Stopping check:} If $|\hat{\Lambda}^t| \ge d$
for a user-specified minimum certified set size $d \ge 1$, or if the maximum number of
rounds $t_{\max}$ has been reached, the procedure terminates and returns
$\hat{\Lambda} = \hat{\Lambda}^t$. Otherwise, time index $t$ is incremented and the procedure
returns to \textcircled{1}.

The entire sequential procedure is thus specified by the tuple
$(\{\mathcal{Q}_t\}_{t \geq 1},\allowbreak\, \mathcal{A},\allowbreak\, T)$, where $\{\mathcal{Q}_t\}_{t \geq 1}$ is the
acquisition policy, $\mathcal{A}$ is the decision rule that maps accumulated evidence
to a certified set, and $T$ is a stopping time that may itself depend on the evidence
accumulated so far (adaptive stopping) or be fixed in advance (non-adaptive stopping).

As described in the remainder of this chapter, aLTT implements adaptive HPO by leveraging the statistical tool of e-processes (see Fig. \ref{fig:ch6_sequential_loop}).

\begin{figure}[t]
    \centering
    \includegraphics[width=\linewidth]{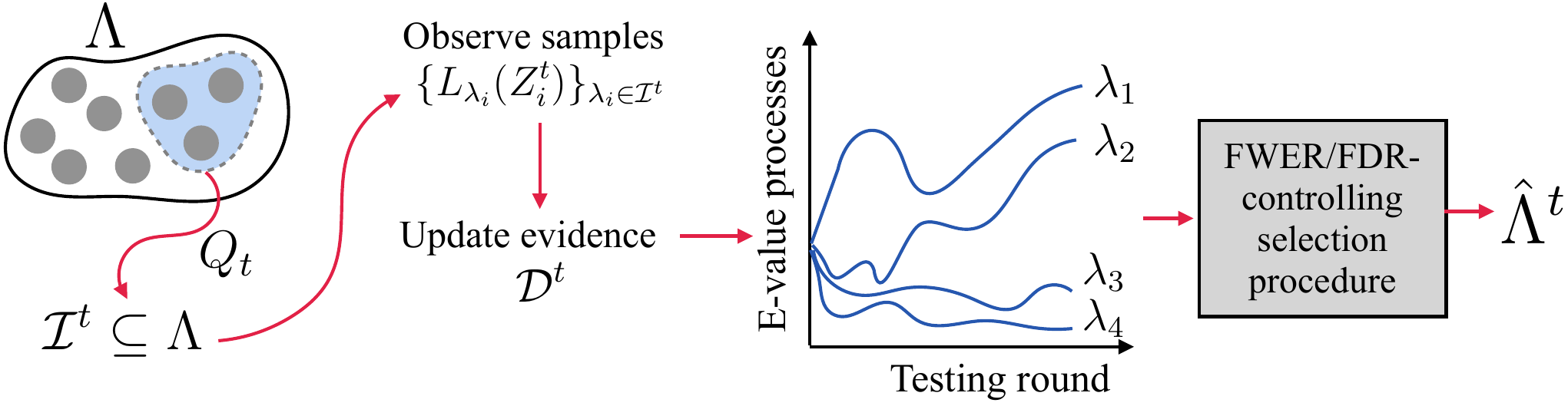}
    \caption{Illustration of the operations performed within a single round $t$
    of the aLTT calibration procedure. The acquisition policy $\mathcal{Q}_t$ selects a
    subset $\mathcal{I}^t \subseteq \Lambda$ of candidates to test
    (\textcircled{1}). Fresh losses $\{L_{\lambda_i}(Z_i^t)\}_{\lambda_i \in
    \mathcal{I}^t}$ are observed (\textcircled{2}) and used to update the
    evidence $\mathcal{D}^t$ (\textcircled{3}), which in practice amounts to
    updating the e-process $E_i^t$ for each tested candidate, as shown by the
    curves. Candidates with growing e-processes (e.g., $\lambda_1, \lambda_2$)
    accumulate evidence of reliability, while those with flat or decaying
    e-processes (e.g., $\lambda_3, \lambda_4$) do not. The current e-process
    values are then passed to an FWER- or FDR-controlling selection procedure
    to produce the updated certified set $\hat{\Lambda}^t$ (\textcircled{4}).
    At the end of each round, a stopping check (\textcircled{5}) determines
    whether calibration should continue or terminate.}
    \label{fig:ch6_sequential_loop}
\end{figure}

\section{E-Processes and Anytime-Valid Inference}
\label{sec:ch6_eprocess}

The central obstacle to applying LTT in a sequential,
data-dependent manner is the invalidity of p-values under adaptive stopping.
In fact, the type-I error guarantee relies on the superuniform property \eqref{eq:ch2_pvalue_validity}, which is no longer guaranteed if one conditions on observations made in past testing rounds. This violation of statistical validity is known as p-hacking, and is a key contributor to the reproducibility crisis in science
\citep{simonsohn2014p}. A more detailed discussion of this limitation of p-values, including a worked illustration of the optional-stopping pathology, is provided in Appendix \ref{app:evidence}.

E-values, introduced in Sec.~\ref{sec:ch2_eval}, are special types of test statistics that support anytime-valid testing: e-values computed at different time points can be combined by multiplication without
losing validity. This makes them the natural foundation for sequential, adaptive testing \citep{ramdas2025hypothesis, wang2022fdr}.

\subsection{E-Processes}

An \emph{e-process} for null hypothesis $\mathcal{H}_i:
R(\lambda_i) > \alpha$ is a sequence of non-negative random variables
$\{E_i^t\}_{t \geq 0}$ such that, for any stopping time $T$, which may depend
on all evidence accumulated up to and including round $T$, the stopped value $E_i^T$ is
a valid e-value, i.e.,
\begin{equation}
\label{eq:ch6_eprocess_def}
\mathbb{E}[E_i^T \mid \mathcal{H}_i] \le 1.
\end{equation}
Condition~\eqref{eq:ch6_eprocess_def} is the sequential analog of the e-value
validity condition (\ref{eq:ch2_evalue_validity}), and it holds for all stopping
times simultaneously, not just for a fixed time horizon. A more general treatment of e-processes, including Ville's inequality and the corresponding anytime-valid testing guarantees, is given in Appendix \ref{app:evidence}.

A common and powerful construction of an e-process for the null
$\mathcal{H}_i:R(\lambda_i) > \alpha$ follows the principle of testing by betting \citep{ramdas2023gametheoretic}. Accordingly, at each round $t$ in
which hyperparameter $\lambda_i$ is tested, we observe the loss $L_{\lambda_i}(Z_i^t)$, and define a
per-round e-value
\begin{equation}
\label{eq:ch6_per_round_eval}
\epsilon_i^t
\;=\;
1 + \mu_i^t\!\left(\alpha - L_{\lambda_i}(Z_i^t)\right),
\end{equation}
where
\begin{equation}
\label{eq:ch6_condition}
\mu_i^t \in \left(0,\, \tfrac{1}{1-\alpha}\right)
\end{equation}
is a tunable \emph{betting
parameter} \citep{waudby2024estimating}.
In the testing by betting interpretation, the term \eqref{eq:ch6_per_round_eval} represents the multiplicative wealth growth for a bettor investing an amount $\mu_i^t$ against the null hypothesis $\mathcal{H}_i$. By \eqref{eq:ch6_per_round_eval}, the wealth grows by an amount that increases whenever the inequality $L_{\lambda_i}(Z_i)<\alpha$ holds. Recalling that the empirical estimate is unbiased, i.e., $\mathbb{E}[L_{\lambda_i}(Z_i)] = R(\lambda_i)$, this condition provides evidence against the null $\mathcal{H}_i:R(\lambda_i) > \alpha$.

Setting $E_i^0 = 1$, the e-process is then given by the running product of the multiplicative wealth gains \eqref{eq:ch6_per_round_eval} as
\begin{equation}
\label{eq:ch6_eprocess_update}
E_i^t
\;=\;
\begin{cases}
E_i^{t-1} \cdot \epsilon_i^t & \text{if } \lambda_i \in \mathcal{I}^t, \\
E_i^{t-1} & \text{otherwise.}
\end{cases}
\end{equation}
That is, the e-process is updated by multiplying in the per-round factor $\epsilon_i^t$
whenever hyperparameter $\lambda_i$ is tested, and is held constant otherwise. In the testing by betting interpretation, $E_i^t$ represents the current wealth of the bettor for hyperparameter $\lambda_i$.

\begin{lemma}[Validity of the sequential product e-process]
\label{lem:ch6_eprocess_valid}
The sequence $\{E_i^t\}$ in \eqref{eq:ch6_eprocess_update} is an e-process for the null hypothesis $\mathcal{H}_i$.
\end{lemma}
\begin{proof}
By the unbiasedness of the empirical estimate, we have the inequality $\mathbb{E}[L_{\lambda_i}(Z^t)] = R(\lambda_i) \ge \alpha$ under the null $\mathcal{H}_i$,
and hence the average multiplicative gain in \eqref{eq:ch6_per_round_eval} satisfies the condition
\begin{equation}
\mathbb{E}[\epsilon_i^t \mid \mathcal{D}^{t-1}, \mathcal{H}_i]
= 1 + \mu_i^t\!\left(\alpha - R(\lambda_i)\right)
\;\le\; 1.
\end{equation}
Given the condition in \eqref{eq:ch6_condition}, i.e., $\mu_i^t < 1/(1-\alpha)$, as well as the assumption $L_{\lambda_i}(Z^t) \in [0,1]$, each factor
$\epsilon_i^t$ is non-negative. Together with the bound above, this shows that
the product $E_i^t$ is non-negative, and that it satisfies the inequality $\mathbb{E}[E_i^t\mid \mathcal{H}_i]\leq 1$. This shows that the process $\{E_i^t\}$ is a non-negative supermartingale with respect to the filtration generated
by $\mathcal{D}^t$ under $\mathcal{H}_i$. By Doob's optional stopping theorem for
non-negative supermartingales \citep{ramdas2023gametheoretic}, the expectation at
any stopping time $T$ satisfies
\begin{equation}
\mathbb{E}[E_i^T \mid \mathcal{H}_i] \;\le\; E_i^0 \;=\; 1,
\end{equation}
confirming that the sequence $E_i^T$ is an e-process as per condition~\eqref{eq:ch6_eprocess_def}.
\end{proof}

\subsection{Betting Strategies}

The power of the e-process $\{E_i^t\}_{t = 1}^T$ to detect reliable hyperparameters depends on how quickly it grows
when hyperparameter $\lambda_i$ is truly reliable. The rate of this growth, in turn, depends critically on the choice of the betting
parameters $\{\mu_i^t\}_{t \geq 1}$.

If the betting parameter $\mu_i^t$ is chosen to be large, the gambler is aggressive: the
wealth grows rapidly when the hyperparameter is reliable, but can also decrease rapidly
when individual losses exceed $\alpha$. Conversely, if $\mu_i^t$ is small, the procedure is
conservative, and growth is slow but stable. The optimal choice for the sequence of betting parameters $\mu_i^t$ is typically considered to be one that
maximizes the expected log-wealth growth, i.e., $\mathbb{E}[\log \epsilon_i^t]$, subject to
the constraint $\mu_i^t \in (0, 1/(1-\alpha))$ \citep{shafer2019gamefoundation,ramdas2023gametheoretic,waudby2024estimating}.

Accordingly, a practical and principled choice is the
\emph{adaptive growth rate adaptive to the particular alternative} (aGRAPA) strategy
\citep{waudby2024estimating}, which sets the parameter $\mu_i^t$ based on the empirical mean of the
losses observed up to round $t-1$. More precisely, aGRAPA estimates the true risk
$R(\lambda_i)$ from past data and sets the parameter $\mu_i^t$ to be the solution of the per-round empirical wealth optimization problem. This adaptive approach requires no prior knowledge of
$R(\lambda_i)$ and performs well across a wide range of true risk values.

Another option is the \emph{online Newton step} (ONS) betting
strategy \citep{waudby2024estimating}, which updates $\mu_i^t$ using a second-order
online learning rule applied to the log-wealth objective. ONS typically achieves faster
growth rates in practice.

For simplicity of exposition, we set $\mu_i^t = \mu$ as a fixed
constant in the theoretical development below, while noting that data-adaptive
strategies are preferable in practice.

\subsection{Anytime-Valid P-values}

Given an e-process $\{E_i^t\}_{t \geq 0}$, one can derive an
\emph{anytime-valid p-value} $P_i^t$ for hypothesis $\mathcal{H}_i$ as follows. Define
\begin{equation}
\label{eq:ch6_anytime_pval}
P_i^t
\;=\;
\frac{1}{\max_{\tau \le t} E_i^\tau}.
\end{equation}
That is, the anytime-valid p-value at time $t$ is the reciprocal of the
running maximum of the e-process up to round $t$.

\begin{lemma}[Anytime validity of the p-value \eqref{eq:ch6_anytime_pval}]
\label{lem:ch6_anytime_pval}
For any stopping time $T$ that may depend on the entire history
$\mathcal{D}^T$, the statistic $P_i^T$ in (\ref{eq:ch6_anytime_pval}) is a valid
p-value for $\mathcal{H}_i$.
\end{lemma}
\begin{proof}
For any $u \in [0,1]$, note that the inequality $P_i^T \leq u$ holds if and only if
$\max_{\tau \leq T} E_i^\tau \geq 1/u$, which implies $E_i^T \geq 1/u$ for
some $\tau \leq T$. By Ville's inequality for non-negative martingales
\citep{ville1939etude}, we have
\begin{equation}
\mathbb{P}\!\left(
\max_{\tau \leq T} E_i^\tau \geq \frac{1}{u}
\;\middle|\; \mathcal{H}_i
\right)
\;\leq\;
u\, \mathbb{E}[E_i^0 \mid \mathcal{H}_i]
\;=\;
u,
\end{equation}
and thus the superuniformity condition $\mathbb{P}(P_i^T \leq u \mid \mathcal{H}_i) \leq u$ holds, as required.
\end{proof}

\section{Adaptive Learn-Then-Test}
\label{sec:ch6_altt}

We now present the full aLTT algorithm \citep{zecchin2024adaptive},
which combines the e-process machinery described in Sec.~\ref{sec:ch6_eprocess} with an adaptive
acquisition policy and an adaptive stopping rule.

\subsection{Algorithm Description}

aLTT operationalizes the five-step sequential procedure of
Sec.~\ref{sec:ch6_sequential_setting} by maintaining, for each candidate
$\lambda_i \in \Lambda$, a running e-process $E_i^t$ initialized at $E_i^0 = 1$.
At each round $t$, the acquisition policy selects a subset of hyperparameters $\mathcal{I}^t \subseteq \Lambda$ (step \textcircled{1});
empirical losses are evaluated (step \textcircled{2}); and the e-processes are updated
multiplicatively as per \eqref{eq:ch6_eprocess_update} (step \textcircled{3}).
The current e-process values are then passed to the chosen MHT procedure to produce
the certified set $\hat{\Lambda}^t$ as (step \textcircled{4})
\begin{equation}
\hat{\Lambda}^t
\;=\;
\begin{cases}
\mathcal{A}_{\mathrm{FWER}}(P_1^t, \ldots, P_K^t) & \text{if FWER control is desired,}\\
\mathcal{A}_{\mathrm{FDR}}(E_1^t, \ldots, E_K^t) & \text{if FDR control is desired,}
\end{cases}
\end{equation}
where $P_i^t = 1/\max_{\tau \leq t} E_i^\tau$ are the anytime-valid p-values in
\eqref{eq:ch6_anytime_pval}. For FWER control, any procedure from
Sec.~\ref{sec:ch2_pval_mht} can be applied; for FDR control, the e-BH procedure
from Sec.~\ref{sec:FDR} is typically used directly on the e-values $E_i^t$.

Finally, the
stopping check (step \textcircled{5}) terminates the procedure if $|\hat{\Lambda}^t| \geq d$
or $t = t_{\max}$, and otherwise increments $t$ and returns to \textcircled{1}.
The full aLTT procedure is summarized in Algorithm~\ref{alg:ch6_altt}.

\begin{algorithm}[t]
\caption{Adaptive Learn-Then-Test (aLTT) \citep{zecchin2024adaptive}}
\label{alg:ch6_altt}
\begin{algorithmic}
\STATE \textbf{Input:} Candidate set $\Lambda = \{\lambda_1, \ldots, \lambda_K\}$;
risk threshold $\alpha$; error level $\delta$; error type
$\mathrm{err} \in \{\mathrm{FWER}, \mathrm{FDR}\}$; acquisition policy
$\{\mathcal{Q}_t\}_{t \geq 1}$; betting parameters $\{\mu_i^t\}$; maximum rounds
$t_{\max}$; minimum certified set size $d$; error-controlling MHT procedure $\mathcal{A}_{\mathrm{err}}(\cdot)$ at level $\delta$;

\STATE \textbf{Output:} Certified set $\hat{\Lambda} \subseteq \Lambda$

\STATE Initialize $E_i^0 \leftarrow 1$ for all $i$; $t \leftarrow 1$;
$\hat{\Lambda} \leftarrow \emptyset$

\WHILE{$t \leq t_{\max}$ \textbf{and} $|\hat{\Lambda}| < d$}
    \STATE $\mathcal{I}^t \leftarrow \mathcal{Q}_t(\{E_i^{t-1}\}_{i=1}^K)$

    \FOR{each $\lambda_i \in \mathcal{I}^t$}
        \STATE Observe the losses $L_{\lambda_i}(Z_i^t)$ and update the e-process
        $E_i^t \leftarrow E_i^{t-1} \cdot \bigl(1 + \mu_i^t(\alpha -
        L_{\lambda_i}(Z_i^t))\bigr)$
    \ENDFOR

    \FOR{each $\lambda_i \notin \mathcal{I}^t$}
        \STATE $E_i^t \leftarrow E_i^{t-1}$
    \ENDFOR

    \IF{$\mathrm{err} = \mathrm{FWER}$}
        \STATE Compute the p-value $P_i^t \leftarrow 1 / \max_{\tau \leq t} E_i^\tau$
        for all $i$
        \STATE $\hat{\Lambda} \leftarrow
        \mathcal{A}_{\mathrm{FWER}}(P_1^t, \ldots, P_K^t)$
    \ELSE
        \STATE $\hat{\Lambda} \leftarrow
        \mathcal{A}_{\mathrm{FDR}}(E_1^t, \ldots, E_K^t)$
    \ENDIF

    \STATE $t \leftarrow t + 1$
\ENDWHILE

\STATE \textbf{return} $\hat{\Lambda}$
\end{algorithmic}
\end{algorithm}

\subsection{Acquisition Policies}
\label{sec:ch6_acquisition}

A key design choice in aLTT is the acquisition policy $\mathcal{Q}_t$, which
determines which candidates to test at each round. Following \citep{zecchin2024adaptive}, we discuss two natural choices.

\paragraph{Non-adaptive (random) acquisition:} The simplest policy
selects one or more candidates uniformly at random from set $\Lambda$ at each round,
independently of the e-process values. This recovers the same long-run performance as
LTT when $T \to \infty$, but, crucially, produces valid intermediate certified sets at
every round $t$, enabling early termination.

\paragraph{$\varepsilon$-greedy acquisition:} A more effective
adaptive policy is the $\varepsilon$-greedy rule: With probability $1 - \varepsilon$,
select the uncertified candidate $\lambda_i \notin \hat{\Lambda}^{t-1}$ with the
largest current e-process value $E_i^{t-1}$; otherwise, with probability
$\varepsilon$, select a candidate uniformly at random from the uncertified set.
The greedy component concentrates testing effort on the candidates
that have accumulated the most evidence of reliability so far, accelerating the growth
of their e-processes toward the rejection threshold. In contrast, the random component ensures that
candidates whose e-processes have grown slowly due to limited testing are not
permanently ignored.

\subsection{Statistical Guarantees}

The fundamental validity of aLTT rests on the anytime-validity
property of the e-process (Lemma~\ref{lem:ch6_anytime_pval}): no matter at which round
$T$ the procedure terminates, whether triggered by the early stopping criterion or the
maximum round limit, the resulting certified set carries the same formal guarantees as
the batch LTT procedure.

\begin{theorem}[Validity of aLTT \citep{zecchin2024adaptive}]
\label{thm:ch6_altt_validity}
Let $T$ be any stopping time (possibly adaptive) at which aLTT terminates.
\begin{enumerate}
    \item[\normalfont(\textit{i})] If $\mathcal{A}_{\mathrm{FWER}}$ is applied, the certified
    set $\hat{\Lambda} = \hat{\Lambda}^T$ satisfies the FWER requirement
    \begin{equation}
    \label{eq:ch6_fwer_guarantee}
    \mathbb{P}\!\left(
    \sup_{\lambda \in \hat{\Lambda}} R(\lambda) \le \alpha
    \right) \;\ge\; 1 - \delta.
    \end{equation}
    \item[\normalfont(\textit{ii})] If $\mathcal{A}_{\mathrm{FDR}}$ (e-BH) is applied, the
    certified set satisfies the FDR requirement
    \begin{equation}
    \label{eq:ch6_fdr_guarantee}
    \mathbb{E}\!\left[
    \frac{|\hat{\Lambda} \cap \Lambda_0|}{\max\{1, |\hat{\Lambda}|\}}
    \right] \;\le\; \delta,
    \end{equation}
    where $\Lambda_0 = \{\lambda \in \Lambda : R(\lambda) > \alpha\}$.
\end{enumerate}
\end{theorem}

\begin{proof}
For part (\textit{i}), by Lemma~\ref{lem:ch6_anytime_pval}, each statistic $P_i^T$ is a valid p-value for
$\mathcal{H}_i$ at the (possibly adaptive) stopping time $T$. The FWER-controlling
procedure $\mathcal{A}_{\mathrm{FWER}}$ applied to these p-values therefore provides
the guarantee \eqref{eq:ch6_fwer_guarantee} by the same argument as
Theorem~\ref{thm:ch2_certified_set}.

For part (\textit{ii}), the e-values $E_i^T$ satisfy $\mathbb{E}[E_i^T \mid \mathcal{H}_i] \leq
1$ by the e-process property \eqref{eq:ch6_eprocess_def}, for any stopping time $T$.
The FDR guarantee then follows from the e-BH theorem (Appendix~\ref{app:e-BH}) applied
to the stopped e-values.
\end{proof}

\section{Example: Adaptive Wireless Scheduling}
\label{sec:ch6_app_wireless}

We illustrate the aLTT framework using the wireless scheduling
application introduced in Sec.~\ref{sec:app_wireless}, now adapted to the sequential calibration setting.
This application is a natural fit for adaptive testing: each evaluation of a candidate
hyperparameter configuration requires executing a full scheduling episode in the Nokia
wireless simulator \citep{Nokia}, and reducing the total number of such episodes is a
primary operational objective.

\subsection{Setting}

We consider the same downlink scheduling system described in
Sec.~\ref{sec:app_wireless}. The learning-based scheduler of \citep{de2020radio} is
parameterized by a vector $\lambda = (\lambda_1, \lambda_2, \lambda_3, \lambda_4)$
controlling the weights of four terms in the reward model. Following
Sec.~\ref{sec:app_wireless}, a candidate set of $K = 20$ configurations is
considered.

The per-episode loss is the mean packet delay of the
highest-priority QoS class, normalized to $[0,1]$ by clipping at $20$ ms as
\begin{equation}
    L_\lambda(Z^t)
    \;=\;
    \frac{\min\!\left(\mathrm{delay}_\lambda(Z^t),\; 20\,\mathrm{ms}\right)}{20\,\mathrm{ms}}.
\end{equation}
The reliability requirement is $R(\lambda) = \mathbb{E}[L_\lambda(Z)] \le \alpha$,
where $\alpha = 0.5$ corresponds to the mean delay target of $10$ ms used in
Sec.~\ref{sec:app_wireless}. The target error level is $\delta = 0.1$.

In the sequential calibration setting, data arrive one episode at
a time. At each round $t$, the acquisition policy $\mathcal{Q}_t$ selects a single
candidate $\lambda_i$ from the pool of uncertified configurations, one scheduling
episode is drawn i.i.d.\ from the Nokia simulator, and the per-round normalized loss
$L_{\lambda_i}(Z_i^t)$ is observed. The e-process for $\lambda_i$ is then updated as
per \eqref{eq:ch6_eprocess_update} with fixed betting parameter $\mu = 1.5$, while
the e-processes of all other candidates remain unchanged. The maximum number of
testing rounds is $t_{\max} = 2000$.

We compare aLTT with the $\varepsilon$-greedy acquisition policy
for $\varepsilon \in \{0.25, 0.50, 0.75, 0.95\}$ against non-adaptive LTT
($\varepsilon = 1.0$, uniform random selection with no greedy component). Both FWER
control via Bonferroni and FDR control via e-BH are evaluated. Results are averaged
over $200$ independent calibration runs.

\subsection{Results}

Fig.~\ref{fig:ch6_wireless_tpr} shows the TPR as a function of
the testing round $t$ on a logarithmic scale for both FWER and FDR control.
The advantage of the $\varepsilon$-greedy acquisition policy over
non-adaptive LTT is clear in both panels. Under FDR control, aLTT with
$\varepsilon = 0.25$ reaches $50\%$ of LTT's final TPR in approximately
150 rounds, while non-adaptive LTT requires the
full testing budget of $t_{\max} = 2000$ rounds to reach the same level. The greedy
policy achieves this by concentrating episodes on the candidates with the largest
accumulated e-process values, building up evidence of reliability much faster than
uniform random selection. Under FWER control, the Bonferroni threshold
$\delta/K = 0.005$ is more stringent, and the absolute TPR values are lower for all
methods, but the relative advantage of aLTT is equally pronounced: aLTT with
$\varepsilon = 0.25$ reaches $50\%$ of LTT's final TPR in approximately
335 rounds.

In both panels, the benefit of the greedy component is monotone in the $\varepsilon$-greedy acquisition probability
$\varepsilon$: smaller $\varepsilon$ (more exploitation of accumulated evidence) yields
both faster convergence and a higher final TPR, while $\varepsilon = 1.0$ (non-adaptive
LTT) is dominated by all greedy variants across the entire testing horizon.

\begin{figure}[t]
    \centering
    \includegraphics[width=\linewidth]{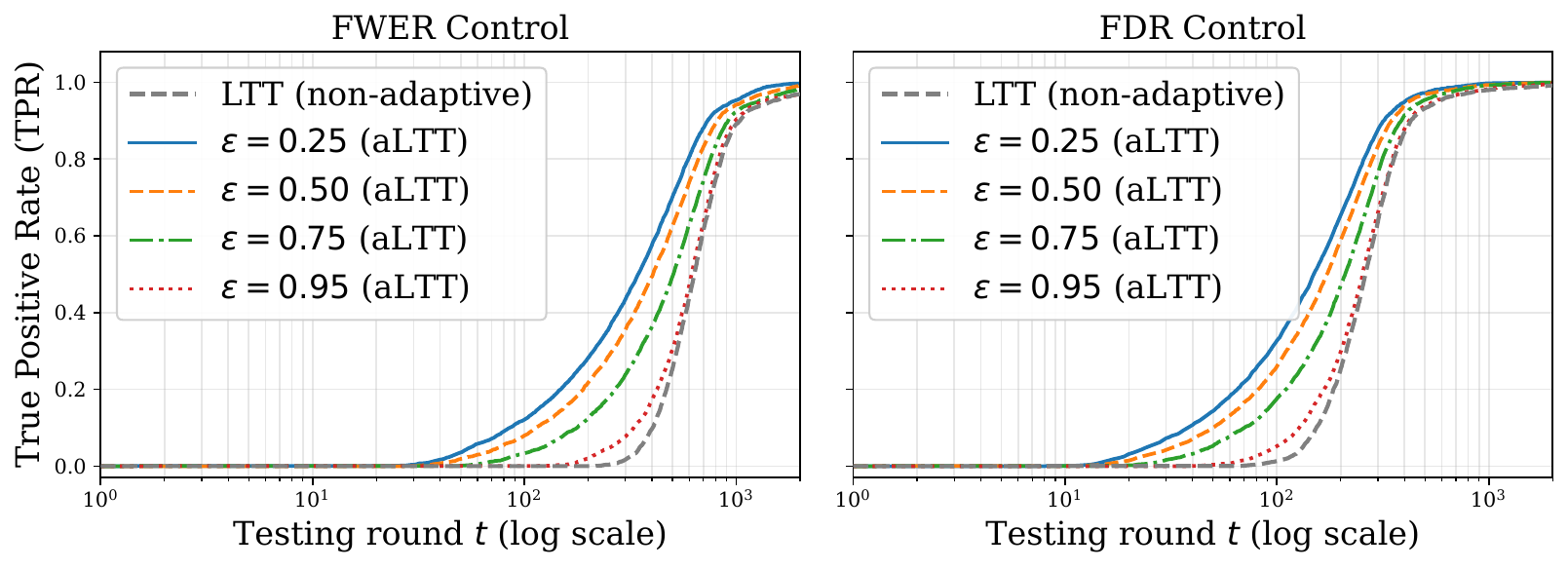}
    \caption{True positive rate (TPR) as a function of the number of testing rounds
    $t$ (log scale) in the wireless scheduling application of
    Sec.~\ref{sec:ch6_app_wireless}, using episode data from the Nokia wireless
    simulator \citep{Nokia} with $K = 20$ candidate configurations and reliability
    threshold $\alpha = 0.5$ (corresponding to a mean packet delay of $10$ ms).
    Results are shown for aLTT with $\varepsilon$-greedy acquisition for
    $\varepsilon \in \{0.25, 0.50, 0.75, 0.95\}$ and non-adaptive LTT
    ($\varepsilon = 1.0$), under FWER control via Bonferroni correction (left)
    and FDR control via e-BH (right), both at level $\delta = 0.1$.
    Curves are averaged over $200$ independent calibration runs. Smaller $\varepsilon$
    (more evidence-driven acquisition) yields substantially faster convergence to a
    large certified set.}
    \label{fig:ch6_wireless_tpr}
\end{figure}

These results confirm the two key properties of aLTT established
in this chapter. First, the $\varepsilon$-greedy acquisition policy substantially
accelerates the discovery of reliable hyperparameter configurations by concentrating
episodes on the candidates with the largest accumulated e-process values. Second, this
acceleration comes at no cost to statistical validity: by
Theorem~\ref{thm:ch6_altt_validity}, the FWER and FDR guarantees hold at every round
and for any data-dependent stopping time, as a direct consequence of the anytime
validity of the e-processes.

\section{Summary}
\label{sec:ch6_summary}

This chapter extended the LTT framework of
Chapter~\ref{chapter:mht} to the adaptive sequential setting, introducing the aLTT
algorithm of \citep{zecchin2024adaptive}. The central motivation is that, in many
practical deployments, evaluating candidate hyperparameters is costly, whether in
terms of online interactions with a physical environment, LLM inference costs, or
scarce labeled data, and it is therefore desirable to concentrate testing effort on
the most promising candidates and terminate calibration as soon as a reliable set has
been identified.

The key methodological departure from LTT is the replacement of
static p-values with e-processes. As established in
Lemma~\ref{lem:ch6_eprocess_valid}, the running product of per-round e-values forms a
non-negative supermartingale under the null hypothesis, and its stopped value remains a
valid e-value at any stopping time, including data-dependent ones. This anytime-validity
property, formalized in Lemma~\ref{lem:ch6_anytime_pval} for the derived p-values,
is what makes adaptive acquisition and early stopping statistically rigorous: the FWER
and FDR guarantees of Theorem~\ref{thm:ch6_altt_validity} hold regardless of when and
why the procedure terminates.

It is worth emphasizing that the statistical guarantees of aLTT
do not depend on the quality of the acquisition policy: any policy, adaptive or not,
yields a valid certified set. The policy affects only the speed at which reliable
candidates are discovered, not the correctness of the certification. This separation
between validity and efficiency is a practically important property that allows
practitioners to design aggressive acquisition strategies without fear of compromising
the reliability guarantees.

Finally, aLTT is complementary to the multi-objective extensions
of Chapter~\ref{chapter:multi_objective}. While PT and RG-PT address the challenge of
certifying hyperparameters under multiple simultaneous reliability constraints, aLTT
addresses the challenge of doing so efficiently when each evaluation is costly. Combining
the two directions, adaptive sequential testing under multiple objectives, is a natural
extension.


\chapter{Hyperparameter Selection with Autoevaluation}
\label{chapter:autoeval}

The frameworks developed in previous chapters assume that the loss $L_\lambda(X_i, Y_i)$ can be evaluated exactly on each calibration sample $(X_i, Y_i)$, where $Y_i$ is a ground-truth label obtained through human annotation or direct measurement. In many modern AI applications, however, producing ground-truth labels is prohibitively expensive.

Consider, for instance, the problem of selecting a prompt template for an LLM deployed on open-ended generation tasks, or of choosing a quantization format for a pretrained model evaluated on reading comprehension. In both cases, assessing whether a model output is correct ideally requires human judgment: a domain expert must read each response and decide whether it meets the desired quality standard. At scale, this process can be slow and costly, limiting the number of calibration samples available for hyperparameter selection.

\emph{Autoevaluation} (AutoEval) \citep{gu2024survey, zheng2023judging, novikova2017we, park2025adaptive} addresses this bottleneck by replacing ground-truth labels with labels generated by an automated evaluator, typically a larger pretrained model acting as a judge. Given an input $X$, the autoevaluator $f: \mathcal{X} \to \mathcal{Y}$ produces a synthetic label $f(X)$ at low cost. This makes it possible to evaluate candidate hyperparameters on large collections of unlabeled inputs, substantially reducing the dependence on human annotation.

The central challenge with AutoEval is \emph{bias}: the autoevaluator $f$ may systematically disagree with the ground-truth labeling mechanism, e.g., with human judgment, so that the synthetic loss $L_\lambda(X, f(X))$ is a biased estimate of the true risk $R(\lambda) = \mathbb{E}[L_\lambda(X,Y)]$. Relying naively on autoevaluated losses can therefore lead to incorrect hyperparameter selection, with no control over the resulting failure probability.

This chapter develops a statistically principled framework for hyperparameter selection that combines a small amount of ground-truth labeled calibration data with a large pool of autoevaluated data, while providing the same finite-sample reliability guarantees as the LTT framework of Chapter~\ref{chapter:mht}. The methodology is based on \emph{prediction-powered inference} (PPI) \citep{angelopoulos2023prediction, zrnic2024cross, song2026demystifying}, which uses real data to correct for the bias introduced by the autoevaluator, and on the testing-by-betting framework introduced in Chapter~\ref{chapter:adaptive} to construct valid e-values from the bias-corrected estimates.

The chapter is organized as follows. Section~\ref{sec:ch7_setting} formalizes the autoevaluation setting and introduces the evaluation protocols. Section~\ref{sec:ch7_reval} reviews reliable evaluation using only real data (R-Eval) \citep{waudby2024estimating}. Section~\ref{sec:ch7_rautoeval} describes how PPI can be used to incorporate autoevaluated data into the testing framework (R-AutoEval) \citep{einbinder2024semi}. Section~\ref{sec:ch7_rautoeval_plus} presents R-AutoEval+, an adaptive method that automatically tunes the reliance on synthetic data based on the quality of the autoevaluator \citep{park2025adaptive}. Section~\ref{sec:ch7_example} illustrates the framework on the wireless scheduling application of Sec.~\ref{sec:app_wireless}.

\section{The Autoevaluation Setting}
\label{sec:ch7_setting}

As in the rest of the monograph, we fix a finite candidate set $\Lambda = \{\lambda_1, \ldots, \lambda_K\}$ of hyperparameters and a bounded loss function $L_\lambda: \mathcal{X} \times \mathcal{Y} \to [0,1]$. The true risk of each candidate is given by $R(\lambda) = \mathbb{E}[L_\lambda(X,Y)]$, with test data prior $(X,Y) \sim \mathbb{P}$.

As illustrated in Fig. \ref{fig:ch7_autoeval_setting}, in the autoevaluation setting, we assume access to two data sources:
\begin{itemize}
    \item A \textbf{ground-truth labeled dataset} $\mathcal{D}^{\mathrm{cal}} = \{(X_i, Y_i)\}_{i=1}^n$ of $n$ ground-truth labeled samples $(X_i, Y_i) \overset{\text{i.i.d.}}{\sim} \mathbb{P}$, where the number of labeled samples $n$ is small due to, e.g., annotation cost in the case of human labels.
    \item An \textbf{unlabeled dataset} $\mathcal{D}^{\mathrm{unl}} = \{\tilde{X}_i\}_{i=1}^N$ of $N$ unlabeled inputs $\tilde{X}_i \overset{\text{i.i.d.}}{\sim} \mathbb{P}_X$, where the number of unlabeled inputs $N$ is large and the inputs are cheap to obtain.
\end{itemize}
A pre-trained \textbf{autoevaluator} $f: \mathcal{X} \to \mathcal{Y}$ provides synthetic labels $f(\tilde{X})$ for inputs in the unlabeled set. The autoevaluator may be, for instance, a large LLM used as a judge. The autoevaluator $f$ is applied to the unlabeled dataset $\mathcal{D}^{\mathrm{unl}}$ to obtain the synthetic calibration dataset $\tilde{\mathcal{D}}^\mathrm{cal} = \{(\tilde{X}_i, f(\tilde{X}_i))\}_{i = 1}^N$.

\begin{figure}[t]
    \centering
    \includegraphics[width=0.65\linewidth]{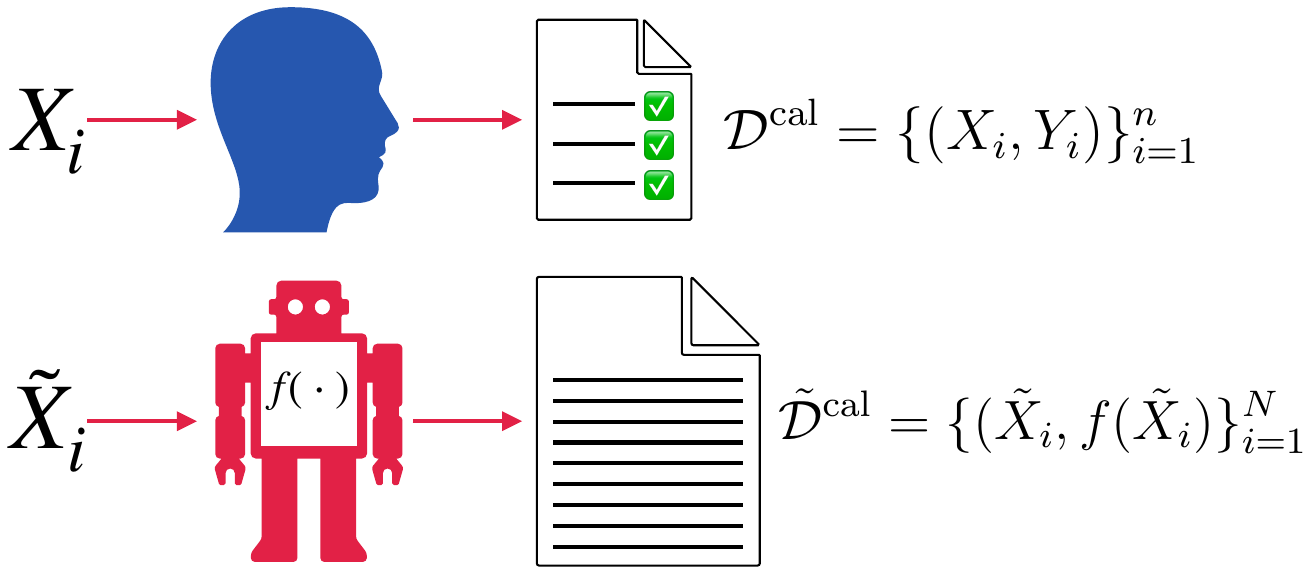}
    \caption{Illustration of the two data sources in the autoevaluation setting.
    The ground-truth labeled dataset $\mathcal{D}^{\mathrm{cal}} = \{(X_i, Y_i)\}_{i=1}^n$
    is produced, e.g., by a human annotator and is small due to annotation cost.
    The synthetic calibration dataset $\tilde{\mathcal{D}}^{\mathrm{cal}} =
    \{(\tilde{X}_i, f(\tilde{X}_i))\}_{i=1}^N$ is produced by applying the
    autoevaluator $f$ to the large unlabeled dataset $\mathcal{D}^{\mathrm{unl}}$,
    and is cheap to obtain but potentially biased.}
    \label{fig:ch7_autoeval_setting}
\end{figure}

The autoevaluated risk estimate for a hyperparameter $\lambda$ is obtained using the synthetic labels as
\begin{equation}
\label{eq:ch7_autoeval_risk}
\widehat{R}^f_N(\lambda) = \frac{1}{N} \sum_{i=1}^N L_\lambda\!\big(\tilde{X}_i, f(\tilde{X}_i)\big).
\end{equation}
While this estimate has low variance due to the large size $N$ of the unlabeled dataset, it is generally biased, i.e.,
\begin{equation}
\mathbb{E}[\widehat{R}^f_N(\lambda)] = \mathbb{E}[L_\lambda(X, f(X))] \neq R(\lambda).
\end{equation}

Three strategies are distinguished in this chapter:
\begin{itemize}
    \item \textbf{Eval}: This baseline scheme uses only the ground-truth labeled
    data $\mathcal{D}^{\mathrm{cal}}$ to estimate the risk via the empirical average
    \begin{equation}
    \label{eq:ch7_eval_risk}
    \widehat{R}_n(\lambda) = \frac{1}{n} \sum_{i=1}^n L_\lambda(X_i, Y_i).
    \end{equation}
    This estimate is unbiased, i.e., $\mathbb{E}[\widehat{R}_n(\lambda)] = R(\lambda)$,
    but it has high variance due to the small labeled dataset size $n$.
    \item \textbf{AutoEval}: This complementary baseline uses only the autoevaluated
    data via the estimate \eqref{eq:ch7_autoeval_risk}. This approach exhibits low
    variance due to large unlabeled dataset size $N$ but is generally biased.
    \item \textbf{R-AutoEval \citep{einbinder2024semi} / R-AutoEval+
    \citep{park2025adaptive}}: These methods combine both data sources to achieve
    bias-corrected, statistically valid risk certification by leveraging PPI \citep{angelopoulos2023prediction}.
\end{itemize}

As in Chapter~\ref{chapter:adaptive}, reliability certification in this chapter is formulated as a sequential testing problem. Specifically, the ground-truth labeled samples in $\mathcal{D}^{\mathrm{cal}}$ are processed one by one, and at each step the accumulated evidence is used to update an e-process for each candidate hyperparameter, as described in Sec.~\ref{sec:ch6_eprocess}. This sequential formulation enables early termination as soon as sufficient evidence has been gathered, reducing the number of expensive ground-truth labeled samples required.

As formalized in Definition~\ref{def:ch2_rcp}, the goal remains to select a hyperparameter $\hat{\lambda} \in \Lambda$ satisfying the $(\alpha, \delta)$-risk-controlling guarantee $\mathbb{P}(R(\hat{\lambda}) \leq \alpha) \geq 1 - \delta$. The challenge is to achieve this guarantee while reducing the amount of ground-truth labeled data required.

To measure the cost of an evaluation method, we define its \emph{sample complexity} as the minimum expected number of ground-truth labeled samples needed to certify a reliable hyperparameter $\lambda$, i.e., to ensure that it is included in the set $\hat{\Lambda}$. This can be defined mathematically as the conditional average
\begin{equation}
\label{eq:ch7_sample_complexity}
n_{\min}(\delta, \lambda) = \mathbb{E}\!\left[\min\{n : {\lambda} \in \hat{\Lambda}\} \;\middle|\; R({\lambda}) \le \alpha\right].
\end{equation}
A method with smaller sample complexity can certify reliable hyperparameters using fewer expensive ground-truth labeled samples.

\section{Reliable Evaluation with Real Data}
\label{sec:ch7_reval}

The simplest statistically valid approach is to use only the ground-truth labeled data $\mathcal{D}^{\mathrm{cal}}$. In the sequential testing setting described above, this is done by applying the e-process framework of Chapter~\ref{chapter:adaptive}, with the per-round observation $L_\lambda(X_i, Y_i)$, which is an unbiased estimate of $R(\lambda)$ bounded in $[0,1]$.

Following the testing-by-betting construction of Sec.~\ref{sec:ch6_eprocess}, for any given candidate hyperparameter $\lambda$, R-Eval \citep{waudby2024estimating} constructs the e-process
\begin{equation}
\label{eq:ch7_reval_eprocess}
E^{\mathrm{R\text{-}Eval}}_n = \prod_{i=1}^n \bigl(1 - \mu_i\bigl(L_\lambda(X_i, Y_i) - \alpha\bigr)\bigr),
\end{equation}
where the betting parameter $\mu_i \in (0, 1/(1-\alpha))$ is chosen using past observations. By Lemma~\ref{lem:ch6_eprocess_valid}, the sequence \eqref{eq:ch7_reval_eprocess} provides a valid e-process for the null hypothesis $\mathcal{H}_\lambda: R(\lambda) > \alpha$. The corresponding anytime-valid test, with target error level $\delta \in (0,1)$, is
\begin{equation}
\label{eq:ch7_reval_test}
T_n^{\mathrm{R\text{-}Eval}} = \mathbb{I}\!\left\{\max_{i \le n} E^{\mathrm{R\text{-}Eval}}_i \ge 1/\delta\right\},
\end{equation}
which certifies $\lambda$ as reliable as soon as the e-process exceeds the threshold $1/\delta$ and satisfies $\mathbb{P}(T_n^{\mathrm{R\text{-}Eval}} = 1 \mid R(\lambda) > \alpha) \le \delta$. To certify a set $\hat{\Lambda} \subseteq \Lambda$, the e-values $\{E^{\mathrm{R\text{-}Eval}}_n(\lambda_k)\}_{k=1}^K$ are fed into any of the MHT procedures of Chapter~\ref{chapter:mht}, exactly as in Algorithm~\ref{alg:ch6_altt}.

The limitation of R-Eval is its sample complexity. Since only the $n$ ground-truth labeled samples are used, the e-process grows slowly when $n$ is small, delaying certification. In fact, as shown in \citep{waudby2024estimating}, as $\delta \to 0$, the sample complexity $n_{\min}(\delta, \lambda)$ in \eqref{eq:ch7_sample_complexity} scales as $\log(1/\delta)/g^\star(\lambda)$, where $g^\star(\lambda) = \mathbb{E}[\log(1 - \mu^\star(L_\lambda(X,Y) - \alpha)) \mid R(\lambda) \le \alpha]$ is the maximum expected log-wealth increment, and $\mu^\star$ is the optimal constant betting parameter. As discussed in Sec.~\ref{sec:ch6_eprocess}, this quantity decreases with the variance of the per-sample loss, so high-variance losses lead to larger sample complexity.

\section{Bias-Corrected Evaluation via PPI}
\label{sec:ch7_rautoeval}

R-AutoEval \citep{einbinder2024semi} addresses the sample complexity
limitations of R-Eval by incorporating autoevaluated data through a bias correction
mechanism based on PPI \citep{angelopoulos2023prediction}. The core idea of PPI
is to use the large unlabeled dataset to obtain a low-variance but potentially
biased risk estimate, and then use the small ground-truth labeled dataset to correct
for this bias. This yields an effective observation that is unbiased and
potentially lower-variance than the raw ground-truth labeled loss alone when the autoevaluator
is sufficiently accurate.

To proceed, let $r = \lfloor |\mathcal{D}^{\mathrm{unl}}|/n \rfloor$ denote
the ratio of unlabeled to labeled samples. Then, pair each labeled data point
$(X_i, Y_i)$ with a distinct subset of $r$ autoevaluated samples drawn from
$\tilde{\mathcal{D}}^{\mathrm{cal}}$, namely $\{(\tilde{X}_{i'}, f(\tilde{X}_{i'}))\}_{i' = r(i-1)+1}^{ri}$.

Using the labeled data point $(X_i, Y_i)$ and the corresponding autoevaluated batch $\{(\tilde{X}_{i'}, f(\tilde{X}_{i'}))\}_{i' = r(i-1)+1}^{ri}$, we construct the PPI-based estimate
\begin{equation}
\label{eq:ch7_ppi_obs}
\ell^f_i(\lambda) = \underbrace{\frac{1}{r} \sum_{i' = r(i-1)+1}^{ri} L_\lambda\!\big(\tilde{X}_{i'}, f(\tilde{X}_{i'})\big)}_{\text{autoevaluator data}} + \underbrace{L_\lambda(X_i, Y_i) - L_\lambda(X_i, f(X_i))}_{\text{bias correction}}.
\end{equation}
The first term is the empirical autoevaluated risk over the $r$ synthetic samples, providing low-variance, but biased, information about the true risk $R(\lambda)$. The second term corrects for the bias of the autoevaluator using the paired real observation. As shown next, the estimate \eqref{eq:ch7_ppi_obs} is unbiased.

\begin{lemma}[Validity of PPI estimates]
\label{lem:ch7_ppi_unbiased}
The estimates $\ell^f_i(\lambda)$ in \eqref{eq:ch7_ppi_obs} are unbiased, i.e.,
\begin{equation}
\mathbb{E}[\ell^f_i(\lambda)] = R(\lambda).
\end{equation}
Moreover, the estimate is bounded as $\ell^f_i(\lambda) \in [-1, 2]$.
\end{lemma}
\begin{proof}
By linearity of expectation and using the fact that $\tilde{X}$ and $X$ share the same marginal distribution $\mathbb{P}_X$, i.e., $\mathbb{E}[L_\lambda(\tilde{X},
f(\tilde{X}))] = \mathbb{E}[L_\lambda(X, f(X))]$, we obtain
\begin{equation}
\mathbb{E}[\ell^f_i(\lambda)] = \mathbb{E}[L_\lambda(\tilde{X}, f(\tilde{X}))] +
\mathbb{E}[L_\lambda(X,Y)] - \mathbb{E}[L_\lambda(X, f(X))] = R(\lambda).
\end{equation}
For the range of the estimate, since each loss term lies in $[0,1]$, the lower bound is achieved
when the autoevaluated average equals $0$ and $L_\lambda(X_i, Y_i) -
L_\lambda(X_i, f(X_i)) = -1$, giving the lower bound $\ell^f_i(\lambda) \ge -1$. The upper bound is
achieved when the autoevaluated average equals $1$ and $L_\lambda(X_i, Y_i) -
L_\lambda(X_i, f(X_i)) = 1$, giving the upper bound $\ell^f_i(\lambda) \le 2$.
\end{proof}

Since the estimate $\ell^f_i(\lambda)$ is unbiased and bounded, it can be used in place of the estimate $L_\lambda(X_i, Y_i)$ based solely on labeled data to construct a valid e-process for the null hypothesis $\mathcal{H}_\lambda: R(\lambda) > \alpha$. Accordingly, R-AutoEval \citep{einbinder2024semi} replaces the real-data observations in \eqref{eq:ch7_reval_eprocess} with the effective observations \eqref{eq:ch7_ppi_obs}, yielding
\begin{equation}
\label{eq:ch7_rautoeval_eprocess}
E^{\mathrm{R\text{-}AutoEval}}_n = \prod_{i=1}^n \bigl(1 - \mu_i\bigl(\ell^f_i(\lambda) - \alpha\bigr)\bigr),
\end{equation}
where the betting parameter $\mu_i \in (0, 1/(2-\alpha))$ accounts for the extended range $[-1, 2]$ of $\ell^f_i(\lambda)$. The test is then constructed as in \eqref{eq:ch7_reval_test} with the e-process $E^{\mathrm{R\text{-}AutoEval}}_n$ in lieu of $E^{\mathrm{R\text{-}Eval}}_n$.

The sample efficiency of R-AutoEval relative to R-Eval depends on whether the estimate $\ell^f_i(\lambda)$ has lower variance than the raw loss $L_\lambda(X_i, Y_i)$. When the autoevaluator is accurate, the bias correction term $L_\lambda(X_i, Y_i) - L_\lambda(X_i, f(X_i))$ in \eqref{eq:ch7_ppi_obs} has small variance and the synthetic data effectively augments the calibration set, reducing sample complexity. However, when the autoevaluator is inaccurate, this correction term has high variance, and incorporating the synthetic data can actually increase the sample complexity compared to R-Eval \citep{angelopoulos2023prediction}.

This observation motivates a more adaptive approach: rather than committing to a fixed reliance on synthetic data, one should automatically adjust the reliance of the PPI estimator \eqref{eq:ch7_ppi_obs} based on the observed quality of the autoevaluator.

\section{Adaptive Autoevaluation (R-AutoEval+)}
\label{sec:ch7_rautoeval_plus}

R-AutoEval+ \citep{park2025adaptive} resolves the tension between R-Eval and R-AutoEval by introducing a \emph{reliance factor} $\rho \in [0,1]$ that controls the degree to which the estimate \eqref{eq:ch7_ppi_obs} relies on synthetic data. Rather than fixing $\rho$ in advance, R-AutoEval+ considers a finite grid of $S$ candidate values
\begin{equation}
\label{eq:ch7_rho_grid}
0 = \rho_1 < \rho_2 < \cdots < \rho_S = 1,
\end{equation}
and adaptively determines which candidate best matches the quality of the autoevaluator based on the accumulated evidence.

For each candidate $\rho_s$, define the corresponding estimate by generalizing \eqref{eq:ch7_ppi_obs} as \citep{angelopoulos2023prediction}
\begin{equation}
\label{eq:ch7_rho_obs}
\begin{split}
\ell^f_{s,i}(\lambda) &= \underbrace{\frac{\rho_s}{r} \sum_{i' = r(i-1)+1}^{ri} L_\lambda\!\big(\tilde{X}_{i'}, f(\tilde{X}_{i'})\big)}_{\text{autoevaluator data}} \\
&\quad + \underbrace{L_\lambda(X_i, Y_i) - \rho_s \cdot L_\lambda(X_i, f(X_i))}_{\text{bias correction}}.
\end{split}
\end{equation}
Setting $\rho_s = 0$ recovers the pure real-data observation, i.e., R-Eval, whereas $\rho_s = 1$ recovers the PPI effective observation \eqref{eq:ch7_ppi_obs}, i.e., R-AutoEval.

\begin{lemma}[Validity of the $\rho_s$-weighted estimates]
\label{lem:ch7_rho_unbiased}
The estimate $\ell^f_{s,i}(\lambda)$ in \eqref{eq:ch7_rho_obs} is an unbiased estimate of $R(\lambda)$, i.e., $\mathbb{E}[\ell^f_{s,i}(\lambda)] = R(\lambda)$, with range $\ell^f_{s,i}(\lambda) \in [-\rho_s, 1+\rho_s]$.
\end{lemma}
\begin{proof}
The proof follows by the same argument as Lemma~\ref{lem:ch7_ppi_unbiased}.
\end{proof}

For each candidate factor $\rho_s$, R-AutoEval+ constructs a separate e-process
\begin{equation}
\label{eq:ch7_Es}
E_{s,i}^\lambda = \prod_{j=1}^i \bigl(1 - \mu_{s,j}\bigl(\ell^f_{s,j}(\lambda) - \alpha\bigr)\bigr),
\end{equation}
using the effective observations \eqref{eq:ch7_rho_obs}. Larger values of $E_{s,i}^\lambda$ indicate that reliance factor $\rho_s$ is accumulating strong evidence in favor of the reliability of $\lambda$.

The R-AutoEval+ e-value combines these $S$ processes via adaptive exponential weighting as
\begin{equation}
\label{eq:ch7_combined_evalue}
E_n^\lambda = \prod_{i=1}^n \sum_{s=1}^S w_{s,i} \bigl(1 - \mu_{s,i}\bigl(\ell^f_{s,i}(\lambda) - \alpha\bigr)\bigr),
\end{equation}
where the weight $w_{s,i}$ assigned to the $s$-th factor at round $i$ is updated proportionally to the evidence accumulated by that factor as
\begin{equation}
\label{eq:ch7_weight_update}
w_{s,i}(\lambda) = \frac{w_{s,0} \cdot E_{s,i-1}^\lambda}{\sum_{s'=1}^S w_{s',0} \cdot E_{s',i-1}^\lambda},
\end{equation}
with strictly positive initial weights $\{w_{s,0}\}_{s=1}^S$ satisfying $\sum_{s=1}^S w_{s,0} = 1$.

The update rule \eqref{eq:ch7_weight_update} is an instance of exponential weighting, namely the Hedge algorithm \citep{freund1997decision}. In this algorithm, factors that accumulate large e-values, indicating that the corresponding level of reliance on synthetic data is well-supported by the data, are assigned proportionally larger weights.

\begin{lemma}[Validity of the R-AutoEval+ e-value]
\label{lem:ch7_rautoevalplus_valid}
The quantity $E_n^\lambda$ in \eqref{eq:ch7_combined_evalue} is a valid e-value for the null hypothesis $\mathcal{H}_\lambda: R(\lambda) > \alpha$, i.e., $\mathbb{E}[E_n^\lambda \mid R(\lambda) > \alpha] \le 1$.
\end{lemma}
\begin{proof}
Since each quantity $\ell^f_{s,i}(\lambda)$ is an unbiased estimate of the true risk $R(\lambda)$ with range $[-\rho_s, 1+\rho_s]$ and $\mu_{s,i} \in (0, 1/(1+\rho_s - \alpha))$, each term $1 - \mu_{s,i}(\ell^f_{s,i}(\lambda) - \alpha)$ is non-negative. Under the null hypothesis $R(\lambda) > \alpha$, since the weights $w_{s,i}$ depend only on past observations $\{\ell^f_{s,j}(\lambda)\}_{j < i}$, the conditional expectation of each factor satisfies
\begin{equation}
\begin{split}
&\mathbb{E}\!\left[\sum_{s=1}^S w_{s,i}\bigl(1 - \mu_{s,i}(\ell^f_{s,i}(\lambda) - \alpha)\bigr) \;\middle|\; \{\ell^f_{s,j}\}_{j < i},\, R(\lambda) > \alpha\right] \\
&\qquad = \sum_{s=1}^S w_{s,i}\bigl(1 - \mu_{s,i}(R(\lambda) - \alpha)\bigr) \le 1.
\end{split}
\end{equation}
Taking iterated expectations and using the fact that $E_0^\lambda = 1$, by induction we obtain $\mathbb{E}[E_n^\lambda \mid R(\lambda) > \alpha] \le 1$.
\end{proof}

The validity of the e-value in Lemma~\ref{lem:ch7_rautoevalplus_valid} implies that the test $T_n = \mathbb{I}\{\max_{i \le n} E_i^\lambda \ge 1/\delta\}$ satisfies the reliability condition $\mathbb{P}(T_n = 1 \mid R(\lambda) > \alpha) \le \delta$, providing the same finite-sample guarantee as R-Eval and R-AutoEval. To certify a set of hyperparameters $\hat{\Lambda} \subseteq \Lambda$ with FWER or FDR control, Algorithm~\ref{alg:ch7_rautoeval_plus} is run for each candidate $\lambda_k \in \Lambda$, and the resulting e-values $\{E_n^{\lambda_k}\}_{k=1}^K$ are fed into any of the MHT procedures described in Chapter~\ref{chapter:mht}.

The adaptive weighting in \eqref{eq:ch7_weight_update} is
designed so that R-AutoEval+ tracks the best reliance factor $\rho_s$ in the
grid \eqref{eq:ch7_rho_grid}, achieving a sample complexity no larger than the
best of R-Eval and R-AutoEval. In particular, reference \citep{park2025adaptive} shows that, under the main condition that the betting strategies $\{\mu_{s,i}\}$ satisfy sublinear regret with respect to the optimal constant betting parameters $\mu_{s,\star}$, the sample complexity of R-AutoEval+ satisfies, as $\delta \to 0$, the inequality
\begin{equation}
\label{eq:ch7_dominance}
n^{\mathrm{R\text{-}AutoEval+}}_{\min}(\delta, \lambda) \;\le\; \min\!\left\{n^{\mathrm{R\text{-}Eval}}_{\min}(\delta, \lambda),\; n^{\mathrm{R\text{-}AutoEval}}_{\min}(\delta, \lambda)\right\}.
\end{equation}
Furthermore, a strict improvement in sample complexity is obtained when there exists an intermediate reliance factor $\rho_s \in (0,1)$ for which the variance of $\ell^f_{s,i}(\lambda)$ is strictly smaller than for both $\rho_1 = 0$ (corresponding to R-Eval, with observation $L_\lambda(X_i, Y_i)$) and $\rho_S = 1$ (corresponding to R-AutoEval, with observation $\ell^f_i(\lambda)$). This occurs whenever the autoevaluator is sufficiently but not perfectly accurate.

The complete R-AutoEval+ procedure for a single candidate hyperparameter $\lambda$ is summarized in Algorithm~\ref{alg:ch7_rautoeval_plus}. It processes labeled samples sequentially, updating the weighted e-value at each round. Once Algorithm~\ref{alg:ch7_rautoeval_plus} has been run for all candidates $\lambda_k \in \Lambda$, the resulting e-values are fed into the MHT procedures of Chapter~\ref{chapter:mht} to obtain a certified set $\hat{\Lambda}$, exactly as in Algorithm~\ref{alg:ch6_altt}.

\begin{algorithm}[t]
\caption{R-AutoEval+ \citep{park2025adaptive}}
\label{alg:ch7_rautoeval_plus}
\begin{algorithmic}
\STATE \textbf{Input:} Candidate hyperparameter $\lambda$; ground-truth labeled data
$\{(X_i, Y_i)\}_{i=1}^n$; unlabeled data $\mathcal{D}^{\mathrm{unl}}$; autoevaluator
$f$; risk threshold $\alpha$; error level $\delta$; reliance factors
$\{\rho_s\}_{s=1}^S$ as in \eqref{eq:ch7_rho_grid}; initial weights
$\{w_{s,0}\}_{s=1}^S$ with $\sum_{s=1}^S w_{s,0} = 1$

\STATE \textbf{Output:} E-value $E_n^\lambda$ for use in MHT (see Chapter~\ref{chapter:mht})

\STATE Initialize $E_0^\lambda \leftarrow 1$, $E_{s,0}^\lambda \leftarrow 1$ for all $s$

\FOR{$i = 1, \ldots, n$}
    \STATE Compute weights $w_{s,i} \leftarrow
    \dfrac{w_{s,0} \cdot E_{s,i-1}^\lambda}{\sum_{s'=1}^S w_{s',0} \cdot E_{s',i-1}^\lambda}$
    for all $s$
    \FOR{$s = 1, \ldots, S$}
        \STATE Compute effective observation $\ell^f_{s,i}(\lambda)$ via \eqref{eq:ch7_rho_obs}
        \STATE Update $E_{s,i}^\lambda \leftarrow E_{s,i-1}^\lambda \cdot \bigl(1 - \mu_{s,i}
        (\ell^f_{s,i}(\lambda) - \alpha)\bigr)$
    \ENDFOR
    \STATE Update $E_i^\lambda \leftarrow E_{i-1}^\lambda \cdot \sum_{s=1}^S w_{s,i} \cdot
    \bigl(1 - \mu_{s,i}(\ell^f_{s,i}(\lambda) - \alpha)\bigr)$
\ENDFOR

\STATE \textbf{return} $E_n^\lambda$
\end{algorithmic}
\end{algorithm}

\section{Example: Adaptive Autoevaluation for Wireless Scheduling}
\label{sec:ch7_example}

We illustrate the autoevaluation framework on the wireless scheduling application of Sec.~\ref{sec:app_wireless}, using the Nokia delay table \citep{Nokia} as the source of ground-truth episodes. As in Sec.~\ref{sec:ch6_app_wireless}, we consider $K=20$ candidate scheduler configurations. To bring the variance of the per-sample loss into a regime where the autoevaluation framework is informative, we adopt the per-episode binary service-level agreement (SLA) loss
\begin{equation}
\label{eq:ch7_example_loss}
L_{\lambda_k}(X_i,Y_i) \;=\; \mathbb{I}\!\left\{\mathrm{delay}_{\lambda_k}(X_i) > 10\,\mathrm{ms}\right\},
\end{equation}
which records whether the average packet delay in episode $X_i$ violates the $10$ ms SLA under configuration $\lambda_k$. We set the tolerated violation rate to $\alpha = 0.3$ and the outage rate to $\delta = 0.1$. Under the loss (\ref{eq:ch7_example_loss}), the empirical risks of the $20$ candidates range from $0$ to $0.84$, with $16$ candidates reliable.

The autoevaluator is constructed as a biased low-variance proxy of the binary loss. Specifically, for each candidate $\lambda_k$ with population mean $R(\lambda_k)$, the autoevaluator's per-sample loss is given by the convex-shrinkage construction
\begin{equation}
\label{eq:ch7_example_autoeval}
L_{\lambda_k}(X_i, f(X_i)) \;=\; R(\lambda_k) + b_k + \kappa\,\bigl(L_{\lambda_k}(X_i,Y_i) - R(\lambda_k)\bigr),
\end{equation}
where $\kappa\in[0,1]$ controls the fidelity of the autoevaluator and $b_k$ is a small per-candidate systematic bias drawn from $\mathcal{N}(0,\,0.02^2)$ and truncated to the range that keeps $L_{\lambda_k}(X_i, f(X_i))$ in $[0,1]$ without clipping. The construction \eqref{eq:ch7_example_autoeval} models a low-fidelity surrogate of the simulator: with $\kappa = 0.9$, the variance of the autoevaluator is reduced by a factor of $\kappa^2 \approx 0.8$ relative to the binary truth, while the per-candidate bias $b_k$, fixed once drawn, introduces a constant offset between the autoevaluator's mean and the true risk $R(\lambda_k)$, simulating the calibration mismatch that occurs when an automated evaluator is not perfectly aligned with the ground-truth loss. The unlabeled term in the PPI estimate \eqref{eq:ch7_ppi_obs} is approximated by the population mean of the autoevaluator over all $200$ episodes.

A direct computation shows that, with the autoevaluator (\ref{eq:ch7_example_autoeval}), the PPI effective observation \eqref{eq:ch7_rho_obs} reduces to
\begin{equation}
\label{eq:ch7_example_ppi}
\ell^f_{s,i}(\lambda_k) \;=\; (1 - \rho_s\,\kappa)\,L_{\lambda_k}(X_i,Y_i) \;+\; \rho_s\,\kappa\, R(\lambda_k),
\end{equation}
which is a convex combination of two quantities in $[0,1]$, implying the inclusion $\ell^f_{s,i}(\lambda_k) \in [0,1]$ for every reliance factor $\rho_s$. Under this tighter range, the per-round factor $1-\mu_{s,i}(\ell^f_{s,i}(\lambda_k)-\alpha)$ remains non-negative as long as $\mu_{s,i} \in (0,\,1/(1-\alpha))$, which is strictly more permissive than the worst-case range $\mu_{s,i}\in(0,\,1/(1+\rho_s-\alpha))$ of Lemma~\ref{lem:ch7_rautoevalplus_valid}.

We compare three statistically valid certification methods: R-Eval, which uses only ground-truth labels via the e-process \eqref{eq:ch7_reval_eprocess}; R-AutoEval with reliance factor $\rho=1$, i.e., the PPI e-process \eqref{eq:ch7_rautoeval_eprocess}; and R-AutoEval+ with the reliance grid $\{0,\,1/3,\,2/3,\,1\}$ and uniform initial Hedge weights $w_{s,0}=1/4$. For all three methods, the betting parameter $\mu_{s,i}$ is chosen via the aGRAPA strategy of Sec.~\ref{sec:ch6_eprocess}, which sets $\mu_{s,i}$ to the value maximizing the expected log-wealth growth (see Sec.~\ref{sec:ch6_eprocess}), estimated from the empirical mean and variance of the past $i-1$ effective observations, clipped to the validity range $(0,\,1/(1-\alpha))$. Certification of $\lambda_k$ from a budget of $n$ ground-truth labeled samples uses the Bonferroni-corrected anytime-valid test $\max_{\tau\le n} E^\lambda_\tau \ge |\Lambda|/\delta$, providing FWER control at level $\delta$. We report the true positive rate (TPR), defined as the fraction of truly reliable hyperparameters certified, averaged over $300$ independent calibration sequences.

Fig.~\ref{fig:ch7_rautoeval_wireless} presents the results. Panel (a) plots TPR against the number of ground-truth labeled samples $n$, while panel (b) plots the miss rate $1-\mathrm{TPR}$ on log--log axes. The autoevaluation methods are dramatically more sample-efficient than R-Eval: R-AutoEval already certifies $94\%$ of the reliable hyperparameters by $n = 80$ labeled samples and reaches near-perfect TPR by $n = 400$, whereas R-Eval reaches only $59\%$ at $n = 80$ and is still below $80\%$ at $n = 2000$.

The bottleneck for R-Eval is the variance of the binary loss, which equals $R(\lambda_k)(1-R(\lambda_k))$ and can be close to $1/4$ for hyperparameters near the threshold; the PPI estimate \eqref{eq:ch7_example_ppi} reduces this variance by a factor $(1-\rho_s\kappa)^2$, which gives R-AutoEval roughly an order-of-magnitude advantage.

R-AutoEval+ tracks R-AutoEval after a brief warm-up of fewer than $50$ rounds, demonstrating that the Hedge re-weighting in \eqref{eq:ch7_weight_update} concentrates rapidly on the high-reliance processes when the autoevaluator is informative. All three methods satisfy the finite-sample reliability constraint, with empirical FWER violation rate equal to zero across all values of $n$.

The main takeaway message is that, when the autoevaluator is informative enough to reduce the variance of the PPI effective observation, R-AutoEval and R-AutoEval+ certify reliable hyperparameters with markedly fewer ground-truth labeled samples than R-Eval. R-AutoEval+ retains this benefit while remaining safe to deploy without prior knowledge of the quality of the autoevaluator.

\begin{figure}[t]
    \centering
    \includegraphics[width=\linewidth]{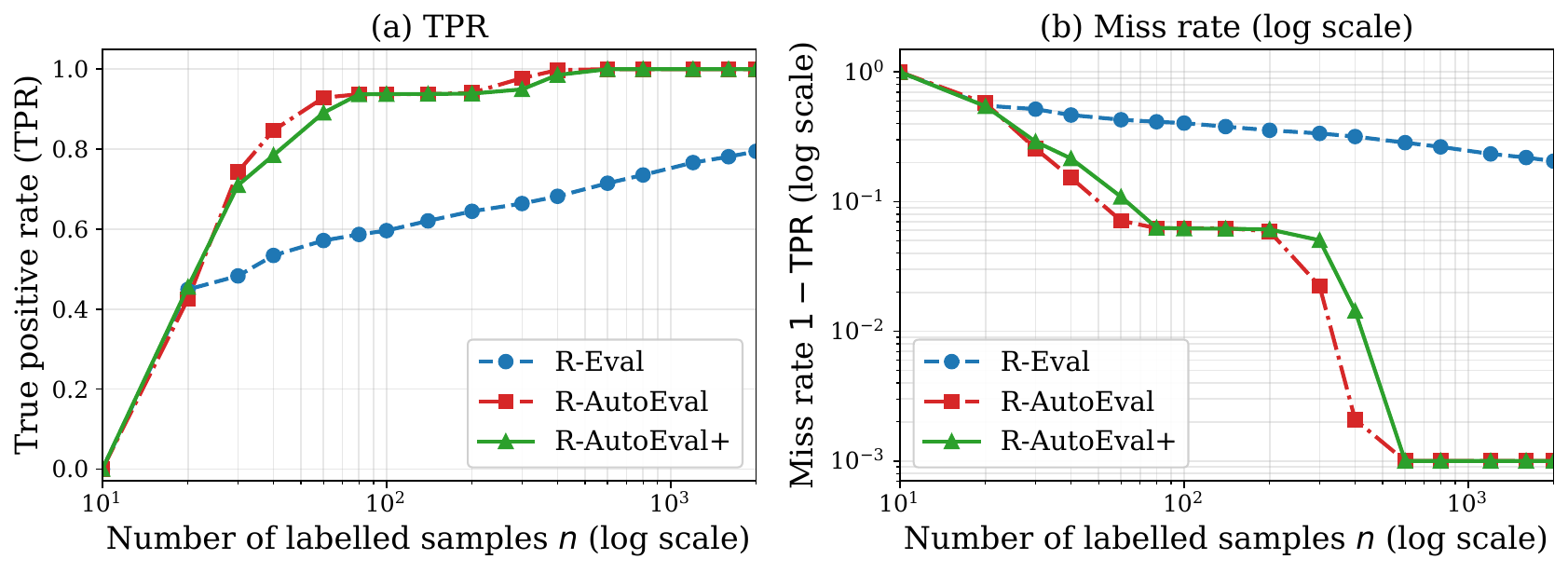}
    \caption{Comparison of R-Eval, R-AutoEval (with reliance factor $\rho=1$), and R-AutoEval+ on the wireless scheduling task of Sec.~\ref{sec:app_wireless}, with $K=20$ candidate scheduler configurations, the binary SLA loss (\ref{eq:ch7_example_loss}), target risk $\alpha=0.3$, outage rate $\delta=0.1$, autoevaluator fidelity $\kappa=0.9$, and bias standard deviation $0.02$. (a) True positive rate (TPR) as a function of the number of ground-truth labeled samples $n$. (b) Miss rate $1-\mathrm{TPR}$ on log--log axes, which exposes the late-saturation regime. Results are averaged over $300$ independent calibration sequences. R-AutoEval and R-AutoEval+ certify essentially all reliable hyperparameters by $n\approx 400$, whereas R-Eval is still missing more than $20\%$ at $n = 2000$.}
    \label{fig:ch7_rautoeval_wireless}
\end{figure}

\section{Summary}
\label{sec:ch7_summary}

This chapter extended the LTT framework to the autoevaluation setting, where ground-truth labeled data are expensive and a large pool of autoevaluated data is available at low cost.

The core methodological challenge is that autoevaluators introduce bias, making it impossible to use synthetic labels directly in the testing framework without compromising the reliability guarantee. R-AutoEval addresses this challenge through PPI \citep{angelopoulos2023prediction}. By using a small amount of real data to correct for the autoevaluator's bias, R-AutoEval constructs unbiased effective observations that can be plugged into the e-process framework of Chapter~\ref{chapter:adaptive}.

However, R-AutoEval can increase sample complexity compared to R-Eval when the autoevaluator is inaccurate, since the bias correction term introduces additional variance. R-AutoEval+ \citep{park2025adaptive} resolves this tension by maintaining a family of reliance factors $\{\rho_s\}_{s=1}^S$ and adaptively weighting them based on accumulated evidence via the exponential weighting scheme \eqref{eq:ch7_weight_update}. The resulting method provably dominates both R-Eval and R-AutoEval in terms of sample complexity \eqref{eq:ch7_dominance}, while preserving the same finite-sample reliability guarantee.

A key practical message of this chapter is that, under suitable assumptions, autoevaluators can be incorporated into statistically valid hyperparameter selection without risk: R-AutoEval+ is guaranteed to perform at least as well as conventional ground-truth labeled testing, and strictly better when the autoevaluator provides useful information. This makes it a natural default for modern AI deployment pipelines where automated evaluation tools are readily available \citep{einbinder2024semi,angelopoulos2023prediction,angelopoulos2023ppi++}.

The autoevaluation framework of this chapter is complementary to the adaptive sequential framework of Chapter~\ref{chapter:adaptive}. While aLTT reduces evaluation cost by concentrating testing effort on the most promising candidates across rounds, R-AutoEval+ reduces the cost of each individual evaluation by substituting expensive human labels with cheap synthetic ones. Combining both directions, i.e., adaptive sequential testing with autoevaluated observations, is a natural avenue for future work.

\chapter{Outlook and Future Work}
\label{chapter:conclusions}

This monograph has developed a rigorous framework for statistically valid hyperparameter selection in AI and machine learning systems. Starting from the Learn-Then-Test (LTT) method \citep{angelopoulos2025learn}, the monograph has progressively extended the core methodology along several dimensions: non-average risk measures such as quantile risk and information-theoretic constraints (Chapter~\ref{chapter:beyond_average}); multi-objective selection via Pareto testing and reliability graphs (Chapter~\ref{chapter:multi_objective}); adaptive and sequential selection using e-processes and the adaptive LTT (aLTT) algorithm (Chapter~\ref{chapter:adaptive}); and hyperparameter selection with limited labeled data via prediction-powered inference and the R-AutoEval framework (Chapter~\ref{chapter:autoeval}).

A unifying theme across all chapters is the use of hypothesis testing as the bridge between statistical evidence and deployment decisions. P-values and e-values, introduced in Chapter~\ref{chapter:mht} and reviewed in depth in Appendix~\ref{app:evidence}, each provide finite-sample reliability guarantees with distinct practical trade-offs: P-values are well-suited to single-shot evaluation with pre-specified sample sizes, while e-values enable anytime-valid, adaptive, and compositional testing that is essential for the sequential and multi-stage settings treated in Chapters~\ref{chapter:adaptive} and~\ref{chapter:autoeval}.

\section{Outlook}
\label{sec:ch8_outlook}

The LTT framework and its variants have reached a level of maturity that supports deployment in real AI engineering pipelines. The core theoretical guarantees, i.e., control of FWER and FDR at finite sample sizes, without distributional assumptions beyond boundedness of the loss, are well established and have been validated in image classification (Sec.~\ref{sec:app_basic_ml}) and wireless scheduling (Sec.~\ref{sec:app_wireless}) applications.

The shift from p-value-based to e-value-based testing has proven particularly fruitful for adaptive and sequential settings. E-processes, introduced in Chapter~\ref{chapter:adaptive}, allow the analyst to interleave evaluation and selection, concentrate testing resources on the most promising candidates, and incorporate autoevaluated data without compromising the reliability guarantee. These properties make e-value-based methods the natural choice for large-scale or online deployment scenarios.

The prediction-powered inference (PPI) framework \citep{angelopoulos2023prediction,einbinder2024semi} underlying R-AutoEval and R-AutoEval+ provides a principled way to exploit cheap automated evaluations without sacrificing statistical validity. As automated evaluation tools, e.g., large language models used as judges \citep{zheng2023judging}, become ubiquitous in modern AI pipelines, the autoevaluation framework of Chapter~\ref{chapter:autoeval} provides a ready-made solution for incorporating them into statistically valid selection.

\section{Open Problems and Research Opportunities}
\label{sec:ch8_open}

\textbf{Continuous and structured hyperparameter spaces:} The LTT framework, in both its basic and extended forms, is designed for a finite candidate set $\Lambda = \{\lambda_1, \ldots, \lambda_K\}$ of hyperparameters. Extending the methodology to continuous or combinatorially large spaces, without requiring exhaustive evaluation of all candidates, is an important open problem. One direction is to combine LTT with Bayesian optimization or active learning to focus evaluation on a small promising subset \citep{paine2020hyperparameter}, but theoretical guarantees for this combination remain largely open.

\textbf{Distribution shift:} All the reliability guarantees developed in this monograph assume that the calibration data $\mathcal{D}$ are drawn from the same distribution as the deployment environment. In practice, distribution shift between calibration and deployment is common and may invalidate the guarantees. Extending the framework to handle covariate shift, label shift, or general non-stationarity is a natural and practically important direction, with potential connections to robust statistics and distributionally robust optimization.

\textbf{Multi-task and meta-learning settings:} In many practical applications, hyperparameter selection is performed repeatedly for related tasks, e.g., multiple users, environments, or problem instances. Sharing calibration data or statistical power across tasks, e.g., via meta-learning or hierarchical models, could substantially reduce the per-task sample complexity. Formalizing this in a way that preserves the finite-sample reliability guarantee is an open problem.

\textbf{Online and non-stationary environments:} The sequential and adaptive methods of Chapter~\ref{chapter:adaptive} allow the sample size to grow over time but still assume a stationary data-generating distribution. In non-stationary environments, where the optimal hyperparameter may change over time, a framework that tracks the best hyperparameter while maintaining reliability guarantees at every time step is needed. Connections to online learning and tracking \citep{audibert2009exploration} may be fruitful here.

\textbf{Integration with autoevaluation at scale:} The R-AutoEval framework of Chapter~\ref{chapter:autoeval} handles a single fixed autoevaluator. As AI systems increasingly use ensembles of automated evaluators, or evaluators that are themselves updated over time, generalizing R-AutoEval+ to the multi-autoevaluator and online-autoevaluator settings is an important open problem. The PPI++ framework \citep{angelopoulos2023ppi++} and its cross-prediction variant \citep{zrnic2024cross} offer relevant technical machinery.

\textbf{Interpretability and post-selection analysis:} The LTT framework certifies which hyperparameters are reliable, but does not explain why certain configurations are reliable or how to improve unreliable ones. Combining reliability certification with interpretability tools, such as feature attribution or sensitivity analysis, would provide more actionable guidance for practitioners and is an avenue largely unexplored in the current literature.


\chapter*{Acknowledgements}
\addcontentsline{toc}{chapter}{Acknowledgements}
The authors are grateful to Sangwoo Park and Matteo Zecchin for their collaboration on several of the results presented in this monograph. This work was supported by the European Research Council (ERC) under the European Union’s Horizon Europe Programme (grant agreement No.~101198347). The work of O.~Simeone was also supported by an EPSRC Open Fellowship (EP/W024101/1) and by the EPSRC project (EP/X011852/1).

\appendix

\chapter{How to Measure Evidence}
\label{app:evidence}

This appendix reviews in a self-contained fashion the
modern toolbox for quantifying evidence against a statistical null hypothesis.
We start from the classical notion of a p-value as a stochastically large
statistic under the null. We then discuss two
structural limitations of p-values, namely their lack of anytime validity
and the absence of a simple rule for combining independent evidence. To address these limitations, we
introduce e-values and e-processes as non-negative statistics with unit mean under the null.
Technical details on p-value constructions from concentration bounds and additional FWER- and FDR-controlling procedures are collected in Appendices~\ref{app:ch2_pvalue_constructions} and~\ref{app:additional_mht}.

\section{Evidence via Hypothesis Testing}
\label{app:sec:evidence}

We observe data $X_1, X_2, \ldots, X_n$ drawn from some
unknown distribution $\mathbb{P}$. A \emph{statistical hypothesis} is a claim
about the distribution $\mathbb{P}$. Formally, it corresponds to a partition of the set of all allowed
candidate distributions $\mathcal{P}$, which is possibly a strict subset of
all distributions, into two disjoint subsets as in
\begin{equation}
  \mathcal{H}_0: \mathcal{P}_0 \;\text{(null)} \qquad \text{and} \qquad
  \mathcal{H}_1: \mathcal{P}_1 = \mathcal{P} \setminus \mathcal{P}_0
  \;\text{(alternative)}.
\end{equation}

The null $\mathcal{H}_0$ typically represents the \emph{status quo}: a
new drug is not effective, an AI model is not reliable, a new software still
contains bugs, and so on. The alternative encodes the corresponding positive
claim. Rejecting the null in favor of the alternative means that the observed
data are sufficiently \emph{inconsistent} with every distribution
$\mathbb{P} \in \mathcal{P}_0$, lending support to the claim
described by $\mathcal{P}_1$.

The central questions addressed in this appendix are: Does
the data provide \emph{evidence against the null}? Can the amount of evidence
be \emph{quantified} so as to support principled decision making?

\subsection{Examples}
\label{app:sec:examples}

The following are useful and common examples.

\paragraph{Example 1: Testing a mean value:}
Let $X_1, X_2, \ldots \iid \mathbb{P}$ take values in $[0,1]$ with unknown
mean $\mu(\mathbb{P}) = \mathbb{E}[X]$. For fixed values
$\mu_0, \mu_1 \in (0,1)$, test
\begin{equation}
  \mathcal{H}_0: \mu(\mathbb{P}) = \mu_0 \qquad \text{vs.} \qquad
  \mathcal{H}_1: \mu(\mathbb{P}) = \mu_1.
\end{equation}
In a clinical context, the observation $X_i$ may represent a biomarker, with $\mu_0$ the
average value for a healthy patient and $\mu_1$ that of a sick patient. The
hypotheses are \emph{composite}: each contains many distributions, since only
the mean is fixed. \hfill $\square$

\paragraph{Example 2: Testing a zero mean:}
With $X_1, X_2, \ldots \iid \mathbb{P}$ on the real line, test
\begin{equation}
  \mathcal{H}_0: \mu(\mathbb{P}) = 0 \qquad \text{vs.} \qquad
  \mathcal{H}_1: \mu(\mathbb{P}) \neq 0.
\end{equation}
This test may be useful, for example, to assess whether augmenting a dataset with synthetic samples changes the average model performance, i.e., whether the mean difference in losses equals zero. Both hypotheses are composite, and
the set $\mathcal{P}$ encompasses all distributions with finite mean.

\paragraph{Example 3: Testing a bound on the mean:}
For $X_1, X_2, \ldots \iid \mathbb{P}$ in $[0,1]$ and tolerance
$\mu_{\max} \in (0,1)$, test
\begin{equation}
  \mathcal{H}_0: \mu(\mathbb{P}) \ge \mu_{\max} \;\text{(negative)} \qquad
  \text{vs.} \qquad
  \mathcal{H}_1: \mu(\mathbb{P}) < \mu_{\max} \;\text{(positive)}.
\end{equation}
Here the observation $X_i$ may represent a score or loss (smaller is better)
of a drug, AI model, or theory.

\paragraph{Example 4: Two-sample testing:}
Given $X_1, \ldots, X_n \iid \mathbb{P}_X$ independent of
$Y_1, \ldots, Y_m \iid \mathbb{P}_Y$, test
\begin{equation}
  \mathcal{H}_0: \mathbb{P}_X = \mathbb{P}_Y \qquad \text{vs.} \qquad
  \mathcal{H}_1: \mathbb{P}_X \neq \mathbb{P}_Y.
\end{equation}
A typical application contrasts measurements from nominal conditions ($X_i$)
with those from a test population ($Y_i$).

\paragraph{Example 5: Testing independence:}
For paired observations\\
$(X_1, Y_1), \ldots, (X_n, Y_n) \iid \mathbb{P}_{XY}$, test
\begin{equation}
  \mathcal{H}_0: \mathbb{P}_{XY} = \mathbb{P}_X \otimes \mathbb{P}_Y
  \qquad \text{vs.} \qquad
  \mathcal{H}_1: \mathbb{P}_{XY} \neq \mathbb{P}_X \otimes \mathbb{P}_Y,
\end{equation}
where observation $X_i$ might be a treatment assigned to patient $i$ and $Y_i$ a
downstream biomarker.

\subsection{Setting Up the Hypotheses}
\label{app:sec:null}

Setting a hypothesis test to evaluate a claim about the
data distribution $\mathbb{P}$ requires defining a null class $\mathcal{P}_0$
together with the alternative class $\mathcal{P}_1$, or equivalently the set
$\mathcal{P}$ of all allowed distributions. The null $\mathcal{H}_0: \mathbb{P} \in \mathcal{P}_0$ is a
\emph{modelling choice} that governs both what can be concluded from a
rejection and how hard it is to design a valid test. A broad null,
corresponding to a large set $\mathcal{P}_0$, may support a more tractable
test construction, but it typically requires more data to be rejected, since
the test must hedge against more null distributions $\mathbb{P}\in \mathcal{P}_0$. A more specific null,
with a small set $\mathcal{P}_0$, may make valid test construction harder, but
it allows the test to enhance its statistical power, i.e., its capacity to correctly reject the null under an alternative distribution $\mathbb{P}\in \mathcal{P}_1$.

\paragraph{Example 2 (continued): Testing a zero mean:}
Denote as $\distas$ equality in distribution. In place of the zero-mean null
$\mathcal{H}_0: \mu(\mathbb{P}) = 0$, one may test the more specific null
\begin{equation}
  \mathcal{H}'_0: X \distas -X,
\end{equation}
stipulating that the distribution $\mathbb{P}$ is symmetric about zero. Since
symmetry about zero implies a zero mean, the set $\mathcal{P}'_0$ covered by
the null $\mathcal{H}'_0$ is included in the original set $\mathcal{P}_0$,
i.e.,
\begin{equation}
  \mathcal{P}'_0 = \{\mathbb{P}: X \distas -X\} \;\subset\;
  \{\mathbb{P}: \mu(\mathbb{P}) = 0\} = \mathcal{P}_0.
\end{equation}
The price of this specialization is that rejecting the null $\mathcal{H}'_0$
does \emph{not} imply rejection of the original null $\mathcal{H}_0$, because
the data may be incompatible with symmetry while the mean is still zero. As we
will see, the advantage of this approach is that one can construct more
powerful tests for the null $\mathcal{H}'_0$. \hfill $\square$

\paragraph{Example 6: Testing exchangeability and anomaly detection:}
Conversely, testing an \emph{enlarged null} can leverage structural symmetries
or invariances to make test design more tractable. In this case, rejecting the
larger null also rejects the original null.

A canonical example of a structured larger null is \emph{exchangeability}. To
elaborate, denote as $S_n$ the set of all permutations of $n$ integers.

\begin{definition}
\label{app:def:exchangeability}
The random variables $X_1, \ldots, X_n$ are \emph{exchangeable} if for every
permutation $\pi \in S_n$ the joint distribution $\mathbb{P}$ is unchanged,
i.e.,
\begin{equation}
  (X_1, \ldots, X_n) \distas (X_{\pi(1)}, \ldots, X_{\pi(n)});
\end{equation}
equivalently, if the joint distribution $\mathbb{P}$ depends only on the
multiset of values $\{x_1, \ldots, x_n\}$.
\end{definition}

Sequences with i.i.d.\ samples are exchangeable, but the
class of exchangeable distributions is strictly larger, including equicorrelated
Gaussians, samples drawn without replacement from a finite population, and functions
of i.i.d.\ variables with shared randomness.

The exchangeability null,
\begin{equation}
\label{app:eq:exc}
  \mathcal{H}_0: X_1, \ldots, X_n \text{ exchangeable},
\end{equation}
can provide a tractable, enlarged hypothesis for many settings of interest.
To illustrate this point, consider the task of \emph{anomaly detection}. In it, given a calibration set
$X_1, \ldots, X_{n-1}$ and a new point $X_n$, we wish to determine if the \textit{test point} $X_n$ is an anomaly or not. To this end, we can test whether the sample $X_1, \ldots, X_n$, including calibration and test data is exchangeable.
\hfill $\square$

\paragraph{Example 4 (continued): Two-sample testing:}
For two-sample testing, define the pooled samples as
\begin{equation}
Z = (Z_1, \ldots, Z_{n+m}) = (X_1, \ldots, X_n, Y_1, \ldots, Y_m).
\end{equation}
Under the original null $\mathcal{H}_0: \mathbb{P}_X = \mathbb{P}_Y$, the pooled
sample is exchangeable, so that the larger null
\begin{equation}
  \mathcal{H}'_0: Z_1, \ldots, Z_{n+m} \text{ exchangeable}
\end{equation}
contains $\mathcal{H}_0$, and rejecting $\mathcal{H}'_0$ rejects
also $\mathcal{H}_0$. \hfill $\square$

\paragraph{Example 5 (continued): Testing independence:}
For independence testing, under the original null
$\mathcal{H}_0: \mathbb{P}_{XY} = \mathbb{P}_X \otimes \mathbb{P}_Y$ the
sequence $Y_1, \ldots, Y_n$ is exchangeable conditional on
$X_1, \ldots, X_n$, so the exchangeability null
\begin{equation}
  \mathcal{H}'_0: Y_1, \ldots, Y_n \text{ exchangeable given }
  X_1, \ldots, X_n
\end{equation}
contains the original null $\mathcal{H}_0$. \hfill $\square$

Beyond the null, the overall set $\mathcal{P}$ encodes the
distributions deemed plausible, with
$\mathcal{P}_1 = \mathcal{P} \setminus \mathcal{P}_0$ being the alternative set. A
broad set $\mathcal{P}$ ensures robustness with respect to possible alternative
ground-truth conditions. However, this comes at the cost of power, since the
test must hedge against many possible deviations from the null $\mathcal{H}_0$.
Conversely, a narrow set $\mathcal{P}$ delivers higher power, but it exposes
the analysis to misspecification, as the true distribution may lie outside
$\mathcal{P}$. The power of a test will be discussed further in Sec. \ref{app:sec:power}.

\subsection{Measures of Evidence}
\label{app:sec:measures}

To carry out a test, one needs a statistic computed from
the data that serves as a \emph{measure of evidence against the null}. Three
properties are desirable for such measures:
\begin{itemize}
  \item \textbf{Validity}: The statistic should take an extreme value (small or
  large, by convention) when the data would be surprising under the null $\mathcal{H}_0$. This property supports tests obtained by thresholding the statistic while allowing for the control of the probability of false alarm.
  \item \textbf{Anytime validity}: Evidence should remain ``honest'' regardless
  of when one decides to stop collecting data, whether before, after, or
  mid-study. In other words, the probability of false alarm should not be inflated by optional stopping mechanisms for data collection.
  \item \textbf{Compositional validity}: Evidence from independent sources
  should combine into a single, valid measure of evidence without additional
  assumptions: more experiments should mean more evidence.
\end{itemize}

This appendix focuses on two such measures: p-values,
also known sometimes as p-variables, and e-values, or e-variables. As
summarized in Table~\ref{app:tab:p_vs_e_values}, both statistics satisfy the first
property, namely validity at a fixed sample size, but only e-values automatically
satisfy the other two.

\begin{table}[t]
\centering
\renewcommand{\arraystretch}{1.3}
\resizebox{\columnwidth}{!}{%
\begin{tabular}{lcc}
\hline
\textbf{Property} & \textbf{P-value} & \textbf{E-value} \\\hline
Validity at fixed sample size & Yes & Yes \\
Anytime validity & No & Yes \\
Compositional validity (product, independent) & No & Yes \\
Compositional validity (average, any dependence) & No & Yes \\
\hline
\end{tabular}%
}
\caption{Comparison of key properties of P-values and E-values.}
\label{app:tab:p_vs_e_values}
\end{table}

\section{P-values}
\label{app:sec:pvalues}

\subsection{Definition}

A \emph{p-value} for the null $\mathcal{H}_0$ is a
statistic $p = p(X_1, \ldots, X_n) \in [0,1]$ that tends to be large when
the null $\mathcal{H}_0$ is true, so that small values of variable $p$ flag surprising
observations under the null $\mathcal{H}_0$. The resulting test takes the form
\begin{equation}
  \text{reject } \mathcal{H}_0 \text{ if } p \le \delta,
\end{equation}
for a prespecified threshold $\delta \in (0,1)$, chosen as a function of the target probability of false alarm.

Formally, the condition of being large under the null means
that $p$ is stochastically larger than a uniform random variable on $[0,1]$,
i.e., superuniform.

\begin{definition}[p-value]
\label{app:def:pvalue}
A \emph{p-value} for $\mathcal{H}_0$ is a statistic $p \in [0,1]$ such
that, for every distribution $\mathbb{P} \in \mathcal{P}_0$ and every $u \in (0,1)$, we have the inequality
\begin{equation}
  \mathbb{P}(p \le u) \;\le\; u.
\end{equation}
Equivalently, as illustrated in Fig.~\ref{app:fig:pvalue-cdf}, the
cumulative distribution function (CDF) of $p$ lies on or below the diagonal
of the unit square. As also illustrated in Fig.~\ref{app:fig:pvalue-cdf}, a p-value controls the probability of false alarm, i.e., the type-I error.
\end{definition}

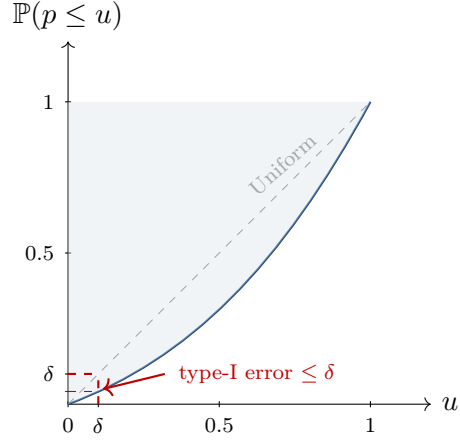
\begin{figure}[h]
\centering
\begin{tikzpicture}[scale=4,font=\small]
  \draw[->] (0,0)--(1.20,0) node[right] {$u$};
  \draw[->] (0,0)--(0,1.20) node[above] {$\mathbb{P}(p\le u)$};
  \foreach \x in {0,0.5,1}{
    \draw (\x,0.01)--(\x,-0.01) node[below,font=\scriptsize]{\x};
  }
  \foreach \y in {0.5,1}{
    \draw (0.01,\y)--(-0.01,\y) node[left,font=\scriptsize]{\y};
  }
  \draw[dashed,gray] (0,0)--(1,1);
  \node[gray,font=\scriptsize,rotate=42,anchor=south] at (0.75,0.75)
       {Uniform};
  \draw[myblue,thick] (0,0) .. controls (0.4,0.15) and (0.7,0.45) .. (1,1);
  \fill[myblue!15,opacity=0.4]
       (0,0) .. controls (0.4,0.15) and (0.7,0.45) .. (1,1)
       -- (0,1) -- cycle;
  \draw (0.10,0.01)--(0.10,-0.01)
       node[below,font=\scriptsize]{$\delta$};
  \draw (0.01,0.10)--(-0.01,0.10)
       node[left,font=\scriptsize]{$\delta$};
  \draw[myred,dashed,thick] (0.10,0) -- (0.10,0.10);
  \draw[myred,dashed,thick] (0,0.10) -- (0.10,0.10);
  \draw[myred,dashed] (0,0.042) -- (0.10,0.042);
  \draw[myred] (0.01,0.042)--(-0.01,0.042);
  \draw[->,myred,thick] (0.32,0.10) -- (0.115,0.05);
  \node[myred,font=\scriptsize,anchor=west] at (0.33,0.10)
       {type-I error $\le \delta$};
\end{tikzpicture}
\caption{By definition, the CDF of a p-value lies on or below the uniform
diagonal under any null distribution $\mathbb{P} \in \mathcal{P}_0$. The value
of the CDF at $u = \delta$ coincides with the type-I error of the test
``reject $\mathcal{H}_0$ if $p \le \delta$'', and is no larger than $\delta$
by construction.}
\label{app:fig:pvalue-cdf}
\end{figure}

\begin{proposition}[type-I error control]
\label{app:prop:typeI}
The test ``reject $\mathcal{H}_0$ if $p \le \delta$'' controls the probability
of false alarm at level $\delta$, i.e., for every distribution $\mathbb{P} \in \mathcal{P}_0$, we have the inequality
\begin{equation}
  \mathbb{P}(\text{reject } \mathcal{H}_0) \;\le\; \delta.
\end{equation}
\end{proposition}

\begin{proof}
By the definition of a p-value, we have
$\mathbb{P}(\text{reject } \mathcal{H}_0) = \mathbb{P}(p \le \delta) \le
\delta$.
\end{proof}

A p-value whose CDF is strictly below the diagonal is
\emph{conservative}, in the sense that it controls the type-I error at a value
smaller than $\delta$, which typically entails a \emph{lower power} under the
alternative $\mathcal{H}_1$ (see Sec. \ref{app:sec:power}). Accordingly, the ideal case is a uniformly
distributed p-value under the null.

As a final remark, based on Proposition~\ref{app:prop:typeI}, it is tempting
to read smaller p-values as stronger evidence against the null $\mathcal{H}_0$.
The popular argument runs: ``$p$ is the smallest significance level $\delta$ at
which, had I pre-specified it, I would have rejected $\mathcal{H}_0$''. It
should be noted, however, that this is only a \emph{counterfactual} statement.
In fact, setting $\delta = p$ \emph{after} observing the data does not provide
a valid type-I error guarantee, and is not the same as a quantitative statement
about factual evidence. That said, in practice, the p-value is often used as
an informal measure of evidence, whose usefulness in that role depends on the
test's \emph{power} against the alternative of interest (see Sec. \ref{app:sec:power}).

\subsection{Rank P-value}
\label{app:sec:rank}

The null $(\mathcal{H}_0: X_1, \ldots, X_n
\text{ exchangeable})$ in \eqref{app:eq:exc} admits a particularly simple p-value, which serves as a useful first example to illustrate the
definition of p-values. Given real-valued scalar data
$X_1, \ldots, X_n \in \mathbb{R}$, the rank statistic at the $i$-th index is
the number of values $X_j$ that are no larger than $X_i$, i.e.,
\begin{equation}
  \operatorname{rank}(X_i; X_1, \ldots, X_n) = |\{j : X_j \le X_i\}|,
\end{equation}
where $|\cdot|$ denotes the set cardinality.

\begin{proposition}
\label{app:prop:uniformrank}
Under the exchangeability null $\mathcal{H}_0$ and assuming no ties almost surely, for any fixed index
$i \in \{1, \ldots, n\}$, the rank statistic is uniformly distributed in the set $\{1,\ldots, n\}$, i.e.,
\begin{equation}
  \operatorname{rank}(X_i; X_1, \ldots, X_n) \sim \Unif\{1, \ldots, n\}.
\end{equation}
\end{proposition}

\begin{proof}
All $n!$ orderings of $X_1, \ldots, X_n$ are equally likely, so the variable
$X_i$ is equally likely to occupy each of the $n$ positions.
\end{proof}

By Proposition \ref{app:prop:uniformrank}, the distribution of the rank statistic is uniform within the discrete set $\{1, 2, \ldots, n\}$, and hence normalizing by $n$, we obtain a distribution function in the set $\{1/n, 2/n, \ldots, 1\}$. The corresponding CDF is upper bounded by the diagonal in Fig. \ref{app:fig:pvalue-cdf}, implying the following.

\begin{definition}[Rank p-value]
\label{app:def:rankpvalue}
For real-valued scalar data $X_1, \ldots, X_n \in \mathbb{R}$, the statistic
\begin{equation}
\label{app:eq:rankp}
  p \;=\; \frac{\operatorname{rank}(X_n; X_1, \ldots, X_n)}{n}
\end{equation}
is a p-value for the exchangeability null.
\end{definition}

Note that the rank p-value can be proved to remain valid
when ties are allowed \citep{dufour2006monte}. It is also emphasized that the reference index $n$ in \eqref{app:eq:rankp}
is arbitrary, and it may also be chosen at random from the set
$\{1, \ldots, n\}$.

\paragraph{Example 6 (continued): Testing exchangeability and anomaly detection:}
Consider the calibration set $\{X_1, \ldots, X_{n-1}\}$ given by
\begin{equation}
\label{eq:example_calibration_set}
\{-1.21,\,-0.95,\,-0.62,\,-0.37,\,0.17,\,0.43,\,0.88,\,1.02,\,1.55\}
\end{equation}
together with a test point $X_{10}$. The rank p-value
$p_{10} = |\{j : X_j \le X_{10}\}|/10$ measures the fraction of pooled
samples not exceeding the test point $X_{10}$. Set the target type-I error to $\delta = 0.1$. For the test point $X_{10} = 0.50$, we
have $p_{10} = 7/10 = 0.70 > \delta$ and the point is not flagged as an
anomaly. In contrast, for the test point $X_{10} = -3.80$, we have
$p_{10} = 1/10 = 0.1 \le \delta$ and the point is flagged. \hfill $\square$

\subsection{Power}
\label{app:sec:power}

For a given null $\mathcal{H}_0$, many statistics give
rise to valid p-values. Among all such statistics, one would ideally wish to
choose one that maximizes \emph{power} against the alternative class
$\mathcal{H}_1$. Power refers to the probability of correctly rejecting the
null under the alternative, i.e., to
$\mathbb{P}_{\mathbb{P}}(p(X^n) \le \delta)$ for distributions
$\mathbb{P} \in \mathcal{P}_1$. Considering the worst case among all
alternative distributions $\mathbb{P} \in \mathcal{P}_1$ yields the objective
\begin{equation}
\label{app:eq:power}
 \min_{\mathbb{P} \in \mathcal{P}_1}
  \mathbb{P}_{\mathbb{P}}\!\bigl(p(X^n) \le \delta\bigr).
\end{equation}

\paragraph{Example 6 (continued): Testing exchangeability and anomaly detection:}
Consider an anomaly detection problem in which the test sample $X_n$ is
anomalous if it is too large as compared to typical values. In this setting,
the rank p-value $p = |\{j : X_j \le X_n\}|/n$ discussed above would tend
to $1$ under the alternative. Therefore, while valid, this p-value is powerless against
such anomalies, i.e., against such alternatives. Can we devise a more powerful
rank p-value?

Negating the data gives an equally valid rank p-value
under the exchangeability null, i.e.,
\begin{equation}
  p' = \frac{|\{j : X_j \ge X_n\}|}{n}.
\end{equation}
Importantly, this p-value tends to be small precisely when the test sample
$X_n$ is unusually large. For example, for the calibration set \eqref{eq:example_calibration_set},
with the test point $X_{10} = 3.80$, we set the p-value $p' = 1/10$, correctly flagging
the outlier, whereas the original rank p-value, $p= 1$, would not.

\subsection{Summary Statistics}

As seen for the exchangeability null, the construction of
a p-value typically starts with the design of a suitable summary statistic as
a function of the data $X^n$, such as
$\operatorname{rank}(X_n; X_1, \ldots, X_n)$ used in Example 6. As a general
rule, the summary statistic $T(X^n)$ should be designed to be small under the
null $\mathcal{H}_0$ and large under $\mathcal{H}_1$.

\paragraph{Example 7: Simple hypotheses:}
In the case of a simple null $\mathcal{P}_0 = \{\mathbb{P}_0\}$, the natural
choice is the information-theoretic surprise (see, e.g.,
\citep{simeone2026classical})
\begin{equation}
  T(X^n) = -\log \mathbb{P}_0(X_1, \ldots, X_n).
\end{equation}
If the alternative is also simple, i.e., $\mathcal{P}_1 = \{\mathbb{P}_1\}$,
the log-likelihood ratio
\begin{equation}
\label{app:eq:LLR}
  T(X^n) = \log\!\left(\frac{\mathbb{P}_1(X^n)}{\mathbb{P}_0(X^n)}\right)
\end{equation}
is the preferred choice, as it can be shown to be optimal in the sense of
optimizing the power \eqref{app:eq:power}. \hfill $\square$

\paragraph{Example 3 (continued): Testing a bound on the mean:}
For the null $\mathcal{H}_0: \mu(\mathbb{P}) \ge \mu_{\max}$, a reasonable
summary statistic is the mean-deficit
\begin{equation}
  T(X^n) = \mu_{\max} - \frac{1}{n}\sum_{i=1}^n X_i.
\end{equation}
In fact, when the empirical mean $\sum_{i=1}^n X_i/n$ is close to the true mean
$\mu(\mathbb{P})$, this statistic tends to be small under the null, as
$T(\mu(\mathbb{P})) = \mu_{\max} - \mu(\mathbb{P}) \le 0$ under $\mathcal{H}_0$,
and large under the alternative, as
$T(\mu) = \mu_{\max} - \mu(\mathbb{P}) > 0$ under $\mathcal{H}_1$.
\hfill $\square$

\paragraph{Example 6 (continued): Testing exchangeability:}
The rank p-values are only applicable when the observed data are real scalars
and when the alternative prescribes unusually small or large values. More
generally, the exchangeability null may need to be tested for an arbitrary data
type. In this case, one can start with a summary statistic $T(X^n)$ designed to
measure how unlikely data $X^n$ are under the null. As discussed next, any
summary statistic can be turned into a valid p-value. To elaborate, for any
permutation $\pi \in S_n$, let
$X^n_\pi = (X_{\pi(1)}, \ldots, X_{\pi(n)})$ denote the permuted sequence.

\begin{proposition}
\label{app:prop:permexch}
If the random vector $X^n$ is exchangeable, then the $n!$ transformed random
variables $\{T(X^n_\pi)\}_{\pi \in S_n}$ are exchangeable.
\end{proposition}

\begin{proof}[Proof sketch]
A transformation preserves exchangeability if, for each input permutation,
there exists an output permutation leaving the function unchanged. See
\citep{ritzwoller2024randomization} for details.
\end{proof}

Applying the rank p-value to these $n!$ exchangeable
values gives the \emph{permutation p-value}
\begin{equation}
\label{app:eq:permpval}
  p \;=\; \frac{\left|\pi \in S_n : T(X^n_\pi) \ge T(X^n)\right|}{n!},
\end{equation}
where the inequality reflects the design assumption that large values of the summary statistic $T(X^n)$
are surprising under $\mathcal{H}_0$.

Since computing the summary statistic $T(X^n)$ on all
$n!$ permutations is usually infeasible, one can practically either leverage
symmetries or apply Monte Carlo sampling
\citep{ritzwoller2024randomization}. \hfill $\square$

\paragraph{Example 5 (continued): Testing independence:}
For paired observations $(X, Y)$, a dependence statistic $T(X, Y)$, such as a mutual-information estimator, can be
leveraged to yield a permutation p-value for the independence null upon
permutation of the $Y$ values \citep{ritzwoller2024randomization}. \hfill $\square$

The set of permutations $S_n$ is a \emph{group}: it
contains the identity, inverses, and is closed under composition. The
permutation p-value generalizes to nulls defined by an arbitrary
group-invariance property, yielding randomization p-values \citep{ritzwoller2024randomization}.

\subsection{Constructing P-values Using the Probability Integral Transform}

Once a summary statistic is selected, the main challenge
is to turn it into a valid p-value. To this end, the following result plays
a key role.

\begin{theorem}[Probability integral transform]
\label{app:thm:pit}
If the random variable $T$ is continuous with complementary CDF (CCDF)
$\bar{F}(t) = \mathbb{P}(T \ge t)$, then the random variable $U = \bar{F}(T)$
is uniformly distributed in the interval $(0,1)$, i.e., $\Unif(0,1)$.
\end{theorem}

\begin{proof}
For any $u \in (0,1)$, since the CCDF $\bar{F}(t)$ is continuous and strictly
decreasing, we have
\begin{equation}
  \mathbb{P}(U \le u) = \mathbb{P}\!\bigl(\bar{F}(X) \le u\bigr)
            = \mathbb{P}\!\bigl(X \ge \bar{F}^{-1}(u)\bigr)
            = \bar{F}\!\bigl(\bar{F}^{-1}(u)\bigr) = u. \qedhere
\end{equation}
\end{proof}

The transformation $\bar{F}(T)$ for a random variable $T$
with CCDF $\bar{F}(t)$ is known as the \emph{probability integral transform}
(PIT). Note that the PIT is also often defined in terms of the CDF of a random variabled (see, e.g., \citep{angus1994probability}), and that the formulation in Theorem \ref{app:thm:pit} is equivalent for continuous variables (see, e.g., \citep[Appendix~A.1]{qi2024toward}). The PIT can be readily applied to obtain a p-value.

\begin{lemma}[p-value from PIT]
\label{app:def:PITpvalue}
Denote as $\bar{F}_{\mathbb{P}}(t)$ the CCDF of the test statistic $T(X^n)$
under distribution $\mathbb{P}$. Then, the quantity
\begin{equation}
\label{app:eq:PITp}
  p(X^n) \;=\; \max_{\mathbb{P} \in \mathcal{P}_0}
  \bar{F}_{\mathbb{P}}(T(X^n))
\end{equation}
is a p-value for the null $\mathcal{H}_0: \mathbb{P} \in \mathcal{P}_0$.
\end{lemma}

Note that the definition \eqref{app:eq:PITp} contains a
maximum over all the distributions in set $\mathcal{P}_0$ (or a supremum if
the set is not closed) in order to account for the worst case among all the
null distributions.

To understand the significance of the p-value
\eqref{app:eq:PITp}, it is useful to rewrite it in an equivalent way as
follows. Define as $\tilde{X}^n$ a fictitious dataset drawn from a distribution
$\mathbb{P}$. Fix the data $X^n$, and imagine generating $\tilde{X}^n$
independently of the fixed realization $X^n$. As it can be directly checked,
we have the equality
\begin{equation}
  \bar{F}_{\mathbb{P}}(X^n) \;=\;
  \mathbb{P}_{\tilde{X}^n \sim \mathbb{P}}\!\bigl(T(\tilde{X}^n) \ge T(X^n)\bigr),
\end{equation}
and thus the p-value \eqref{app:eq:PITp} can also be expressed as
\begin{equation}
\label{app:eq:pprob}
  p(X^n) \;=\; \max_{\mathbb{P} \in \mathcal{P}_0}
  \mathbb{P}_{\tilde{X}^n \sim \mathbb{P}}\!\bigl(T(\tilde{X}^n) \ge T(X^n)\bigr).
\end{equation}

For a well-designed test statistic $T(\cdot)$, under an
alternative distribution $\mathbb{P} \in \mathcal{P}_1$, the random variable
$T(X^n)$ with data $X^n \sim \mathbb{P}$ should be large. Accordingly, the probability in \eqref{app:eq:pprob} tends
to take small values under the alternative. This is precisely what is required
to maximize the power.

\subsection{Constructing P-values Using Concentration Bounds}

In practice, the CCDF of a test statistic may not be
known and, as a result, the PIT-based construction \eqref{app:eq:PITp} of a p-value is not directly applicable. However, it is
often possible to identify an upper bound $\tilde{p}(X^n) \ge p(X^n)$ on the p-value \eqref{app:eq:PITp}. Any
such upper bound is also a valid p-value. In fact, if $\tilde{p}\geq p$, the inequality
$\tilde{p} \le u$ implies $p \le u$, and hence we have the defining condition
\begin{equation}
  \mathbb{P}(\tilde{p}(X^n) \le u)
  \le \mathbb{P}(p(X^n) \le u) \le u,
\end{equation}
for all $\mathbb{P}\in \mathcal{H}_0$, which coincides with the superuniformity of the variable $\tilde{p}(X^n)$
under the null. The bound $\tilde{p}(X^n)$ is generally conservative, yielding
a CDF further below the diagonal than the exact p-value, but it is always
valid.

\paragraph{Example 3 (continued): Testing a bound on the mean:}
For observations $X_i \in [0,1]$ i.i.d.\ and the null
$\mathcal{H}_0: \mu(\mathbb{P}) \ge \mu_{\max}$, consider the test statistic
$T(X^n) = \mu_{\max} - \sum_{i=1}^n X_i/n$, where $\sum_{i=1}^n X_i/n$ is
the empirical mean. The maximum in \eqref{app:eq:PITp} is attained at the
boundary $\mu (\mathbb{P}) = \mu_{\max}$, since a larger mean increases the empirical mean
$\bar{X}_n$, making a large statistic $T(X^n)$ less likely. Applying
Hoeffding's inequality \citep{boucheron2013concentration} to the p-value \eqref{app:eq:PITp} produces the p-value
\begin{equation}
  p(X^n) \le \tilde{p}(X^n)
  = \exp\!\left(-2n\left(\frac{1}{n}\sum_{i=1}^n X_i
  - \mu_{\max}\right)^2\right).
\end{equation}

More discussion on p-values obtained from concentration bounds can be found in Appendix~\ref{app:ch2_pvalue_constructions}.

\section{Limitations of P-values}
\label{app:sec:limitations}

Elaborating on Table~\ref{app:tab:p_vs_e_values}, we now
discuss two key limitations of p-values, which motivate the introduction of
e-values.

\subsection{No Anytime Validity}
\label{app:sec:noanytime}

The defining inequality $\mathbb{P}(p_n \le u) \le u$ is a
guarantee at a \emph{pre-specified} sample size $n$. If, instead, the analyst
inspects the data as it arrives and stops at a random time $N$ chosen
adaptively, then in general the inequality
\begin{equation}
  \mathbb{P}(p_N \le u) \gg u
\end{equation}
holds. This inequality can be interpreted as the phenomenon that ``peeking'' inflates the
type-I error. A classical example is the ``stop and reject as soon as
$p_n \le \delta$'' strategy. Under the null $\mathcal{H}_0$, the running p-value will
eventually fall below $\delta$ by chance, yielding a false rejection with
probability tending to one. This optional-stopping pathology has been singled
out as a contributor to the reproducibility crisis
\citep{howard2021time,ramdas2023gametheoretic}.

\subsection{No Compositional Validity}
\label{app:sec:nocombine}

A second limitation concerns combining independent
evidence. If two independent studies report p-values $p^{(1)}$ and $p^{(2)}$
for the same null, neither the product $p^{(1)} p^{(2)}$ nor the average is a
valid p-value.

\section{E-values}
\label{app:sec:evalues}

The structural shortcomings of p-values reviewed in the
previous section motivate the introduction of a stronger measure of evidence.
As discussed next, e-values can be derived as a generalization of likelihood
ratios.

\subsection{Definition}
\label{app:sec:edef}

As seen in the previous section, when the hypotheses are
both simple, i.e., $\mathcal{P}_0 = \{\mathbb{P}_0\}$ and
$\mathcal{P}_1 = \{\mathbb{P}_1\}$, the likelihood ratio
\begin{equation}
  T(X^n) \;=\; \frac{\mathbb{P}_1(X^n)}{\mathbb{P}_0(X^n)},
\end{equation}
or equivalently its logarithm \eqref{app:eq:LLR}, offer the most powerful
test statistic at any type-I level by the Neyman--Pearson lemma \citep{lehmann2005testing}. We now observe
that the likelihood ratio has a key property that supports generalization to a
new class of evidence measures.

\begin{lemma}
Under the null $\mathcal{H}_0$, the likelihood ratio has mean equal to one,
i.e.,
\begin{equation}
  \mathbb{E}_{\mathbb{P}_0}[T(X^n)] = 1.
\end{equation}
\end{lemma}

\begin{proof}
The proof is by direct calculation:
\begin{equation}
  \mathbb{E}_{\mathbb{P}_0}[T(X^n)]
  = \int \frac{\mathbb{P}_1(x)}{\mathbb{P}_0(x)} \mathbb{P}_0(x)\,dx = 1.
\end{equation}
\end{proof}

The e-value generalizes the likelihood ratio by
retaining and generalizing this property as follows.

\begin{definition}[e-value]
\label{app:def:evalue}
An \emph{e-value} for the null $\mathcal{H}_0$ is a non-negative random
variable $E$ such that its mean under any of the null distributions does not
exceed 1, i.e.,
\begin{equation}
  \mathbb{E}_{\mathbb{P}}[E] \;\le\; 1 \qquad \text{for every }
  \mathbb{P} \in \mathcal{P}_0.
\end{equation}
\end{definition}

According to this definition, large e-values indicate
evidence against the null. The next proposition shows how an e-value supports
testing by thresholding.

\begin{proposition}
\label{app:prop:emarkov}
The test ``reject $\mathcal{H}_0$ when $E \ge 1/\delta$'' controls the
probability of false alarm at level $\delta$, i.e., for every distribution
$\mathbb{P} \in \mathcal{P}_0$, we have the inequality
\begin{equation}
  \mathbb{P}(\text{reject } \mathcal{H}_0) \;\le\; \delta.
\end{equation}
\end{proposition}

\begin{proof}
By Markov's inequality, we have
$\mathbb{P}(E \ge 1/\delta) \le \delta\, \mathbb{E}_{\mathbb{P}}[E] \le
\delta$.
\end{proof}

\paragraph{Example 6 (continued): Testing exchangeability:}
For non-negative observations $X_1, \ldots, X_n \ge 0$, the statistic
\begin{equation}
\label{app:eq:evalue}
  E \;=\; \frac{X_n}{\frac{1}{n}\sum_{i=1}^n X_i}
\end{equation}
is an e-value for the exchangeability null, which is known as the
\emph{soft-rank e-value} \citep{ramdas2025hypothesis}. Indeed, under exchangeability,
the ratio $X_j/\sum_{i=1}^n X_i$ has the same distribution for every $j$,
so we have the equalities
\begin{equation}
  n\,\mathbb{E}_{\mathbb{P}}\!\left[\frac{X_n}{\sum_{i=1}^n X_i}\right]
  = \mathbb{E}_{\mathbb{P}}\!\left[\frac{\sum_{j=1}^n X_j}{\sum_{i=1}^n X_i}\right] = 1,
\end{equation}
implying the desired equality $\mathbb{E}_{\mathbb{P}}[E] = 1$ for all $\mathbb{P}\in \mathcal{P}_0$.
\hfill $\square$

\subsection{P-values from E-values}
\label{app:sec:evsP}

Unlike a p-value, by its definition, an e-value
constrains only its mean and not its entire CDF. As we discuss next, this
yields a more robust, but generally less powerful, test statistic.

\begin{proposition}[From e-values to p-values]
For any e-value $E$ for a null $\mathcal{H}_0$, the statistic
$p = \min(1, 1/E)$ is a p-value for $\mathcal{H}_0$.
\end{proposition}

\begin{proof}
Markov's inequality gives
\begin{equation}
  \mathbb{P}_{\mathcal{H}_0}\!\left(\frac{1}{E} \le \delta\right)
  = \mathbb{P}_{\mathcal{H}_0}\!\left(E \ge \frac{1}{\delta}\right)
  \le \delta\, \mathbb{E}_{\mathcal{H}_0}[E] \le \delta,
\end{equation}
concluding the proof.
\end{proof}

The converse to the proposition, however, fails. In
particular, if $p$ is a p-value for $\mathcal{H}_0$ with
$p \sim \Unif(0,1)$ under $\mathcal{H}_0$, then the average yields
$\mathbb{E}_{\mathcal{H}_0}[1/p] = \int_0^1 u^{-1}\,du = \infty$, so $1/p$
is not an e-value \citep{vovk2021evalues,ramdas2025hypothesis}. Overall, as shown in
Fig.~\ref{app:fig:evsp}, the inverses of e-values form a strict subset of
the family of p-values. As discussed in the next section, e-values have
special, additional, properties as compared to p-values.

\begin{figure}[h]
\centering
\begin{tikzpicture}[font=\small]
  \fill[myblue!12] (0,0) ellipse (3.2 and 2.0);
  \draw[myblue, thick] (0,0) ellipse (3.2 and 2.0);
  \node[myblue, font=\small\bfseries] at (0, 1.75) {p-values};
  \fill[myred!18] (-0.55,-0.05) ellipse (1.85 and 1.15);
  \draw[myred, thick] (-0.55,-0.05) ellipse (1.85 and 1.15);
  \node[myred, font=\scriptsize\bfseries, align=center] at (-0.55,-0.05)
       {$\{1/E : E\text{ is an }e\text{-value}\}$};
\end{tikzpicture}
\caption{The inverses of e-values are valid p-values, but the converse
fails. P-values obtained as the inverse of an e-value have special
properties, namely post-hoc validity, compositional validity, and anytime
validity.}
\label{app:fig:evsp}
\end{figure}
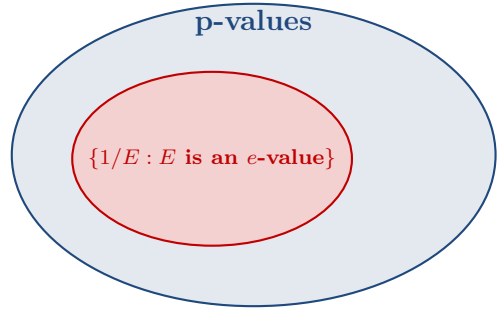

\paragraph{Example 6 (continued): Testing exchangeability:}
The additional properties of e-values typically come with a loss of power.
To see this, for the exchangeability null, consider the soft-rank p-value against the rank p-value, i.e.,
\begin{equation}
  p_E = \frac{1}{E} = \frac{\sum_{i=1}^n X_i}{n X_n},
  \qquad
  p_{\mathrm{rank}} = \frac{|\{j : X_j \ge X_n\}|}{n}.
\end{equation}
One can readily check that we have the inequality $p_E \ge p_{\mathrm{rank}}$
pointwise. In fact, since we have $X_i/X_n \ge 1$ whenever $X_i \ge X_n$, we can write
\begin{equation}
\label{app:eq:softrank-vs-rank}
  p_E = \frac{1}{n}\sum_{i=1}^n \frac{X_i}{X_n}
  \ge \frac{1}{n}\sum_{i: X_i \ge X_n} \frac{X_i}{X_n}
  \ge \frac{|\{i : X_i \ge X_n\}|}{n} = p_{\mathrm{rank}}.
\end{equation}
For a numerical example, with $n = 4$ and $X^n = (0.5, 1.0, 0.8, 2.5)$, the
soft-rank-based p-value equals $p_E = 4.8/10 = 0.48$, while the rank
p-value is $p_{\mathrm{rank}} = 1/4 = 0.25$. The compositional benefits of
e-values, developed below, often justify the loss in power. \hfill $\square$

\section{Properties of E-values}
\label{app:sec:eproperties}

In this section, we review the special properties of
e-values, which set them apart from more general p-values.

\subsection{Post-Hoc Validity}
\label{app:sec:posthoc}

A common practice in applied work is to choose the
significance level only after looking at the test statistic, for instance by
reporting whichever of the conventional levels $0.05$, $0.01$, or $0.001$ is
just exceeded. This kind of post-hoc choice is not valid for p-values, since
the standard guarantee $\mathbb{P}_{\mathcal{H}_0}(p \le \delta) \le \delta$
only protects levels $\delta$ that are fixed before the data are observed.
As discussed in Sec.~\ref{app:sec:noanytime}, the p-value reading ``I would have rejected at level $p$, had I pre-specified
it'' is therefore counterfactual. The corresponding e-value reading ``I reject
at level $1/E$'' is, by contrast, a factual claim, as made precise in the rest of this
subsection.

Suppose that the analyst chooses the rejection level as a
function $\delta(X^n)$ of the data $X^n = (X_1, \ldots, X_n)$. The strongest
guarantee one could hope for is that, conditional on the chosen level, the
probability of false rejection still does not exceed it, i.e.,
\begin{equation}
\label{app:eq:posthoc_ideal}
\mathbb{P}_{X^n \mid \delta(X^n)}\!\left[\text{reject } \mathcal{H}_0\right]
\;\le\; \delta(X^n) \qquad \text{for every } \mathbb{P} \in \mathcal{P}_0,
\end{equation}
where the inner probability is computed under the conditional distribution
$\mathbb{P}(X^n \mid \delta(X^n))$ of the data $X^n$ given the realized level
$\delta(X^n)$, with $\mathbb{P}$ ranging over the null class $\mathcal{P}_0$. Dividing
through by the data-dependent level $\delta(X^n)$, the bound \eqref{app:eq:posthoc_ideal} is equivalent
to the ratio inequality
\begin{equation}
\label{app:eq:posthoc_ratio}
\frac{\mathbb{P}_{X^n \mid \delta(X^n)}\!\left[\text{reject } \mathcal{H}_0\right]}
{\delta(X^n)} \;\le\; 1.
\end{equation}

As shown in \citep{koning2024continuous}, the strong
conditional guarantee \eqref{app:eq:posthoc_ratio} is generally unattainable. A natural
relaxation is to require the ratio to be bounded by one only \emph{on average}
over the random level $\delta(X^n)$. As the following theorem shows, the
e-value rejection rule meets exactly this relaxed requirement.

\begin{theorem}[Post-hoc validity \citep{ramdas2023gametheoretic,koning2024continuous}]
\label{app:thm:posthoc}
Let $E$ be an e-value for the null $\mathcal{H}_0$. For any positive, possibly
data-dependent, level $\delta(X^n)$ and any null distribution
$\mathbb{P} \in \mathcal{P}_0$, we have
\begin{equation}
\label{app:eq:kon}
\mathbb{E}_{\delta(X^n)}\!\left[
\frac{\mathbb{P}_{X^n \mid \delta(X^n)}\!\left[\text{reject } \mathcal{H}_0\right]}
{\delta(X^n)}
\right] \;\le\; 1,
\end{equation}
where the outer expectation is taken under the marginal distribution
$\mathbb{P}(\delta(X^n))$ of the data-dependent level induced by distribution $\mathbb{P}$.
\end{theorem}

\begin{proof}
The rejection event of the e-value test is $\{E \ge 1/\delta(X^n)\}$. By the
tower property of conditional expectation, the left-hand side of
\eqref{app:eq:kon} equals the marginal expectation
\begin{equation}
\mathbb{E}\!\left[
\frac{\mathbb{I}\{E \ge 1/\delta(X^n)\}}{\delta(X^n)}
\right].
\end{equation}
On the rejection event, the indicator equals one, so the fraction equals
$1/\delta(X^n)$, which by definition of the event is at most $E$; off the
rejection event, the fraction equals zero. The pointwise bound
\begin{equation}
\label{app:eq:posthoc_pointwise}
\frac{\mathbb{I}\{E \ge 1/\delta(X^n)\}}{\delta(X^n)} \;\le\; E
\end{equation}
thus holds for every realization. Taking expectations under $\mathbb{P}$ and
using the e-value definition $\mathbb{E}[E] \le 1$ gives \eqref{app:eq:kon}.
\end{proof}

The bound \eqref{app:eq:kon} substantiates the claim that, when using e-values, the analyst is free to inspect the
calibration data, decide when to stop, and pick the rejection level after the
fact. In fact, the expected ratio between the conditional type-I error and the chosen
level is still bounded by one. The criterion \eqref{app:eq:kon} was further shown in \citep{koning2024continuous} to serve as a
natural measure of post-hoc validity within an axiomatic formulation.

\subsection{Compositional Validity}
\label{app:sec:composition}

While p-values cannot be generally combined, e-values
combine in two simple ways. For \emph{independent} e-values $E^{(1)}$ and
$E^{(2)}$ for the same null, the product $E^{(1)} \cdot E^{(2)}$ is again an
e-value, since we have
$\mathbb{E}[E^{(1)} E^{(2)}] =
\mathbb{E}[E^{(1)}]\,\mathbb{E}[E^{(2)}] \le 1$.

For \emph{arbitrarily dependent} e-values $E^{(1)}, \ldots, E^{(K)}$, the
average
\begin{equation}
  \tfrac{1}{K}\bigl(E^{(1)} + \cdots + E^{(K)}\bigr)
\end{equation}
is always an e-value by linearity of expectation.

\paragraph{Example:}
Suppose labs A and B compute e-values $E_A = 7.3$ and $E_B = 8.7$. At level
$\delta = 0.05$ neither value exceeds the threshold $1/\delta = 20$. Their
product, however, does, i.e., $E_A \cdot E_B \approx 63.5 \ge 20$. Thus, the
null $\mathcal{H}_0$ is rejected at $\delta = 0.05$ using no distributional
assumption beyond independence. \hfill $\square$

\subsection{Anytime Validity and E-processes}
\label{app:sec:anytime}

As discussed in Sec.~\ref{app:sec:composition}, a sequence of e-values for the same null can be
multiplied to give a sequential measure of evidence. To elaborate, assume i.i.d.
observations $X_1, X_2, \ldots$ with $E_t$ being an e-value
computed from $X_t$. The running product
\begin{equation}
  W_n \;=\; \prod_{t=1}^n E_t, \qquad \text{with} \; W_0 = 1,
\end{equation}
is an example of an e-process, which can be used for sequential testing,
resolving the ``peeking'' problem of p-values highlighted in
Sec.~\ref{app:sec:limitations}. Before formalizing this property, we define
e-processes more generally.

\begin{definition}[e-process]
\label{app:def:eprocess}
A non-negative process obtained as the product $W_n = \prod_{t=1}^n E_t$,
with $W_0 = 1$, is an \emph{e-process} for the null class $\mathcal{P}_0$
if for every $t \ge 1$, the statistic $E_t \ge 0$ depends only on the past
observations $X_1, \ldots, X_{t-1}$, and for every $\mathbb{P} \in
\mathcal{P}_0$ and every past sequence $X_1 = x_1, \ldots, X_{t-1} =
x_{t-1}$, we have
\begin{equation}
\label{app:eq:eprocess}
  \mathbb{E}_{\mathbb{P}}\!\bigl[E_t \mid X_1 = x_1, \ldots,
  X_{t-1} = x_{t-1}\bigr] \;\le\; 1.
\end{equation}
\end{definition}
The condition \eqref{app:eq:eprocess} generalizes the
requirement that the variables $E_t$ be independent e-values, allowing
dependence among the variables.

The next result, a sequential analog of
Proposition~\ref{app:prop:emarkov}, shows the anytime validity of tests
obtained from e-processes.

\begin{theorem}[Anytime validity, Ville's inequality]
\label{app:thm:ville}
Let $\{W_n\}$ be an e-process. The test ``reject $\mathcal{H}_0$ as soon as
$W_n \ge 1/\delta$'' controls the Type-I error probability at level $\delta$, i.e.,
\begin{equation}
\label{app:eq:ville}
  \mathbb{P}\!\left(\sup_{n \ge 0} W_n \ge \frac{1}{\delta}\right)
  \;\le\; \delta
\end{equation}
for all null distributions $\mathbb{P}\in \mathcal{P}_0$.
\end{theorem}

By Theorem \ref{app:thm:ville}, the analyst is free to peek, to plan optional
stopping, and to revise the stopping rule mid-study, without inflating the
Type-I error.

As an illustration, Fig.~\ref{app:fig:ville} shows several sample paths
under the null $\mathcal{H}_0$, together with one under the alternative
$\mathcal{H}_1$. The alternative path crosses the threshold and triggers
rejection, while the $\mathcal{H}_0$ paths remain under the threshold.

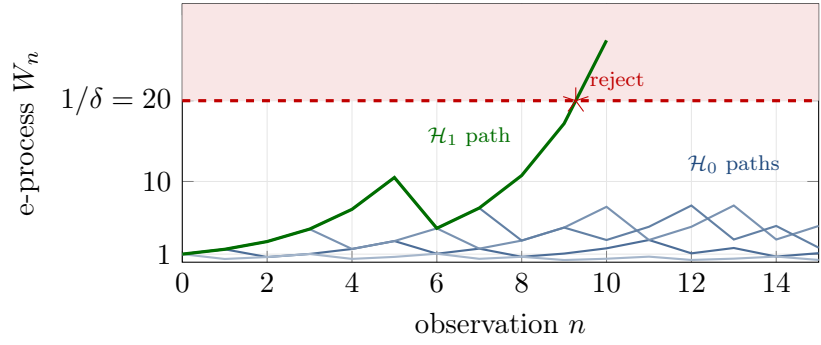
\begin{figure}[h]
\centering
\begin{tikzpicture}[font=\small]
  \begin{axis}[
      width=10cm, height=5.0cm,
      xlabel={\small observation $n$},
      ylabel={\small e-process $W_n$},
      xmin=0, xmax=15, ymin=0, ymax=32,
      ytick={1,10,20},
      yticklabels={$1$, $10$, $1/\delta=20$},
      xtick={0,2,...,14},
      grid=major,
      grid style={line width=0.15pt, draw=gray!20},
      clip=false,
    ]
    \addplot[fill=myred!10, draw=none] coordinates
      {(0,20)(15,20)(15,32)(0,32)} \closedcycle;
    \addplot[myblue!80, thick] coordinates {
      (0,1)(1,1.60)(2,0.64)(3,1.02)(4,1.64)(5,2.62)
      (6,1.05)(7,1.68)(8,0.67)(9,1.07)(10,1.72)
      (11,2.75)(12,1.10)(13,1.76)(14,0.70)(15,1.12)};
    \addplot[myblue!40, thick] coordinates {
      (0,1)(1,0.40)(2,0.64)(3,1.02)(4,0.41)(5,0.66)
      (6,1.05)(7,0.42)(8,0.67)(9,0.27)(10,0.43)
      (11,0.69)(12,0.28)(13,0.44)(14,0.70)(15,0.28)};
    \addplot[myblue!60, thick] coordinates {
      (0,1)(1,1.60)(2,2.56)(3,4.10)(4,1.64)(5,2.62)
      (6,4.19)(7,1.68)(8,2.68)(9,4.29)(10,6.86)
      (11,2.75)(12,4.39)(13,7.02)(14,2.81)(15,4.49)};
    \addplot[myblue!70, thick] coordinates {
      (0,1)(1,1.60)(2,2.56)(3,4.10)(4,6.55)(5,10.49)
      (6,4.20)(7,6.71)(8,2.68)(9,4.29)(10,2.75)
      (11,4.39)(12,7.02)(13,2.81)(14,4.49)(15,1.80)};
    \node[myblue, font=\scriptsize, anchor=south] at (axis cs:13,9.5)
      {$\mathcal{H}_0$ paths};
    \addplot[mygreen, very thick] coordinates {
      (0,1)(1,1.60)(2,2.56)(3,4.10)(4,6.55)(5,10.49)
      (6,4.20)(7,6.71)(8,10.74)(9,17.17)(10,27.47)};
    \addplot[only marks, mark=star, mark size=5pt, myred]
      coordinates {(9.275,20)};
    \node[myred, font=\scriptsize, anchor=south west]
      at (axis cs:9.35,20) {reject};
    \node[mygreen, font=\scriptsize, anchor=south east]
      at (axis cs:8,13) {$\mathcal{H}_1$ path};
    \addplot[myred, dashed, very thick]
      coordinates {(0,20)(15,20)};
  \end{axis}
\end{tikzpicture}
\caption{Several sample paths of an e-process under $\mathcal{H}_0$
(blue, shaded by trajectory) and one under $\mathcal{H}_1$ (green). Ville's
inequality \eqref{app:eq:ville} guarantees that under the null $\mathcal{H}_0$ the supremum $\sup_{n\geq 0} W_n$
exceeds the threshold $1/\delta$ with probability at most $\delta$, so the
test ``stop and reject at the first crossing'' is valid for any
data-dependent stopping rule.}
\label{app:fig:ville}
\end{figure}

\paragraph{Example 3 (continued): Testing a bound on the mean:}
For the null $\mathcal{H}_0: \mu(\mathbb{P}) \ge \mu_{\max}$ with i.i.d.\
data $X_i \in [0,1]$, the quantity
\begin{equation}
  E_t \;=\; 1 + \lambda_t(\mu_{\max} - X_t),
  \qquad
  \lambda_t \in \Bigl[0,\,\tfrac{1}{1-\mu_{\max}}\Bigr],
\end{equation}
is an e-value for every choice of the sequence $\lambda_t$ such that
$\lambda_t$ depends only on the past observations $X_1, \ldots, X_{t-1}$. The
constraint on the parameter $\lambda_t$ keeps each factor non-negative, and one can directly
check the requirement on the mean under $\mathcal{H}_0$. This e-process has
a useful interpretation in terms of betting \citep{waudby2024estimating}.
\hfill $\square$

\chapter{P-value and E-value Constructions}
\label{app:ch2_pvalue_constructions}

This appendix collects technical details on the construction of p-values and e-values, complementing the treatment in Chapter~\ref{chapter:mht}. Sec.~\ref{sec:alternative_constructions} presents a unified viewpoint on constructing valid test statistics via moment generating function bounds, covering bounded, sub-Gaussian, and sub-exponential losses. Sec.~\ref{sec:app_hoeffding_proof} proves superuniformity of the Hoeffding-based p-value (\ref{eq:ch2_hoeffding_pvalue}). Sec.~\ref{sec:app_refined_pvalues} reviews sharper finite-sample constructions, including exact binomial p-values for 0--1 losses, variance-sensitive Bernstein bounds, and the unifying Cram\'{e}r--Chernoff perspective.

\section{Alternative Construction of P-values and E-values}
\label{sec:alternative_constructions}

This section provides an alternative and unifying viewpoint on the construction of p-values and e-values.
Rather than working directly with the empirical risk $\widehat R_n(\lambda_k)$, we express evidence against the null hypothesis $\mathcal{H}_k$ in terms of the cumulative sum of \emph{reliability margins}. Here, $\Delta_{i,k} = \alpha - L_{\lambda_k}(X_i,Y_i)$ denotes the reliability margin of sample $i$ under candidate $\lambda_k$, as introduced in Sec.~\ref{sec:ch2_pval}. This perspective makes it transparent how standard concentration arguments lead to valid test statistics, and it naturally extends beyond bounded losses to more general tail assumptions.

\noindent The \textit{cumulant generating function} (CGF) of a random variable $Z$ is defined as the logarithm of its \textit{moment generating function} (MGF):
\begin{equation}
M_Z(\eta) = \mathbb{E}\!\left[\exp(\eta Z)\right],
\end{equation}
for values of $\eta$ for which the expectation exists. The CGF is thus $\log M_Z(\eta)$.

Following standard derivations of concentration bounds \citep{boucheron2013concentration}, assume that there exists a function $\psi(\eta)$ such that, for all $\eta>0$ the inequality
\begin{equation}
\label{eq:ch2_mgf_bound}
\log \mathbb{E}\!\left[\exp\!\big(\eta \Delta_{i,k}\big)\mid \mathcal{H}_k\right]
\le
\psi(\eta)
\end{equation}
holds. That is, the function $\psi(\eta)$ upper-bounds the CGF of the margin variable $\Delta_{i,k}$ under the null $\mathcal{H}_k$.

Such bounds arise in many settings of practical interest:

\textbf{Bounded losses:}
If the loss $L_{\lambda_k}(X,Y)$ is bounded within the interval $[0,1]$, then the margin variable satisfies the condition $\Delta_{i,k}\in[-1,1]$.
By taking a log of both sides of (\ref{eq:eval_from_bounds}), we obtain
\begin{equation}
\label{eq:ch2_hoeffding_mgf}
\log \mathbb{E}\!\left[\exp\!\big(\eta (\Delta_{i,k} - \mathbb{E}[\Delta_{i,k}])\big)\right]
\le
\frac{\eta^2}{8}.
\end{equation}
This yields a quadratic bound $\psi(\eta)=\eta^2/8 + \eta \mathbb{E}[\Delta_{i,k}]$.

\textbf{Sub-Gaussian losses:}
If $\Delta_{i,k}$ is conditionally sub-Gaussian with variance proxy $\sigma^2$, then the cumulant generating function of $\Delta_{i,k}$ is bounded by
\begin{equation}
\psi(\eta) = \eta \mathbb{E}[\Delta_{i,k}] + \frac{\sigma^2 \eta^2}{2}.
\end{equation}
Such bounds arise for bounded random variables, Gaussian variables, and many light-tailed losses such as the squared loss with additive Gaussian noise, the logistic loss under bounded features, and the negative log-likelihood for sub-Gaussian observations
\citep{vershynin2018high,boucheron2013concentration}.

\textbf{Sub-exponential losses:}
If $\Delta_{i,k}$ is sub-exponential with parameters $(\nu,b)$,
we have
\begin{equation}
\psi(\eta) = \eta \mathbb{E}[\Delta_{i,k}] + \frac{\nu^2 \eta^2}{2}
\qquad \text{for } |\eta| < \frac{1}{b}.
\end{equation}
This setting covers heavier-tailed losses while still permitting exponential concentration
\citep{vershynin2018high,boucheron2013concentration}.

\begin{lemma}
Given the bound (\ref{eq:ch2_mgf_bound}) and a fixed $\eta > 0$, the statistic
\begin{equation}
\label{eq:ch2_exp_evalue}
e_k
=
\exp\!\left(
\eta \sum_{i=1}^n \Delta_{i,k}
-
n \psi(\eta)
\right),
\end{equation}
is a valid e-value for the null hypothesis $\mathcal{H}_k$.
\end{lemma}
\begin{proof}
Using independence of the calibration samples, we can write
\begin{equation}
\mathbb{E}[e_k \mid \mathcal{H}_k]
=
\prod_{i=1}^n
\mathbb{E}
\left[
\exp\!\big(\eta \Delta_{i,k} - \psi(\eta)\big)
\mid \mathcal{H}_k
\right].
\end{equation}
By (\ref{eq:ch2_mgf_bound}) we have
\begin{equation}
\mathbb{E}
\left[
\exp\!\big(\eta \Delta_{i,k} - \psi(\eta)\big)
\mid \mathcal{H}_k
\right]
\le 1.
\end{equation}
Hence,
\begin{equation}
\mathbb{E}[e_k \mid \mathcal{H}_k]
\le 1,
\end{equation}
so $e_k$ satisfies the defining property of an e-value (\ref{eq:ch2_evalue_validity}).
\end{proof}

\section{Proof of Validity for the Hoeffding P-value}
\label{sec:app_hoeffding_proof}

\begin{proposition}[Superuniformity of (\ref{eq:ch2_hoeffding_pvalue})]
\label{prop:app_hoeffding_valid}
Under the null hypothesis $\mathcal{H}_k: R(\lambda_k)>\alpha$, the p-value $p_k$ in (\ref{eq:ch2_hoeffding_pvalue}) is superuniform, i.e., it satisfies $\mathbb{P}(p_k \le u\mid \mathcal{H}_k) \le u$ for every $u\in[0,1]$.
\end{proposition}

\begin{proof}
Fix $u\in(0,1]$ and let
\begin{equation}
s(u) \;=\; \sqrt{\frac{\log(1/u)}{2n}}.
\end{equation}
The event $\{p_k\le u\}$ is equivalent to $\widehat{R}_n(\lambda_k) \le \alpha - s(u)$. Under the null $\mathcal{H}_k$, i.e., $R(\lambda_k)>\alpha$, this implies
\begin{equation}
\widehat{R}_n(\lambda_k) - R(\lambda_k) \;<\; -s(u).
\end{equation}
Applying Hoeffding's inequality (\ref{eq:hoeffding_unit_interval}) with $t = s(u)$ yields
\begin{equation}
\mathbb{P}(p_k \le u\mid \mathcal{H}_k) \;\le\; \exp\!\big(-2n\,s(u)^2\big) \;=\; u.
\end{equation}
\end{proof}

\section{Refined Finite-Sample P-values for Bounded Losses}
\label{sec:app_refined_pvalues}

The Hoeffding p-value (\ref{eq:ch2_hoeffding_pvalue}) is assumption-free, but it depends only on the range of the loss and is therefore conservative whenever the loss exhibits additional structure or small variance. This section reviews sharper constructions for bounded losses. All of them yield valid p-values for the null $\mathcal{H}_k: R(\lambda_k) > \alpha$ via the inversion principle of Sec. \ref{sec:confidence_bounds}: a one-sided concentration bound on $\widehat{R}_n(\lambda_k) - R(\lambda_k)$ is evaluated at the realized deviation $\Delta_k = (\alpha - \widehat{R}_n(\lambda_k))_+$.

\subsection{Exact Binomial P-values for 0--1 Losses}
\label{sec:app_binomial}

When the loss is binary, i.e., $L_{\lambda_k}\in\{0,1\}$, as is the case for the $0$--$1$ misclassification loss in (\ref{eq:example_loss}), the test statistic $S = n\widehat{R}_n(\lambda_k) = \sum_{i=1}^n L_{\lambda_k}(X_i,Y_i)$ follows a binomial distribution with parameters $n$ and $R(\lambda_k)$. Under the null $R(\lambda_k)>\alpha$, the left-tail probability $\mathbb{P}(S \le s)$ is monotonically decreasing in $R(\lambda_k)$, so the least favourable null distribution is the boundary $R(\lambda_k) = \alpha$. The resulting exact one-sided p-value is
\begin{equation}
\label{eq:app_binom_p}
p_k^{\mathrm{bin}}
\;=\;
\sum_{j=0}^{S}\binom{n}{j}\alpha^j(1-\alpha)^{n-j},
\end{equation}
which is the classical Clopper--Pearson tail \citep{clopper1934use}. The p-value (\ref{eq:app_binom_p}) is exact rather than concentration-based, and is typically much sharper than the Hoeffding p-value (\ref{eq:ch2_hoeffding_pvalue}) for binary losses.

\subsection{Variance-Sensitive P-values}
\label{sec:app_bernstein}

When the loss variance is small, the Hoeffding bound is loose because it depends only on the range of the loss. Bernstein-type inequalities incorporate the variance and can yield substantially tighter bounds. Specifically, letting $\sigma^2 = \mathrm{Var}(L_{\lambda_k}(X,Y))$, Bernstein's inequality \citep{bernstein1924modification,bennett1962probability,boucheron2013concentration} states that, for any $t>0$,
\begin{equation}
\label{eq:app_bernstein}
\mathbb{P}\!\left(\widehat{R}_n(\lambda_k) - R(\lambda_k) \le -t\right)
\;\le\;
\exp\!\left(-\frac{n t^2}{2\sigma^2 + \tfrac{2}{3}t}\right).
\end{equation}
Inverting the bound (\ref{eq:app_bernstein}) at the realized deviation $\Delta_k = (\alpha - \widehat{R}_n(\lambda_k))_+$ yields the Bernstein p-value
\begin{equation}
\label{eq:app_bernstein_p}
p_k^{\mathrm{Bern}}
\;=\;
\exp\!\left(-\frac{n \Delta_k^2}{2\sigma^2 + \tfrac{2}{3}\Delta_k}\right).
\end{equation}
Validity follows by the same argument as in the proof of Proposition \ref{prop:app_hoeffding_valid}, with Hoeffding's inequality (\ref{eq:hoeffding_unit_interval}) replaced by (\ref{eq:app_bernstein}). Bennett's inequality \citep{bennett1962probability} gives a related exponential bound with a slightly sharper exponent for moderate deviations, and can be inverted analogously.

The variance $\sigma^2$ is generally unknown, and two strategies are common. The first replaces $\sigma^2$ by the worst-case bound $\sigma^2 \le 1/4$ valid for $[0,1]$-bounded losses, yielding a distribution-free refinement of Hoeffding. The second uses the empirical variance
\begin{equation}
\label{eq:app_emp_variance}
\widehat{V}_n(\lambda_k)
\;=\;
\frac{1}{n}\sum_{i=1}^n
\big(L_{\lambda_k}(X_i,Y_i) - \widehat{R}_n(\lambda_k)\big)^2
\end{equation}
in place of $\sigma^2$ in (\ref{eq:app_bernstein_p}), invoking an empirical Bernstein inequality \citep{maurer2009empirical,audibert2009exploration} to obtain a fully data-adaptive p-value. Empirical Bernstein p-values are particularly powerful when the deployed predictor is near-deterministic for most inputs, so that the per-sample loss variance is small.

\subsection{The Cram\'er--Chernoff Method}
\label{sec:app_chernoff}

The Cram\'er--Chernoff method \citep{chernoff1952measure} provides a unifying perspective on the constructions above. As detailed in Sec. \ref{sec:alternative_constructions}, all the bounds discussed here can be obtained by upper-bounding the cumulant generating function of the per-sample reliability margin under the null and inverting the resulting exponential bound. Different assumptions on the loss, such as boundedness, sub-Gaussianity, or sub-exponentiality, yield different bounding functions $\psi(\eta)$ and hence different p-value constructions, but the underlying methodology is the same.

\chapter{Additional Multiple Testing Procedures}
\label{app:additional_mht}

This appendix collects additional multiple testing procedures that complement the treatment in Chapter~\ref{chapter:mht}. Sec.~\ref{sec:app_fwer} reviews four FWER-controlling procedures beyond Bonferroni and fixed sequence testing: Holm step-down, Hochberg step-up, \v{S}id\'{a}k refinement, and Westfall--Young permutation calibration. Sec.~\ref{app:e-BH} proves the FDR control guarantee for the e-BH procedure (Theorem~\ref{thm:app_eBH}).

\section{Additional FWER-Controlling Procedures}
\label{sec:app_fwer}

This section reviews four FWER-controlling procedures that complement the Bonferroni rule (\ref{eq:ch2_bonferroni}) and fixed sequence testing introduced in Sec. \ref{sec:ch2_pval_mht}. In all cases, given a vector of p-values $\{p_k\}_{k=1}^{|\Lambda|}$, the procedure returns a set $\hat{\mathcal{K}}\subseteq\{1,\dots,|\Lambda|\}$ of rejected null hypotheses, and the corresponding certified set $\hat{\Lambda} = \{\lambda_k : k\in\hat{\mathcal{K}}\}$ enjoys the FWER guarantee of Theorem \ref{thm:ch2_certified_set}. Let $p_{(1)} \le \cdots \le p_{(|\Lambda|)}$ denote the sorted p-values with associated hypotheses $\mathcal{H}_{(1)},\dots,\mathcal{H}_{(|\Lambda|)}$.

\textbf{Holm step-down procedure:} Holm's procedure rejects the first $k^\star$ ordered hypotheses, where
\begin{equation}
\label{eq:ch2_holm_rule}
k^\star \;=\; \max\!\left\{k:\ p_{(i)} \le \frac{\delta}{|\Lambda|-i+1}\ \text{ for all } i=1,\dots,k\right\}.
\end{equation}
Holm controls FWER at level $\delta$ under arbitrary dependence among the p-values, and is uniformly more powerful than Bonferroni in the sense that the certified set produced by (\ref{eq:ch2_holm_rule}) always contains the one produced by (\ref{eq:ch2_bonferroni}).

\textbf{Hochberg step-up procedure:} Hochberg's procedure uses the same thresholds as Holm but in a step-up manner, rejecting the $k^\star$ ordered hypotheses with
\begin{equation}
\label{eq:ch2_hochberg_rule}
k^\star \;=\; \max\!\left\{k:\ p_{(k)} \le \frac{\delta}{|\Lambda|-k+1}\right\}.
\end{equation}
Hochberg is less conservative than Holm, but its FWER guarantee requires additional assumptions on the dependence structure, such as independence or positive regression dependence among the p-values.

\textbf{\v{S}id\'ak refinement:} If the p-values are independent under the joint null, the Bonferroni threshold $\delta/|\Lambda|$ in (\ref{eq:ch2_bonferroni}) can be replaced by the \v{S}id\'ak threshold
\begin{equation}
\label{eq:ch2_sidak}
p_k \le 1-(1-\delta)^{1/|\Lambda|},
\end{equation}
which is slightly larger and hence enlarges the certified set. Analogously, the Holm thresholds in (\ref{eq:ch2_holm_rule}) can be refined to $1-(1-\delta)^{1/(|\Lambda|-i+1)}$, yielding the Holm--\v{S}id\'ak procedure. The improvement over Bonferroni and Holm is modest for small $\delta$, and should only be used when independence is defensible.

\textbf{Westfall--Young permutation calibration:} In hyperparameter certification, the p-values $\{p_k\}_{k=1}^{|\Lambda|}$ are typically highly dependent because they are all computed on the same calibration dataset (\ref{eq:ch2_emp_risk}). When the data admit a valid permutation scheme under the global null, e.g., label permutations in classification, the Westfall--Young procedure estimates the null distribution of the minimum p-value $\min_k p_k$ and uses it to calibrate a threshold $c_\delta$ satisfying
\begin{equation}
\label{eq:ch2_minp}
\mathbb{P}\!\left(\min_k p_k \le c_\delta\right) \;\le\; \delta
\end{equation}
under the global null. The rejection rule $p_k \le c_\delta$ then controls FWER at level $\delta$ and can yield substantial power gains over Bonferroni and Holm when the dependence is strong, at the cost of the additional computation required to sample from the permutation distribution.

\section{Validity of the e-BH Procedure}
\label{app:e-BH}

This section proves the FDR-control property of the e-BH procedure introduced in Sec. \ref{sec:FDR}. The result holds under arbitrary dependence among the e-values and requires only their marginal validity.

\begin{theorem}[FDR control of e-BH \citep{wang2022fdr}]
\label{thm:app_eBH}
Let $e_1,\dots,e_{|\Lambda|}$ be e-values for the null hypotheses $\mathcal{H}_1,\dots,\mathcal{H}_{|\Lambda|}$ satisfying the marginal validity condition
\begin{equation}
\label{eq:app_evalue_validity}
\mathbb{E}[e_k \mid \mathcal{H}_k] \;\le\; 1, \qquad k=1,\dots,|\Lambda|.
\end{equation}
Let $\hat{\Lambda}\subseteq\Lambda$ denote the rejected set returned by the e-BH rule (\ref{eq:ch2_eBH_kstar}). Then,
\begin{equation}
\label{eq:app_eBH_fdr}
\mathrm{FDR}
\;=\;
\mathbb{E}\!\left[
\frac{|\hat{\Lambda}\cap\Lambda_0|}{\max\{1,|\hat{\Lambda}|\}}
\right]
\;\le\;
\frac{|\Lambda_0|}{|\Lambda|}\,\delta
\;\le\;
\delta,
\end{equation}
where $\Lambda_0$ is the set of unreliable hyperparameters.
\end{theorem}

\begin{proof}
Let $R = |\hat{\Lambda}|$ denote the number of rejections. By the e-BH rule (\ref{eq:ch2_eBH_kstar}), every rejected index $k$ satisfies $e_k \ge |\Lambda|/(R\,\delta)$, which can be rewritten as
\begin{equation}
\label{eq:app_eBH_key}
\frac{\mathbb{I}\{k \in \hat{\Lambda}\}}{\max\{1,R\}} \;\le\; \frac{\delta\, e_k}{|\Lambda|}.
\end{equation}
The false discovery rate can therefore be bounded as
\begin{align}
\mathrm{FDR}
&\;=\;
\mathbb{E}\!\left[\frac{|\hat{\Lambda}\cap\Lambda_0|}{\max\{1,R\}}\right] \\
&\;=\;
\sum_{k\in\Lambda_0}\mathbb{E}\!\left[\frac{\mathbb{I}\{k \in \hat{\Lambda}\}}{\max\{1,R\}}\right] \\
&\;\le\;
\frac{\delta}{|\Lambda|}\sum_{k\in\Lambda_0}\mathbb{E}[e_k \mid \mathcal{H}_k]
\;\le\;
\frac{|\Lambda_0|}{|\Lambda|}\,\delta
\;\le\;
\delta,
\end{align}
where the first inequality uses (\ref{eq:app_eBH_key}) and the second invokes the marginal validity (\ref{eq:app_evalue_validity}) at each $k\in\Lambda_0$.
\end{proof}

Crucially, the proof does not require independence among the e-values. FDR control follows from the marginal validity condition (\ref{eq:app_evalue_validity}) and from the rejection rule (\ref{eq:ch2_eBH_kstar}).


\begingroup
\sloppy
\printbibliography[title=References,heading=bibintoc]
\endgroup

\end{document}